\documentclass[11pt]{article}

\usepackage{bm}
\usepackage{amsmath,amssymb,amsfonts,amsthm}
\usepackage{graphicx}
\graphicspath{{Fig/}}

\theoremstyle{plain}
\newtheorem{theorem}{Theorem}[section]

\theoremstyle{definition}

\theoremstyle{remark}

\usepackage[final]{hyperref} 
\usepackage[authoryear]{natbib}
\usepackage{algorithm}
\usepackage{algorithmic}
\usepackage[margin=1in]{geometry}

\newtheorem{condition}{Condition}
\newtheorem{lemma}[theorem]{Lemma}
\def\tn{\tilde{n}}
\def\bx{\bm{x}}

\def\eps{\epsilon}
\def\sig{\sigma}
\def\by{\bm{y}}
\def\bu{\bm{u}}
\def\bdelta{\bm\delta}

\def\bz{\bm{z}}

\def\btheta{\bm{\theta}}
\def\bbeta{\bm{\beta}}
\def\R{\mathbb{R}}
\def\P{\mathbb{P}}
\def\tx{\tilde{\bm{x}}}
\def\Sig{\Sigma}
\def\calD{\mathcal{D}}
\def\E{\mathbb{E}}

\def\uls{\textup{uls}}
\def\ols{\textup{ols}}
\def\rtr{\textup{rtr}}
\DeclareMathOperator*{\argmin}{arg\,min} 
 
\def\unl{\textup{unl}}
\def\bpsi{\bm\psi}
\def\bmu{\bm{\mu}}

\title{Efficient machine unlearning with minimax optimality}
\author{
Jingyi Xie\textsuperscript{1}
\and
Linjun Zhang\textsuperscript{2}
\and
Sai Li\textsuperscript{3}\thanks{Corresponding author. Email: \href{mailto:saili@tsinghua.edu.cn}{saili@tsinghua.edu.cn}}
}
\date{}

\begin{document}



\maketitle






\begin{center}
\textsuperscript{1}Institute of Statistics and Big Data, Renmin University of China, Beijing 100872, China\\
\textsuperscript{2}Department of Statistics, Rutgers University, Piscataway, NJ 08854, USA\\
\textsuperscript{3}Department of Statistics and Data Science, Tsinghua University, Beijing 100084, China
\end{center}

\vspace{1em}

\noindent\textbf{Abstract.}
There is a growing demand for efficient data removal to comply with regulations like the GDPR and to mitigate the influence of biased or corrupted data. This has motivated the field of machine unlearning, which aims to eliminate the influence of specific data subsets without the cost of full retraining.
In this work, we propose a statistical framework for machine unlearning with generic loss functions and establish theoretical guarantees. For squared loss, especially, we develop Unlearning Least Squares (ULS) and establish its minimax optimality for estimating the model parameter of remaining data when only the pre-trained estimator, forget samples, and a small subsample of the remaining data are available.
Our results reveal that the estimation error decomposes into an oracle term and an unlearning cost determined by the forget proportion and the forget model bias. We further establish asymptotically valid inference procedures without requiring full retraining. Numerical experiments and real-data applications demonstrate that the proposed method achieves performance close to retraining while requiring substantially less data access.

\vspace{0.5em}

\section{Introduction}
\label{sec1}
The rapid proliferation of machine learning and artificial intelligence has revolutionized data-driven decision-making across industries. However, this surge has precipitated critical ethical, legal, and practical concerns regarding data privacy, ownership, and control. In particular, regulatory frameworks such as {Article 17 of GDPR}\footnote{\url{https://gdpr-info.eu/art-17-gdpr/}} and the California Consumer Privacy Act (CCPA) have codified the ``right to be forgotten'', granting individuals the right to request the erasure of their data from processing systems.
These legal mandates have intensified the need for \textit{machine unlearning}: the process of efficiently removing the influence of specific data points from pre-trained models, to ensure compliance and uphold user privacy \citep{ginart2019making,guo2020certified}. 

Applications of unlearning extend beyond privacy. In federated learning, where a global model is aggregated from decentralized clients, a participant may withdraw consent, necessitating the removal of their specific contributions \citep{nguyen2024empirical,jin2023forgettable}.
Furthermore, unlearning serves as a critical mechanism for correcting outliers or contaminated data to enhance algorithmic fairness and robustness \citep{cao2015towards,izzo2021approximate}. In the context of Large Language Models (LLMs), unlearning is increasingly employed to excise toxic content or hazardous knowledge \citep{wu2020deltagrad,li2024wmdp,yao2024large}. 

However, removing a subset of data from pre-trained models presents a significant challenge. Large-scale models often encapsulate the complete information of the training set, making clean removal computationally nontrivial. Naively retraining models from scratch is often computationally infeasible, especially in large-scale or real-time settings. In contrast, machine unlearning is a \textit{post-processing} task, aiming to make efficient local updates to a pre-trained model.

In this paper, we study the statistical limits of machine unlearning: how accurately can one estimate the target parameter of the remaining data when only the pre-trained model, the forget data, and a small subsample of the retained data are available?

\subsection{Related work}
We first distinguish between two main  purposes of machine unlearning in the existing literature. The first purpose is to protect privacy or copyright for forget samples. In this case, the desired unlearning algorithm should produce a model which is indistinguishable from a model trained solely on $\calD_r$. A formal definition has been proposed in \citet{izzo2021approximate} and another version based on differential privacy is introduced in \citet{sekhari2021remember} and \citet{neel2021descent}. 
The second purpose is to remove outliers or biased documents, where a significant challenge is that there exists a distribution shift between the forget and remaining data. In this case, the main goal of unlearning is bias correction but there is no strict requirement on whether the output is independent of the forget data or not. This scenario is critical in LLM safety alignment. For instance, \citet{li2024wmdp} develop a benchmark dataset for LLMs aiming at removing hazardous knowledge such as knowledge for developing biological, cyber, and chemical weapons. 
 \citet{yao2024large} show that unlearning can be an efficient way to align LLMs with human preferences. In this work, we focus on the second scenario: unlearning for bias correction and outlier removal.

A major class of machine unlearning methods is gradient-based. A baseline unlearning method, particularly for LLMs, is the gradient ascent (GA) algorithm \citep{yao2024large,maini2024tofu}, which updates model parameters towards maximizing the loss for the forget data. However,  GA lacks convergence guarantees under typical conditions and can degrade model utility \citep{zhang2024negative}. Negative Preference Optimization \citep{zhang2024negative} is a gradient ascent based algorithm whose gradient is an adaptive weighting of that of GA. GradDiff method \citep{liu2022continual,liu2025rethinking} considers maximizing the forget loss and minimizing the loss for the remaining data at the same time, which has more reliable performance in practice. Variants such as \citet{liu2025recovertoforget, wang2025rethinking, yang2025exploring} have been proposed to improve gradient direction and the learning rates.

In the machine learning literature, exact and approximate unlearning have been actively studied. \citet{cao2015towards} first formally articulate the concept of machine unlearning. \citet{guo2020certified} propose a one-step Newton update and establish its certified removal property. 
\citet{izzo2021approximate} develop a projective residual update method to approximate the leave-one-out residuals for linear models.  Machine unlearning methods for tree-based models are studied in \citet{brophy2021machine} and \citet{lin2023machine}. While effective, these approaches often require storing Hessian matrices or accessing the full dataset, which can be prohibitive in high-dimensional or privacy-sensitive settings.
Another class of unlearning methods uses synthetic data to fine-tune the pre-trained model. For instance, \citet{graves2021amnesiac} propose to relabel the sensitive data with randomly selected incorrect labels and then updates the model together with the remaining data. 
\citet{he2025towards} construct the fine-tuning dataset for unlearning by mixing up the remaining and forget dataset.
See \citet{shaik2024exploring} for a systematic review. 

Unlearning shares conceptual similarities with transfer learning, as both aim to adapt a source model to a target distribution.
Many recent works have studied transfer learning approaches for nonparametric regression \citep{cai2021transfer,reeve2021adaptive} and high-dimensional parametric models \citep{li2022transfer,tian2023transfer,li2024estimation} among many others.  However, key distinctions exist. In transfer learning, users have access to the raw source data to construct the pre-trained model. In the unlearning, the pre-trained model is pre-determined. Another distinction is that there is no forget data in transfer learning.
As we will show, state-of-the-art transfer learning approaches are sub-optimal for the unlearning problem.

\subsection{Our contributions}
Despite the pressing need for unlearning, existing methods often lack rigorous statistical guarantees, and the information-theoretic limits of the problem remain largely unknown. The contributions of this work are as follows.

We propose an unlearning estimator $\hat{\btheta}^{(\textup{unl})}_r$ for generic differentiable losses and develops a gradient descent algorithm for its realization.
For squared loss, especially, the proposed estimator reduces to so-called unlearning least squares (ULS), a computationally efficient algorithm that leverages the forget data to debias the pre-trained estimator. 
We rigorously establish the minimax optimality of ULS for estimating the true model of remaining data. Our lower bound analysis involves {a novel perturbation analysis of the pre-trained estimator’s distribution}. This new optimal rate highlights the dependency of the unlearning accuracy on the proportion of the forget data and on the model discrepancy between the forget and remaining data. Furthermore, we develop asymptotically valid inference procedures based on the unlearning estimator, enabling hypothesis testing and confidence interval construction without full retraining.

\subsection{Organization and notation}
The remainder of the paper is organized as follows. Section \ref{sec2} formalizes the problem and introduces our main proposal, the unlearning estimate under generic loss functions. Section \ref{sec3} develops the method and minimax optimal guarantees for our proposal with squared loss. In Section \ref{sec5}, we develop asymptotically valid inference procedures. Section \ref{sec6} proposes a robust version of the main proposal and makes connections to existing LLM unlearning heuristics.
Sections \ref{sec-simu} and \ref{sec-data} present simulation studies and applications to Yelp and UK Biobank data, respectively. Section \ref{sec-diss} concludes with a discussion of future directions.

We introduce some notations that will be repeatedly used in the rest of this work.  For a semi-positive definite matrix $A$, let $\Lambda_{\min}(A)$ and $\Lambda_{\max}(A)$ denotes the smallest and largest eigenvalue of $A$, respectively. Let $a_n=O(b_n)$ and $a_n\lesssim b_n$ denote $|a_n/b_n|\leq c$ for some constant $c$ when $n$ is large enough. Let $a_n=o(b_n)$ and $a_n\ll b_n$ denote $a_n/b_n\rightarrow 0$ as $n\rightarrow\infty$. We use $C,C_0,C_1,\dots,c,c_0,c_1,\dots$ to denote generic constants which may vary across statements.

\section{Problem set-up and generic unlearning estimate}
\label{sec2}
In this section, we first formalize the unlearning problem in Section \ref{sec2-1}. Section \ref{sec2-2} derives the rationale of our proposal via Taylor expansion. Section \ref{sec2-3} presents the gradient descent realization of the proposed unlearning estimate. 
\subsection{Problem statement}
\label{sec2-1}
To formalize the problem, let $\hat{\btheta}_p\in\R^p$ denote the pre-trained model parameter estimated from the full dataset $\calD$. Let $\calD_f\subset \calD$ denote the set of forget data, the data to be removed, and $((\bx_i^{(f)})^{\top},y_i^{(f)})\in\R^p\times \R$ denotes independent samples from $\calD_f$ for $i=1,\dots,N_f$. Let $\calD_r=\calD\setminus\calD_f$ denote the set of remaining data and  $((\bx_i^{(r)})^{\top},y_i^{(r)})\in\R^p\times \R$ denote the independent samples from $\calD_r$, $i=1,\dots,N_r$.  The total dataset $\calD=\calD_r\cup \calD_f$ with sample size $N=N_r+N_f$.  We mention that the samples are not required to be identically distributed in each data set which allows for internal heterogeneity within each data.

Let $\ell(\btheta;\bx,y)$ denote the loss function for model training. If the outcome is continuous, one may consider the squared loss $\ell(\btheta;\bx,y)=(y-\bx^{\top}\btheta)^2$; if the outcome is binary, one may consider the cross entropy loss $\ell(\btheta;\bx,y)=y\log(1+\exp\{-\bx^{\top}\btheta\})+(1-y)\log(1+\exp\{\bx^{\top}\btheta\})$. {To accommodate non-linearity and high-dimensionality, the input features can be latent representations, such as those extracted from deep neural networks, rather than raw covariates.} In the unlearning setting, the target parameter of interest is
\begin{align}
\label{def-thetar}
  \btheta_r=\argmin_{\btheta\in\R^p}\E[\ell(\btheta;\calD_r)],
\end{align}
where $\ell(\btheta;\calD_r)=\sum_{i=1}^{N_r}\ell(\btheta;\bx_i^{(r)},y_i^{(r)})$ and the expectation is taken over the randomness of $\calD_r$. Notably, we do not assume that the samples in $\calD_r$ are identically distributed. To ensure that $\btheta_r$ remains well-defined and stable as the sample size grows, we assume that the normalized loss $\ell(\btheta;\calD_r)/N_r$ converges in probability to a limiting objective function $\ell^{(r)}_{\infty}(\btheta)$ for any $\btheta\in\R^p$  as $N_r\rightarrow\infty$.

In practical unlearning scenarios, full access to remaining dataset $\calD_r$ is often restricted due to massive size of $\calD_r$ or privacy constraints. Instead, we assume access only to a random subsample $\widetilde{\calD}_r\subseteq \calD_r$ of size $\tn_r$. Let $((\tilde{\bx}_i^{(r)})^{\top},\tilde{y}_i^{(r)})\in\R^p\times \R$ denote the independent samples from $\widetilde{\calD}_r$, $i=1,\dots,\tn_r$. This setting aligns with practical constraints in LLM unlearning and has been empirically studied in many existing works \citep{liu2022continual,liu2025rethinking,anjarlekar2025llm}.

The pre-trained model estimate is defined as the empirical minimizer of the full data set
\begin{align}
\label{eq-ptr}
  \hat{\btheta}_p=\argmin_{\btheta\in\R^p} \ell(\btheta;\calD),
  \end{align}
where $ \ell(\btheta;\calD)=\ell(\btheta;\calD_r)+\ell(\btheta;\calD_f)=\sum_{i=1}^{N_r}\ell(\btheta;\bx_i^{(r)},y_i^{(r)})+\sum_{i=1}^{N_f}\ell(\btheta;\bx_i^{(f)},y_i^{(f)})$.

Our objective is to remove the effect of $\calD_f$ from $\hat{\btheta}_p$. Formally, the ideal unlearning estimator should approximate the re-trained (rtr) estimate
\begin{align}
\label{eq-rtr}
    \mathring{\btheta}^{(\rtr)}_{r}=\argmin_{\btheta\in\R^p} \ell(\btheta;\calD_r)
\end{align}
but without accessing or retraining the full $\calD_r$.

To summarize, the accessible information consists of the forget data $\calD_f$, the subsampled remaining data $\widetilde{\calD}_r$, and the pre-trained estimate $\hat{\btheta}_p$.
In typical settings, the forget set is relatively small. Hence, for $\omega_f=N_f/N$ and $\omega_r=N_r/N$, we assume $\omega_f$ is bounded away from 1, i.e., $ \omega_f\leq c<1$ for some constant $c$. For the size of $\widetilde{\calD}_r$, we consider $\tilde{\omega}_r=\tn_r/N_r\in(0,1]$.
We allow the dimension $p$ to grow to infinity but consider the setting $p$ is relatively small to the sample size, i.e., $p\ll \min \{N_f,\tilde{n}_r\}$. 


\subsection{Rationale for our proposal}
\label{sec2-2}
In view of (\ref{eq-rtr}), the ideal re-trained estimate $\mathring{\btheta}_r^{(\textup{rtr})}$ satisfies the stationary condition
$0=\dot{\ell}(\mathring{\btheta}_r^{(\textup{rtr})};\calD_r)$, where $\dot{\ell}(\btheta;\calD')$  denote the first derivative of $\ell(\btheta;\calD')$ with respect to $\btheta$ for any given dataset $\mathcal{D}'$.

To approximate $\mathring{\btheta}_r^{(\textup{rtr})}$ efficiently, we propose a new machine unlearning method which fine-tunes the pre-trained model $\hat{\btheta}_p$ using $\calD_f$ and $\widetilde{\calD}_r$.   Note that
\begin{align}
   \dot{\ell}(\mathring{\btheta}_r^{(\textup{rtr})};\calD_r)-\dot{\ell}(\hat{\btheta}_p;\calD_r)&=0-(\dot{\ell}(\hat{\btheta}_p;\calD)-\dot{\ell}(\hat{\btheta}_p;\calD_f))=\dot{\ell}(\hat{\btheta}_p;\calD_f),\label{eq-gradient1}
\end{align}
where the first equality leverages the stationary condition of $\mathring{\btheta}_r^{(\textup{rtr})}$ and the decomposition of the pre-trained loss and the second equality leverages the stationarity condition $ 0=\dot{\ell}(\hat{\btheta}_p;\calD)$.

Equation (\ref{eq-gradient1}) provides a viable way to approximate $\mathring{\btheta}_r^{(\textup{rtr})}$ by fine-tuning $\hat{\btheta}_p$. However, although $\calD_f$ and $\hat{\btheta}_p$ are both observed in the unlearning phase, $\calD_r$ is not fully observed. As $\widetilde{\calD}_r$ is a random sample from $\calD_r$, we replace $\dot{\ell}(\btheta;\calD_r)$ with $(N_r/\tn_r)\dot{\ell}(\btheta;\widetilde{\calD}_r)$ in (\ref{eq-gradient1}). This gives a generic unlearning estimator, $\hat{\btheta}^{(\textup{unl})}_r$, which is defined as the solution to
\begin{align}
\label{eq-prop}
0=\frac{N_r}{\tn_r}\{\dot{\ell}(\btheta;\widetilde{\calD}_r)-\dot{\ell}(\hat{\btheta}_p;\widetilde{\calD}_r)\}-\dot{\ell}(\hat{\btheta}_p;\calD_f).
\end{align}
If the second-order derivative exists, by Taylor expansion, the unlearning estimate $\hat{\btheta}^{(\textup{unl})}_r$ can be written as
\begin{align}
\label{eq-prop2}
    \hat{\btheta}_r^{(\textup{unl})}=\hat{\btheta}_p+ \left\{\frac{N_r}{\tn_r}\int_{0}^1\ddot{\ell}(\hat{\btheta}_p+u( \hat{\btheta}_r^{(\textup{unl})}-\hat{\btheta}_p);\widetilde{\calD}_r)du\right\}^{-1}\dot{\ell}(\hat{\btheta}_p;\calD_f).
\end{align}
Equation (\ref{eq-prop2}) reveals that unlearning can be viewed as a Hessian-weighted bias correction to the pre-trained estimator, where the gradient of the forget loss estimates the bias introduced by the forget samples. 
Intuitively, increasing the forget loss  can force pre-trained estimate deviating from the minimum of the forget loss so that the effect of $\calD_f$ can be removed. {From another perspective, the last term in (\ref{eq-prop2}) also approximates to the influence function of the forget data on the parameter vector $\hat{\btheta}_p$ \citep{cook1982residuals}.} We mention that similar expansions as in (\ref{eq-gradient1}) have been considered in a few existing works \citep{guo2020certified,wu2020deltagrad}. The novelty in (\ref{eq-prop}) is that it uses subsampled remaining data as a proxy of the full remaining data, which does not require storing Hessian matrices or accessing the full dataset.  
More importantly, the statistical performance of (\ref{eq-prop}) and (\ref{eq-prop2}) have not been formally established and its minimax optimality is unknown.

\subsection{Gradient descent realization of the unlearning estimate}
\label{sec2-3}

Based on the stationarity condition in (\ref{eq-prop}), it is easy to derive the gradient descent algorithm to realize $\hat{\btheta}^{(\textup{unl})}_r$, as detailed in Algorithm \ref{alg-gd}.

\begin{algorithm}[H]

\textbf{Input}:  Pre-trained estimate $\hat{\btheta}_p$, forget data $\calD_f$, and subsampled remaining data $\widetilde{\calD}_r$.

\textbf{Output}: $\hat{\btheta}^{(\textup{unl})}_{r,T}$.

Set the initial value $\hat{\btheta}^{(\textup{unl})}_{r,0}=\hat{\btheta}_p$.

For $t=1,\dots,T$:  Compute
\begin{align}
\label{eq-gd}
\hat{\btheta}^{(\textup{unl})}_{r,t}=\hat{\btheta}^{(\textup{unl})}_{r,t-1}-\alpha\left\{\frac{N_r}{\tn_r}\dot{\ell}(\hat{\btheta}^{(\textup{unl})}_{r,t-1};\widetilde{\calD}_r)-\frac{N_r}{\tn_r}\dot{\ell}(\hat{\btheta}_p;\widetilde{\calD}_r)-\dot{\ell}(\hat{\btheta}_p;\calD_f)\right\},
\end{align}
where $\alpha>0$ is the step size.

\caption{Gradient descent algorithm for the unlearning estimate}
\label{alg-gd}
\end{algorithm}

With proper choice of step size and regularity conditions, it is easy to show that $\hat{\btheta}^{(\textup{unl})}_{r,T}$ converges to $\hat{\btheta}_r^{(\textup{unl})}$  as $T$ goes to infinity. The results for the convergence rate of Algorithm \ref{alg-gd} are presented in the supplementary materials (Section \ref{ap-loss}). Its convergence rate under squared loss is presented in Theorem \ref{thm0-gd}.

We mention that for common loss functions, such as squared loss and cross-entropy loss, the realization of $\hat{\btheta}^{(\textup{unl})}_{r}$ and $\hat{\btheta}^{(\textup{unl})}_{r,T}$ do not involve the outcome data $\tilde{\by}^{(r)}$. Indeed, terms involving $\tilde{\by}^{(r)}$ are canceled out in $\dot{\ell}(\hat{\btheta}^{(\textup{unl})}_{r,t-1};\widetilde{\calD}_r)$ and in $\dot{\ell}(\hat{\btheta}_p;\widetilde{\calD}_r)$.
We will show in later sections that the proposed estimator is minimax optimal for squared loss under mild conditions.  It implies that for optimal estimation procedures, the outcome data in the subsampled remaining data are not needed, which can further reduce the data collection cost.  However, for valid inference procedures, the remaining outcome data $\tilde{\by}^{(r)}$ are needed to compute the variance of the estimator. This highlights a difference between estimation and inference in the unlearning problem.

\section{Efficient machine unlearning under squared loss}
\label{sec3}
{In this section, we focus on the methods and theory under squared loss to maintain comparability with existing learning and unlearning methods. Crucially, this choice does not necessitate a linear relationship between features and outcomes; rather, our approach is designed to be robust under model misspecification. }

Section \ref{sec3-1} derives the closed-form estimate under squared loss.
Section \ref{sec3-2} derives the finite-sample convergence rates for the ULS estimator. Section \ref{sec3-3} establishes the minimax lower bound, proving that ULS achieves the fundamental limit of the unlearning problem. Section \ref{sec-benchmark} introduces baseline methods for comparison and demonstrates the superiority of our proposal for unlearning tasks.

\subsection{Proposed estimate under squared loss}
\label{sec3-1}
Let $(\widetilde{X}^{(r)},\tilde{\bm{y}}^{(r)})\in\R^{\tn_r\times (p+1)}$ denote matrix representation of the data in  $\widetilde{\calD}_r$ and $(X^{(f)},\bm{y}^{(f)})\in\R^{N_f\times (p+1)}$ denote matrix representation of the samples in  $\calD_f$.
Under squared loss,  formula (\ref{eq-prop2}) gives
\begin{align}
\hat{\btheta}^{(\uls)}_r&=\hat{\btheta}_p-\frac{\tn_r}{N_r}\{(\widetilde{X}^{(r)})^{\top}\widetilde{X}^{(r)}\}^{-1}(X^{(f)})^{\top}(\bm{y}^{(f)}-X^{(f)}\hat{\btheta}_p).\label{unlearn-est}
\end{align}
We term this estimate the unlearning least squares (ULS) estimator.
Furthermore, $\hat{\btheta}_r^{(\uls)}$ is the exact solution to (\ref{ul-loss}) as proved in Section \ref{sec-proof1} of the supplements.
This formulation highlights that ULS maximizes the loss on the forget data while being regularized to stay close to the pre-trained knowledge. Regularization is essential here because solely maximizing the squared loss $\ell(\btheta;\calD_f)$ does not have a finite solution. The regularization term involves a weight matrix $\widetilde{\Sig}_p$, which can be viewed as an approximation of the full-sample Hessian matrix $\ddot{\ell}(\hat{\btheta}_p;\calD)$. We provide further discussions to (\ref{ul-loss}) in comparison to other baseline methods in Section \ref{sec-benchmark}. 
The above unlearning procedure is summarized into Algorithm \ref{alg1}.  Let $\widetilde{\Sig}_r=(\tilde{X}^{(r)})^{\top}\tilde{X}^{(r)}/\tn_r$ and $\widehat{\Sig}_f=(X^{(f)})^{\top}X^{(f)}/N_f$.

\begin{algorithm}[H]
\textbf{Input}: The pre-trained estimate $\hat{\btheta}_p$, forget data $\calD_f$, and subsampled remaining data $\widetilde{\calD}_r$.

\textbf{Output}: $\hat{\btheta}^{(\uls)}_{r}$.

Compute
\begin{align}
\label{ul-loss}
\hat{\btheta}_r^{(\uls)}=\argmin_{\btheta\in\R^p}\left\{-\frac{\omega_f}{N_f}\ell(\btheta;\calD_f)+(\hat{\btheta}_p-\btheta)^{\top}\widetilde{\Sig}_p(\hat{\btheta}_p-\btheta)\right\},
\end{align}
where $\omega_f=N_f/N$, $\ell(\btheta;\calD_f)=\sum_{i=1}^{N_f}\{y_i^{(f)}-(\bx_i^{(f)})^{\top}\btheta\}^2$, and $\widetilde{\Sig}_p=\omega_r\widetilde{\Sig}_r+\omega_f\widehat{\Sig}_f$.
\caption{Proposed ULS estimator}
\label{alg1}
\end{algorithm}

When $p$ is large, directly calculating the inverse of Hessian matrices can be computationally expensive.  In this case, one can perform  the gradient descent version, Algorithm \ref{alg-gd}, with squared loss.

\subsection{Convergence analysis for the proposed methods}
\label{sec3-2}
We provide theoretical guarantees for the ULS estimate in this subsection.
We first state the standard regularity conditions required for our theoretical analysis.

\begin{condition}[Sub-Gaussian data]
\label{cond1}
The remaining data $((\bx_i^{(r)})^{\top},y^{(r)}_i)$, $i=1,\dots,N_r$ are independent sub-Gaussian vectors. Moreover, $\E[\widehat{\Sig}_r]=\Sig_r$ and $c^{-1}_{\Sig}\leq \Lambda_{\min}(\Sig_r)\leq \Lambda_{\max}(\Sig_r)\leq c_{\Sig}$ for some $c_{\Sig}>1$.
The forget data $((\bx_i^{(f)})^{\top},y_i^{(f)})$, $i=1,\dots,N_f$ are independent sub-Gaussian vectors with mean zero. Moreover, $\E[\widehat{\Sig}_f]=\Sig_f$ and $c^{-1}_{\Sig}\leq \Lambda_{\min}(\Sig_f)\leq \Lambda_{\max}(\Sig_f)\leq c_{\Sig}$.
\end{condition}

For the remaining and forget data, we assume the observations are independent sub-Gaussian. {This assumption is commonly used in the analysis of strongly convex objective functions and in linear models \citep{izzo2021approximate,guo2020certified}.} We do not require the samples to be identically distributed, which allows for internal heterogeneity within each dataset. Moreover, we do not require the true relationship between the response and covariates to be linear, which allows for model misspecification.

In the next theorem, we establish the convergence rate of the ULS estimator $\hat{\btheta}^{(\textup{uls})}_r$ defined in (\ref{unlearn-est}). Let $\btheta_f\in\R^p$ denote the true coefficient vector under squared loss for the forget data, i.e.,
$\btheta_f=\argmin_{\btheta\in\R^p}\E[\ell(\btheta;\calD_f)]$.  To ensure $\btheta_f$ is well-defined, we assume that the normalized loss $\ell(\btheta;\calD_f)/N_f$ converges in probability to a limiting objective function $\ell_{\infty}^{(f)}(\btheta)$ for any $\btheta\in\R^p$  as $N_f\rightarrow\infty$.
Let $\delta=\|\btheta_r-\btheta_f\|_2$ denote the magnitude of discrepancy between the remaining and forget model distributions, quantifying the bias forget model relative to the remaining data. 

\begin{theorem}
\label{thm2}
Assume Condition \ref{cond1} and $p=o(\min\{\tilde{n}_r, N_f\})$. Then with probability at least $1-\exp\{-c_1p\}$ , it holds that
\vspace{-0.2in}
\begin{align*}
&\|\hat{\btheta}_r^{(\uls)}-\mathring{\btheta}^{(\rtr)}_r\|_2\leq c_2\omega_f\delta\sqrt{\frac{p}{\tilde{n}_r}}\\
&\|\hat{\btheta}_r^{(\uls)}-\btheta_r\|_2\leq c_2\sqrt{\frac{p}{N_r}}+c_2\omega_f\delta\sqrt{\frac{p}{\tilde{n}_r}},
\end{align*}
where $c_1$ and $c_2$ are positive constants.
\end{theorem}
Theorem \ref{thm2} decomposes the error of $\hat{\btheta}_r^{(\uls)}$ into two components: the oracle rate $O(\sqrt{p/N_r})$ and the unlearning cost $O(\omega_f\delta\sqrt{p/\tilde{n}_r})$. Specifically, the second term quantifies the statistical price of unlearning: it increases linearly with the forget proportion and the discrepancy between the two distributions. The convergence rate gets faster when $N_r$ and $\tn_r$ get larger but the rate gets slower when the forget sample size $N_f$ and model discrepancy $\delta$ get larger. Intuitively, this is because the larger the $N_f$ and $\delta$, the larger the influence of the forget data made on $\hat{\btheta}_p$. Notably, if the forget set is small relative to the subsample, i.e., $\omega_f\delta\lesssim \sqrt{\tn_r/N_r}$, the unlearning cost becomes negligible and ULS achieves the oracle efficiency of full retraining.

Let $\hat{\btheta}^{(\uls)}_{r,T}$ denote gradient descent version of $\hat{\btheta}_r^{(\uls)}$, the realization of Algorithm \ref{alg-gd} based on squared loss.
Next, we provide theoretical guarantees for this gradient descent realization.

\begin{theorem}
\label{thm0-gd}
Assume Condition \ref{cond1}.
Let the step size $\alpha$ be a constant  such that $0<\alpha< (1-c_{\alpha})/[N_r\Lambda_{\max}(\widetilde{\Sig}_r)]$ for some $c_{\alpha}<1$. Then with probability at least $1-\exp\{-c_1p\}$, for any given integer $T$,
\begin{align*}
   &\|\hat{\theta}^{(\uls)}_{r,T}-\hat{\btheta}_r^{(\uls)}\|_2\leq c_2c_{\alpha}^{T}\omega_f\left(\delta+\sqrt{\frac{p}{\min\{N_f,N_r\}}}\right).\\
  & \|\hat{\btheta}^{(\uls)}_{r,T}-\btheta_r\|_2\leq c_2c_{\alpha}^{T}\omega_f\left(\delta+\sqrt{\frac{p}{\min\{N_f,N_r\}}}\right)+c_2\sqrt{\frac{p}{N_r}}+c_2\omega_f\delta\sqrt{\frac{p}{\tn_r}}
\end{align*}
for some positive constants $c_1$ and $c_2$.
\end{theorem}
With proper step size $\alpha$, the gradient descent estimate $\hat{\btheta}^{(\uls)}_{r,T}$ converges to the ULS estimate $\hat{\btheta}_r^{(\uls)}$ as $T\rightarrow\infty$. From the second inequality, we see that after $T=O(\ln \tn_r)$ steps, the computational error becomes negligible compared to the statistical error, preserving the same convergence rate proved in Theorem \ref{thm2}. 

\subsection{Minimax optimality}
\label{sec3-3}
In this subsection, we establish the minimax lower bound for estimating $\btheta_r$ in the machine unlearning setting.

Consider the parameter space
{\small
\begin{align}
\label{parameter-space}
  \Theta_p(\delta)&=\{\bm\beta=(\btheta_r,\btheta_f,\Sig_r,\Sig_f,\sig^2_r,\sig^2_f):\btheta_r\in\R^p,\btheta_f\in\R^p,\|\btheta_r-\btheta_f\|_2\leq \delta,~\sig^2_r>0,~\sig^2_f>0,\nonumber\\
  &\quad ~c^{-1}_1\leq \Lambda_{\min}(\Sig_r)\leq \Lambda_{\max}(\Sig_r)\leq c_1,~ c_1^{-1}\leq \Lambda_{\min}(\Sig_f)\leq\Lambda_{\max}(\Sig_f)\leq c_1\},
\end{align}
}
where $p$ characterizes the dimension of target parameter and $\delta$ characterizes the model discrepancy between $\calD_r$ and $\calD_f$.

\begin{theorem}
\label{thm-minimax}
Assume that $\omega_r\geq 1/2$, $\omega_f\delta\leq 1$, and $p\leq c_0\min\{\sqrt{N},\tn_r\}$  for some small enough positive constant $c_0$. For the parameter space $\Theta_p(\delta)$ defined in (\ref{parameter-space}), it holds that
\[
   \inf_{\hat{\btheta}\in\mathcal{F}(\hat{\btheta}_p,\calD_f,\widetilde{\calD}_r)}\sup_{\Theta_p(\delta)} \P\left(\|\hat{\btheta}-\btheta_r\|_2\geq c_1\sqrt{\frac{p}{N_r}}+c_1\omega_f\delta\sqrt{\frac{p}{\tilde{n}_r}}\right)\geq c_2
\]
for some positive constants $c_1$ and $c_2$.
\end{theorem}
The lower bound shows that any unlearning procedure based on $(\hat{\btheta}_p,\calD_f,\widetilde{\calD}_r)$ must incur an additional error proportional to $\omega_f\delta\sqrt{\frac{p}{\tilde{n}_r}}$ under the conditions of Theorem \ref{thm-minimax}. This term captures the intrinsic difficulty of correcting the bias induced by the forget samples when only partial access to the retained data is available. 
This result shows that the proposed $\hat{\btheta}_r^{(\uls)}$ is minimax optimal whenever the ``unlearning cost" is not overwhelmingly large ($\omega_f\delta\leq 1$). To accommodate the extreme scenarios $\omega_f\delta\gg 1$, we show that a regularized version of $\hat{\btheta}_r^{(\uls)}$,  defined in (\ref{ul-loss2}), can achieve the minimax optimal rate even when $\omega_f\delta\gg 1$.

The proof of Theorem \ref{thm-minimax} presents a unique technical challenge due to the complex dependency between the pre-trained estimator $\hat{\btheta}_p$ and the observed data $\{\calD_f, \widetilde{\calD}_r\}$. Standard lower bound techniques typically assume independent observations. To overcome the complexity, we employ a novel perturbation analysis of the pre-trained estimator's distribution, decoupling the dependency structure.This technique may be of independent interest for other problems when the observed statistics have complex dependency. 

\subsection{Comparison to baseline methods}
\label{sec-benchmark}
To better illustrate the theoretical advantages of the proposed methods,
we review some existing methods that can  also be applied to the unlearning problem under the set-up introduced in Section \ref{sec2-1}. The counterparts for comparison include the pre-trained estimate, the OLS estimate based on $\widetilde{\calD}_r$, and transfer learning method for linear models.

Given the observations $\{\hat{\btheta}_p,\calD_f,\widetilde{\calD}_r\}$, a direct approach to estimate $\btheta_r$ is the least square estimate based on subsampled remaining data:
\begin{align}
\tilde{\btheta}_r^{(\ols)}=\left\{(\widetilde{X}^{(r)})^{\top}\widetilde{X}^{(r)}\right\}^{-1}(\widetilde{X}^{(r)})^{\top}\tilde{\bm{y}}^{(r)}.
\end{align}
It is easy to see that $\tilde{\btheta}_r^{(\ols)}$ is unbiased for $\btheta_r$ but its variance is proportional to $p/\tn_r$, which can be large when $\tilde{n}_r$ is relatively small.

The second candidate estimate is the pre-trained estimate under squared loss, which is
\[\hat{\btheta}_p=\left\{(X^{(r)})^{\top}X^{(r)}+(X^{(f)})^{\top}X^{(f)}\right\}^{-1}\left\{(X^{(r)})^{\top}\by^{(r)}+(X^{(f)})^{\top}\by^{(f)}\right\}.
\] We know that it has relatively small variance but can have large bias when there is a large model distinction between $\calD_r$ and $\calD_f$.

The third benchmark method is the transfer learning estimator. We may treat $\hat{\btheta}_p$ as an estimate from the source data and treat $\widetilde{\calD}_r$ as samples from the target data.
Motivated by the idea of transfer learning, we define
\begin{align}
\label{eq-tl}
   \hat{\btheta}^{(\textup{tl})}_r=\argmin_{\btheta\in\R^p} \left\{\frac{1}{\tn_r}\ell(\btheta;\widetilde{\calD}_r)+\lambda^{(\textup{tl})} \|\btheta-\hat{\btheta}_p\|_2^2\right\},
\end{align}
where $\lambda^{(\textup{tl})}>0$ is a tuning parameter. We consider the Ridge penalty to induce similarity between the pre-trained estimate $\hat{\btheta}_p$ and the target parameter. 
Although transfer learning adapts to the remaining data efficiently, it ignores the information of the forget data $\calD_f$. We will show that $\hat{\btheta}_r^{(\textup{tl})}$ can be sub-optimal for unlearning when the distribution shift is significant.

We treat the re-trained estimate based on $\calD_r$ as the oracle baseline method. By (\ref{eq-rtr}), under squared loss,
\begin{align}
\label{eq-ols}
   \mathring{\btheta}_r^{(\rtr)}=((X^{(r)})^{\top}X^{(r)})^{-1}(X^{(r)})^{\top}\bm{y}^{(r)}.
\end{align}

\begin{lemma}[Convergence rate of benchmark estimators]
\label{tlem2}
Assume Condition \ref{cond1}. Given that $p=o(\min\{\tilde{n}_r,N_f\})$, with probability at least $1-\exp\{-c_1p\}$, we have
\begin{align*}
&\|\mathring{\btheta}_r^{(\rtr)}-\btheta_r\|_2\leq C\sqrt{\frac{p}{N_r}}\\
&\|\hat{\btheta}_p-\btheta_r\|_2\leq C(\sqrt{\frac{p}{N}}+\omega_f\delta)\\
&\|\tilde{\btheta}_r^{(\ols)}-\btheta_r\|_2\leq C\sqrt{\frac{p}{\tilde{n}_r}}.
\end{align*}
If we take $\lambda^{(\textup{tl})}=\sqrt{p/\tn_r}/(\sqrt{p/N}+\omega_f\delta)$ in (\ref{eq-tl}), then
\[
 \|\hat{\btheta}_r^{(\textup{tl})}-\btheta_r\|_2\leq C\min\left\{\sqrt{\frac{p}{\tn_r}},\sqrt{\frac{p}{N}}+\omega_f\delta\right\}.
\]
\end{lemma}

From Lemma \ref{tlem2}, we see that the convergence rate of transfer learning estimate $\hat{\btheta}_r^{(\textup{tl})}$ is the minimum of the rates of $\tilde{\btheta}_r^{(\ols)}$ and $\hat{\btheta}_p$ with proper choice of $\lambda^{(\textup{tl})}$. 
However, its convergence rate involves a bias term $\omega_f\delta$. If $\omega_f\delta\gg \sqrt{p/\tn_r}$, then
the transfer learning estimate has the same convergence rate as $\tilde{\btheta}_r^{(\ols)}$, which fails to borrow information from the pre-trained estimate.
When $\tn_r\ll N_r$, the convergence rate of three baseline methods are much slower than the rate of the re-trained estimator.

We now provide theoretical comparisons between $\hat{\btheta}^{(\uls)}_r$ and the benchmark estimators studied above. We see from Theorem \ref{thm2} and Theorem \ref{thm-minimax} that $\hat{\btheta}_r^{(\uls)}$ is no worse than the three baseline methods and is minimax optimal under mild conditions.
Comparing with the transfer learning estimate, especially, the bias term of ULS is scaled by $\sqrt{p/\tn_r}$, which is significantly smaller.

To summarize, the propose unlearning estimate provides a computationally efficient alternative to debias the pre-trained estimate which only leverages two small datasets $\calD_f$ and $\widetilde{\calD}_r$.

\section{Statistical inference}
\label{sec5}
{While current machine unlearning literature primarily focuses on approximating the point estimates of a fully retrained model, it frequently overlooks the uncertainty inherent in the post-deletion parameters. Establishing a robust framework for statistical inference is crucial to quantify this resulting variability, ensuring that downstream predictions and confidence intervals remain statistically valid after data removal.}

In this section, we develop asymptotically valid inference procedures for linear functionals of parameter $\btheta_r$.
For any given constant vector $\bm{v}\in\R^p$, we study statistical inference for $\bm{v}^{\top}\btheta_r$ enabling hypothesis testing and confidence interval construction without full retraining.  The main challenge is that the unlearning estimator depends on both the forget data and the covariates from subsampled remaining data, creating additional variability beyond the standard regression noise.

Let $\eps_i^{(r)}=y^{(r)}_i-(\bx_i^{(r)})^{\top}\btheta_r$. For $i=1,\dots, N_r$, define the noise terms
\[
a^{(r)}_i=\bm{v}^{\top}\Sig_r^{-1}\bx^{(r)}_i\eps^{(r)}_i~~\text{and}~~b_i^{(r)}=-\bm{v}^{\top}\Sig_r^{-1}(\bx^{(r)}_i(\bx^{(r)}_i)^{\top}-\Sig_r)(\btheta_r-\btheta_p),
\]
where $a_i^{(r)}$ captures the variance of the outcome noise and $b_i^{(r)}$ captures the additional variance introduced by approximating the population Hessian using the subsample $\widetilde{\calD}_r$.

\begin{condition}[Regularity conditions for asymptotic normality]
\label{cond-inf}
 There exists some positive constant $c$ such that $\min_{1\leq i\leq N_r}\textup{Var}(a_i^{(r)})\geq c\|\bm{v}\|_2^2$ and $\min_{1\leq i\leq N_r}\textup{Var}(b_i^{(r)})\geq c\omega_f^2\delta^2\|\bm{v}\|_2^2$.
There exists some positive constant $\rho$ such that
\begin{align*}
\max_{1\leq i\leq N_r}\frac{|\textup{Cov}(a^{(r)}_i,b_i^{(r)})|}{\textup{Var}^{1/2}(a_i^{(r)})\textup{Var}^{1/2}(b_i^{(r)})}
\leq \rho<1.
\end{align*}
\end{condition}
Condition \ref{cond-inf} puts mild conditions on the variance and covariance of $a_i^{(r)}$ and $b_i^{(r)}$ ensuring the non-degeneracy of the asymptotic variance. If $\E[\eps_i^{(r)}|\bx_i^{(r)}]=0$, then the above inequality holds with $\rho=0$.  As we do not assume a linear relationship between covariates and outcomes in (\ref{def-thetar}) and the observations are allowed to be heterogeneous within $\calD_r$, Condition \ref{cond-inf} guarantees that the variance patterns of the residuals are regular.
If $\bx_i^{(r)}$ and $\eps_i^{(r)}$ are independent and jointly normal with positive definite covariance matrix, then Condition \ref{cond-inf} automatically holds.  Let $\widetilde{\mathcal{N}}_r\subseteq[N_r]$ denote the index set corresponding to $\widetilde{\calD}_r$.
\begin{theorem}
\label{thm-inf}
Assume Conditions \ref{cond1} and \ref{cond-inf}, and $p^2=o(\min\{\tn_r,N_f\})$. For any fixed $\bm{v}\in\R^p$, it holds that
\begin{align*}
\frac{\bm{v}^{\top}(\hat{\btheta}^{(\uls)}_r-\btheta_r)}{V_r^{1/2}}\xrightarrow{D}N(0,1)
\end{align*}
for
\[
   V_r=\frac{1}{N_r^2}\sum_{i\in\widetilde{\mathcal{N}}_r}\E[(a^{(r)}_i+\frac{\tn_r-N_r}{\tn_r}b^{(r)}_i)^2]+\frac{1}{N_r^2}\sum_{i\in \mathcal{N}_r\setminus \widetilde{\mathcal{N}}_r}\E[(a^{(r)}_i+b^{(r)}_i)^2].
\]
\end{theorem}
In Theorem \ref{thm-inf}, we prove the asymptotic normality of $\bm{v}^{\top}\hat{\btheta}^{(\uls)}_r$ for any fixed $\bm{v}\in\R^p$.
As shown in the proof, the variance term $V_r= O( \|\bm{v}\|_2(N_r^{-1}+\omega_f^2\delta^2/\tn_r))$, which matches the convergence rate proved in Theorem \ref{thm2}. 
 
Next, we develop a consistent estimate of $V_r$. Since the second term in $V_r$ involves the samples in $\calD_r\setminus \widetilde{\calD}_r$, which are unobserved, we cannot compute it directly. Nevertheless, as $\widetilde{\calD}_r$ is a random sample from $\calD_r$, we can develop a consistent estimate based on $\widetilde{\calD}_r$.
Specifically, define
\[
  \widehat{V}_r=\frac{1}{N_r^2}\sum_{i\in\widetilde{\mathcal{N}}_r}(\hat{a}^{(r)}_i+\frac{\tn_r-N_r}{\tn_r}\hat{b}^{(r)}_i)^2+\frac{N_r-\tn_r}{N_r^2\tn_r}\sum_{i\in \widetilde{\mathcal{N}}_r}(\hat{a}^{(r)}_i+\hat{b}^{(r)}_i)^2,
\]
where
\begin{align*}
\hat{a}^{(r)}_i&=\bm{v}^{\top}(\widetilde{\Sig}_r)^{-1}\tilde{\bx}^{(r)}_i(\tilde{y}_i^{(r)}-(\tx_i^{(r)})^{\top}\hat{\btheta}_r^{(\uls)})\\
\hat{b}^{(r)}_i&=\bm{v}^{\top}(\widetilde{\Sig}_r)^{-1}(\tilde{\bx}^{(r)}_i(\tilde{\bx}^{(r)}_i)^{\top}-\widetilde{\Sig}_r)(\hat{\btheta}_r^{(\uls)}-\hat{\btheta}_p).
\end{align*}
We prove the consistency of $\widehat{V}_r$ in the next lemma.
\begin{lemma}
\label{lem-var}
Assume Conditions \ref{cond1}, \ref{cond-inf}, and $p^2=o(\min\{\tn_r,N_f\})$. For any fixed $\bm{v}\in\R^p$, it holds that
\[
   \frac{|\widehat{V}_r-V_r|}{V_r}=o(1)
\]
with probability at least $1-\exp\{-c_1\log \tn_r\}$.
\end{lemma}
In view of Theorem \ref{thm-inf} and Lemma \ref{lem-var}, we can construct asymptotically valid $100(1-\alpha)\%$-confidence interval for $\bm{v}^{\top}\btheta_r$ as
\begin{equation}
\label{eq-ci_uls}
    \left(\bm{v}^{\top}\hat{\btheta}_r^{(\uls)}-z_{(1-\alpha)/2}\widehat{V}^{1/2}_r,\bm{v}^{\top}\hat{\btheta}_r^{(\uls)}+z_{(1-\alpha)/2}\widehat{V}^{1/2}_r\right),
\end{equation}
where $z_{1-\alpha/2}$ is the standard normal quantile. 
Crucially, this interval is computable solely from the forget set, the pre-trained model, and the small subsample  $\widetilde{\calD}_r$, preserving the efficiency and privacy benefits of our framework.

\section{Robustifying the ULS estimator}
\label{sec6}
From Theorem \ref{thm2}, we see that the ULS estimator can exhibit a slower convergence rate than $\tilde{\btheta}_r^{(\ols)}$ if the forget model bias is overwhelmingly large $\omega_f\delta\gg 1$. In this section, we study a robust version of ULS that addresses this limitation and establish its theoretical connections to prevailing benchmark estimators used in LLMs. 
\subsection{Robustified ULS}
We robustify the ULS estimate by adding a loss term for the subsampled remaining data
\begin{align}
\label{ul-loss2}
\hat{\btheta}_r^{(\uls+)}=\argmin_{\btheta\in\R^p}\left\{-\frac{\omega_f}{N_f}\ell(\btheta;\calD_f)+(\hat{\btheta}_p-\btheta)^{\top}\widetilde{\Sig}_p(\hat{\btheta}_p-\btheta)+\frac{\lambda}{\tilde{n}_r}\ell(\btheta;\widetilde{\calD}_r)\right\},
\end{align}
where $\lambda>0$ is a tuning parameter.
In comparison to the optimization in (\ref{ul-loss}), the last term in (\ref{ul-loss2}) enforces that the solution $\hat{\btheta}_r^{(\uls+)}$ must also maintain a good fit to the subsampled remaining data $\widetilde{\calD}_r$.

\begin{theorem}[Convergence rate of robustified ULS]
\label{thm-combine}
Assume Condition \ref{cond1} and $p=o(\min\{\tn_r,N_f\})$. If we take $\lambda=C\omega_r\omega_f\delta$ for any positive constant $C$, then
\begin{align*}
\|\hat{\btheta}_r^{(\uls+)}-\btheta_r\|_2\leq C_1\sqrt{\frac{p}{N_r}}+C_2\min\{\omega_f\delta,1\}\sqrt{\frac{p}{\tn_r}}
\end{align*}
with probability $1-\exp\{-c_1p\}$.
\end{theorem}
Theorem \ref{thm-combine} demonstrates that incorporating the loss term on $\widetilde{\calD}_r$ successfully robustifies the unlearning estimator when $\omega_f\delta$ is large. Indeed, the estimate $\hat{\btheta}_r^{(\uls+)}$ is always no worse than the subsampled OLS $\tilde{\btheta}^{(\ols)}_r$. In the next theorem, we show the minimax optimality of $\hat{\btheta}^{(\uls)}_r$.

\begin{theorem}[Minimax lower bound]
\label{thm-minimax2}
Assume that $\omega_r\geq 1/2$ and  $p\leq c_0\min\{\sqrt{N},\tn_r\}$  for some small enough positive constant $c_0$. For the parameter space $\Theta_p(\delta)$ defined in (\ref{parameter-space}), it holds that
\[
   \inf_{\hat{\btheta}\in\mathcal{F}(\hat{\btheta}_p,\calD_f,\widetilde{\calD}_r)}\sup_{\Theta_p(\delta)} \P\left(\|\hat{\btheta}-\btheta_r\|_2\geq c_1\sqrt{\frac{p}{N_r}}+c_1\min\{\omega_f\delta,1\}\sqrt{\frac{p}{\tilde{n}_r}}\right)\geq c_2
\]
for some positive constants $c_1$ and $c_2$.
\end{theorem}
Theorem \ref{thm-minimax2} shows that the robustified ULS achieves minimax optimality regardless of the magnitude of $\omega_f\delta$.

\subsection{Connections to GradDiff}

The formulation in (\ref{ul-loss2}) is conceptually related to the Gradient Difference (GradDiff) method \citep{liu2022continual,liu2025rethinking}, a heuristic proposed to mitigate the well-known instability of standard gradient ascent in LLM unlearning tasks \citep{yao2024large,maini2024tofu}. Specifically, GradDiff simultaneously maximizes the forget loss and minimizes the remaining loss
\begin{align*}
\hat{\btheta}_r^{(\textup{diff})}=\argmin_{\btheta\in\R^p}\left\{-\frac{1}{N_f}\ell(\btheta;\calD_f)+ \frac{\lambda^{(\textup{diff})}}{\tn_r}\ell(\btheta;\widetilde{\calD}_r)\right\},
\end{align*}
where $\lambda^{(\textup{diff})}>0$ is a tuning parameter.

Empirical studies have demonstrated that the GradDiff has more reliable performance than gradient ascent in LLM unlearning \citep{fan2024simplicity}. 
In our framework, GradDiff can be viewed as a weighted average of the gradient ascent estimate and the OLS estimate $\tilde{\btheta}_r^{(\ols)}$ but it omits the Hessian-weighted regularization term that anchors the optimization to the pre-trained model. We provide convergence analysis of $\hat{\btheta}_r^{(\textup{diff})}$ in the next theorem.
\begin{theorem}[Convergence rate of GradDiff]
\label{thm-gd}
Assume Condition \ref{cond1} and $p=o(\min\{\tn_r,N_f\})$. For any $\lambda^{(\textup{diff})} \geq 2\Lambda_{\max}(\Sig_f)/\Lambda_{\min}(\Sig_r)$, we have with probability at least $1-\exp\{-c_1 p\}$,
\[
  \|\hat{\btheta}_r^{(\textup{diff})}-\btheta_r\|_2\leq c_2( \sqrt{\frac{p}{\tn_r}}+\frac{\sqrt{\frac{p}{N_f}}+\delta}{\lambda^{(\textup{diff})}}).
\]
Hence, if we take $\lambda^{(\textup{diff})}=\max\{\sqrt{\frac{\tilde{\omega}_r}{\omega_f}}+\sqrt{\frac{\tn_r}{p}}\delta,2\Lambda_{\max}(\Sig_f)/\Lambda_{\min}(\Sig_r)\}$,
\[
  \|\hat{\btheta}_r^{(\textup{diff})}-\btheta_r\|_2\leq c_3\sqrt{\frac{p}{\tn_r}}.
\]
\end{theorem}
In Theorem \ref{thm-gd}, the constraint on $\lambda^{(\textup{diff})}$ ensures the objective function remains strictly convex and admits a unique minimizer. We see from the first inequality that the risk upper bound decreases as $\lambda^{(\textup{diff})}$ gets larger. In the second inequality, with optimal tuning, its convergence rate is $O(\sqrt{p/\tn_r})$, which merely matches the subsampled OLS estimator that relies solely on the remaining data.
Therefore, unlike our proposed ULS framework, GradDiff is not minimax optimal when $\omega_f\delta=o(1)$. 
\section{Numerical experiments}
\label{sec-simu}


We conduct numerical experiments to assess the estimation and inference performance of the proposed unlearning estimator $\hat{\btheta}_r^{(\uls)}$ in comparison to several classical benchmark estimators.  The code for all the methods is available at \url{https://github.com/jyxie96/mu_minimax}.

For the remaining data, we simulate covariates $\bx_i^{(r)}\sim_{i.i.d.} N(0,I_p)$ and responses $y^{(r)}_i = (\bx_i^{(r)})^{\top}\btheta_r + \eps_i^{(r)}$, where the random noise $\eps_i^{(r)} \sim_{i.i.d.} N(0, 1)$. The true parameter $\btheta_r$ is randomly drawn from $N(0,I_p)$ and fixed in Monte Carlo replications. For the forget data, we introduce an autoregressive covariance structure for the covariates, $\bx_i^{(f)}\sim_{i.i.d.} N(0,\Sig_f)$, where $(\Sigma_f)_{j,k}=0.3^{|j-k|}$ , and generate $y^{(f)}_i = (\bx_i^{(f)})^{\top}\btheta_f + \eps_i^{(f)}$, where the random noise $\eps_i^{(f)} \sim_{i.i.d.} N(0, 1)$. We set $\btheta_f=\btheta_r + c_{\delta}  \mathbf{v}/\|\mathbf{v}\|_2$ for $\mathbf{v} = p^{-1/2}\mathbf{1}_p$, where $c_{\delta}=\delta$ varies in different settings.

\subsection{Estimation performance}

We first evaluate estimation with a fixed remaining sample size $N_r =20000$ and a forget sample size $N_f =1000$. 
To isolate the influence of key theoretical quantities, we employ a controlled design varying one factor at a time:
\begin{itemize}
\item[(a)] Subsample Ratio: We vary $\tilde{n}_r/N_r \in \{0.1, 0.2, 0.3\}$ with fixed $p=50$ and $\delta = 2$.
\item[(b)] Dimension: We vary $p \in \{10, 50, 100\}$ with fixed $\tilde{n}_r/N_r = 0.2$ and $\delta = 2$.
\item[(c)] Discrepancy: We vary the shift magnitude $\delta \in \{1, 2, 3\}$ with fixed $p=50$ and $\tilde{n}_r/N_r=0.2$.
\end{itemize}
We compare  $\hat{\btheta}^{(\uls)}_r$ against four baselines: the retrained estimate $\mathring{\btheta}_r^{(\rtr)}$, the pre-trained estimate $\hat{\btheta}_p$, the GradDiff estimate $\hat{\btheta}_r^{(\textup{diff})}$, and the subsampled OLS $\tilde{\btheta}_r^{(\ols)}$.
For the tuning parameter $\lambda$ in $\hat{\btheta}_r^{(\textup{diff})}$, we employ cross-validation to adaptively tune it, searching over the range $\lambda^{(\textup{diff})} \in \left[10^{-4}, 10^4\right]$. For an arbitrary estimator $\hat{\btheta}$, 
its performance is measured by the estimation error $\|\hat{\btheta} - \btheta_r\|_2$ averaged over 1,000 Monte Carlo replications.



\begin{figure}[!t]
\centering
\includegraphics[width=0.99\textwidth]{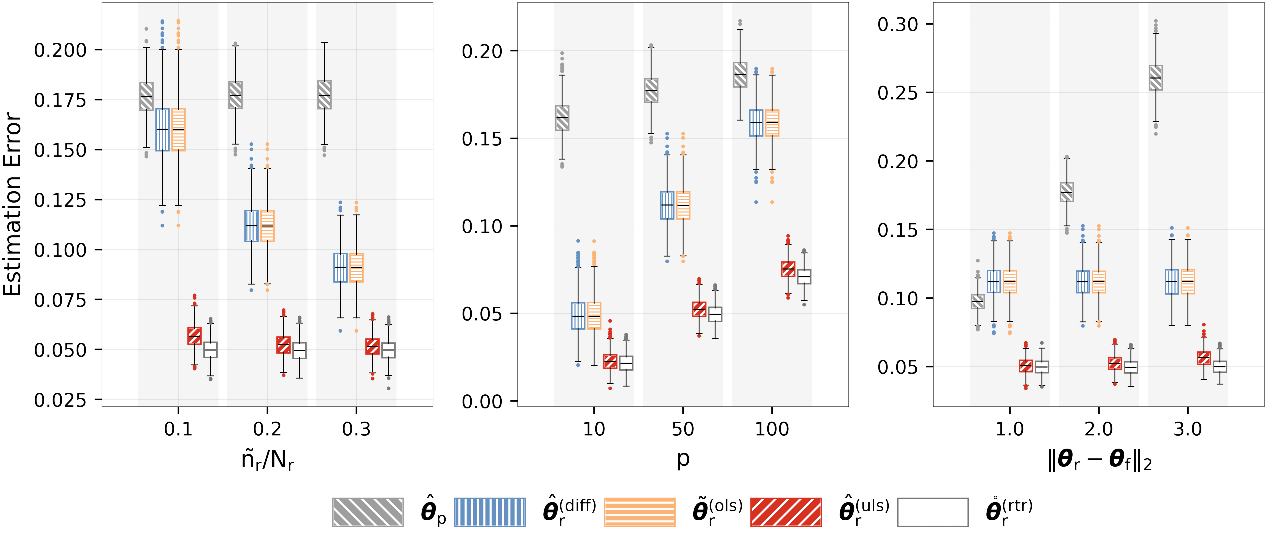}
\caption{Boxplots of the estimation errors for five unlearning estimators: $\mathring{\btheta}_r^{(\rtr)}$ (gray without hatch), $\hat{\btheta}_p$ (gray with backward-slash hatch), $\hat{\btheta}_r^{(\textup{diff})}$ (blue with vertical-line hatch), $\tilde{\btheta}_r^{(\ols)}$ (yellow with horizontal-line hatch), and $\hat{\btheta}_r^{(\uls)}$ (red with forward-slash hatch). The three panels correspond to experimental settings (a), (b), and (c), with $N_r =20000$ and $N_f =1000$. Each setting is replicated with 1000 Monte Carlo trials. \label{fig:first}}
\end{figure}

From Figure \ref{fig:first}, we see that our proposed estimates have estimation errors comparable to the re-trained estimates and are more accurate than other benchmark methods $\hat{\btheta}_p$, $\hat{\btheta}_r^{(\textup{diff})}$, and $\tilde{\btheta}^{(\ols)}_r$ across all the settings.
Moreover, the dependence of the estimation errors on the three factors align with our theoretical analysis. As $\tilde{n}_r$ increases, the estimation errors of $\hat{\btheta}_r^{(\textup{diff})}$, $\tilde{\btheta}_r^{(\ols)}$, and $\hat{\btheta}_r^{(\uls)}$ decay.  Increasing the dimension $p$ leads to a monotonic increase in error across all methods. As $\delta$ grows, both $\hat{\btheta}_r^{(\uls)}$ and $\hat{\btheta}_p$ have increasing errors, but the errors of $\hat{\btheta}_r^{(\uls)}$ grows much slower. This aligns with our theory that the effect of $\delta$ on the error of $\hat{\btheta}_r^{(\uls)}$ is attenuated by the shrinkage factor $\sqrt{p/\tilde{n}_r}$. On the other hand, $\tilde{\btheta}_r^{(\ols)}$ remains unaffected by the bias term as it only uses $\widetilde{\calD}_r$. Furthermore, with an adaptively tuned $\lambda^{(\textup{diff})}$, the estimation error of $\hat{\btheta}_r^{(\textup{diff})}$ becomes virtually indistinguishable from that of $\tilde{\btheta}_r^{(\ols)}$. This empirical result is consistent with our theoretical derivation in Theorem \ref{thm-gd}.
Overall, the empirical behavior matches the convergence rates established in the theory, demonstrating that the error decays with the effective subsample size and scales appropriately with the dimension, and model discrepancy.

These patterns remain consistent under a more aggressive unlearning scenario with a larger forget set $N_f=2000$ as shown in Figure~\ref{fig:second}.  We also report the simulation results comparing $\hat{\btheta}_r^{(\uls)}$ with its gradient descent variant $\hat{\btheta}_{r,T}^{(\uls)}$, and its robustified version $\hat{\btheta}_r^{(\uls+)}$. Comprehensive details and further analysis are provided in the supplements (Section \ref{sec:est_supp}).


\begin{figure}[!t]
\centering
\includegraphics[width=0.99\textwidth]{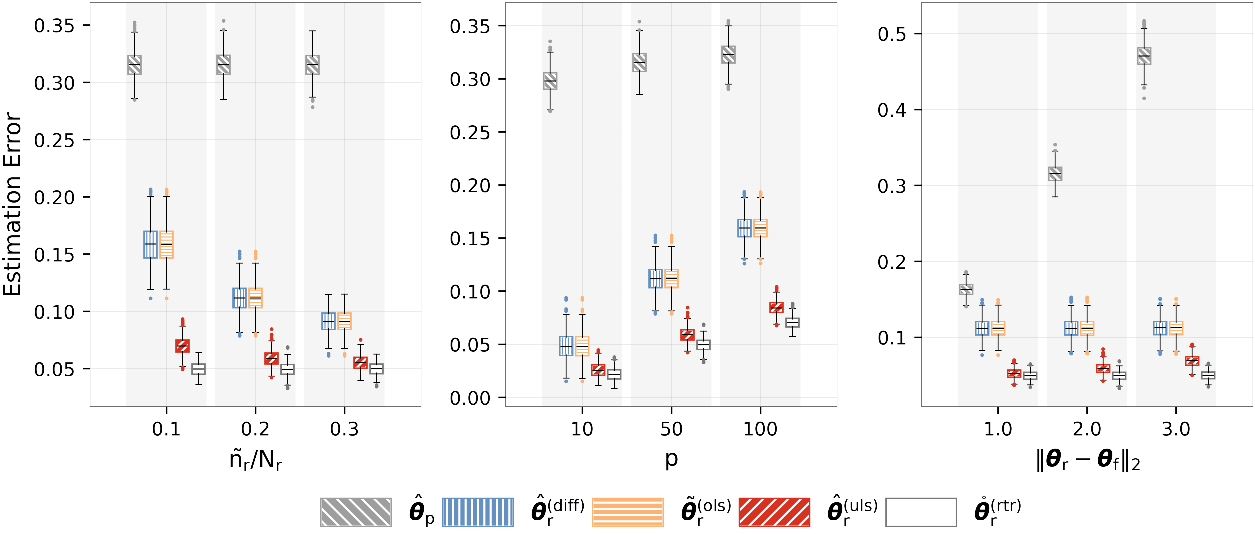}
\caption{Boxplots of the estimation errors for five unlearning estimators: $\mathring{\btheta}_r^{(\rtr)}$ (gray without hatch), $\hat{\btheta}_p$ (gray with backward-slash hatch), $\hat{\btheta}_r^{(\textup{diff})}$ (blue with vertical-line hatch), $\tilde{\btheta}_r^{(\ols)}$ (yellow with horizontal-line hatch), and $\hat{\btheta}_r^{(\uls)}$ (red with forward-slash hatch). The three panels correspond to experimental settings (a), (b), and (c), with $N_r =20000$ and $N_f =2000$. Each setting is replicated with 1000 Monte Carlo trials. }\label{fig:second}
\end{figure}

\subsection{Inference performance}
We evaluate the inference performance of our proposed unlearning estimator $\hat{\btheta}_r^{(\uls)}$ in comparison to  subsampled remaining estimator $\tilde{\btheta}_r^{(\ols)}$, which also enjoys asymptotic normality under mild conditions. For this OLS method based on $\widetilde{\calD}_r$, a $100(1-\alpha)\%$-confidence interval for $\bm{v}^{\top}\btheta_r$ can be constructed as
\begin{equation}
\label{eq-ci_ols}
\bm{v}^{\top}\tilde{\btheta}^{(\ols)}_r \pm z_{1-\alpha/2}\sqrt{\frac{\sum_{i=1}^{\tn_r}(\tilde{y}_i^{(r)} - (\tilde{\bx}_i^{(r)})^{\top}\tilde{\btheta}_r^{(\ols)})^2}{\tilde{n}_r-p}}\sqrt{\bm{v}^{\top}(\widetilde{X}_r^\top\widetilde{X}_r)^{-1}\bm{v}}
\end{equation}
We maintain the same model setup and experimental configuration as in the estimation experiments of Figure \ref{fig:first} or Figure \ref{fig:second}. For a fixed direction $\bm{v} = \bm{e}_1$, we construct $100(1-\alpha)\%$ confidence intervals with a nominal level $\alpha=0.05$. 
We report the average coverage probabilities and average standard deviations in each experimental configuration in Table \ref{tab:inf_res}.


From Table~\ref{tab:inf_res}, we see that when $\tn_r$ is small, both methods are slightly under coverage but they achieve nominal coverage in all the other scenarios.
 On the other hand, the average standard deviation of the proposed method is always smaller than that of $\tilde{\btheta}_r^{(\ols)}$. It demonstrates that the proposed inference method is asymptotically valid and achieves higher efficiency. In the supplements (Section \ref{sec:inf_supp}), we provide inference results when $N_f=2000$, which give analogous conclusions as in Table \ref{tab:inf_res}.

\begin{table}[t]
\centering
\resizebox{\textwidth}{!}{
\setlength{\tabcolsep}{8pt} 
\begin{tabular}{|c|ccc|ccc|ccc|}
\hline
      & \multicolumn{3}{c|}{$\tn_r/N_r$} 
      & \multicolumn{3}{c|}{$p$} 
      & \multicolumn{3}{c|}{$\delta$} \\

      & 0.1 & 0.2 & 0.3 & 10 & 50 & 100  & 1 & 2 & 3 \\
\hline
\multicolumn{10}{|c|}{\textbf{Average Coverage}} \\
\hline
ULS & 0.935 & 0.957 & 0.960 & 0.957 &  0.957 & 0.952 & 0.947 & 0.957 & 0.949 \\
\hline
OLS & 0.945 & 0.952 & 0.942 & 0.951 & 0.952  & 0.942  & 0.954 &  0.952 & 0.960 \\
\hline
\multicolumn{10}{|c|}{\textbf{Average SD} $(\times 10^{-2})$} \\
\hline
ULS & $0.81 $ & $0.75$ & $0.73$ & $0.74$ & $0.75$ & $0.76$  & $0.72$ & $0.75 $ & $0.80$ \\
\hline
OLS & $2.26$ & $1.59$ & $1.30$ & $1.58$ & $1.59$ & $1.60 $ & $1.59 $ & $1.59$ & $1.59$ \\
\hline
\end{tabular}
}
\caption{Inference results for settings (a), (b), and (c) with $N_r=20000$ and $N_f=1000$. Each setting is replicated with 1000 Monte Carlo trials. 
\label{tab:inf_res}}
\end{table}

\section{Real data applications}
\label{sec-data}
In this section, we apply the proposed unlearning procedure to two datasets: the Yelp review public dataset~\citep{zhang2015character} and UK Biobank clinical records~\citep{sudlow2015uk}, for the purpose of outlier removal.  We provide the code for this analysis at \url{https://github.com/jyxie96/mu_minimax/}.

 \subsection{Yelp dataset unlearning}
 The Yelp review public dataset (\url{https://huggingface.co/datasets/Yelp/yelp_review_full}) comprises over six million user reviews and associated ratings. Our objective is to predict ratings based on the textual content of the reviews. This dataset was previously utilized for an unlearning task by \citet{izzo2021approximate}. For computational feasibility, we begin by drawing a random sample of 200,000 reviews from the full corpus. 
 
We define forget set based on the length of the reviews. Very short reviews often lack meaningful context, while excessively long reviews can be overly convoluted, making both extremes potential sources of predictive noise. Therefore, we define the forget set to be the reviews whose lengths fall into the bottom 10\% and top 10\% quantiles. These reviews constitute targeted outliers, and our objective is to unlearn their specific influence on the model. The remaining reviews constitute the dataset to be retained, which we randomly partition into two distinct subsets: 80\% forms the remaining dataset $\calD_r$ and the other 20\% is held out as an independent test set for evaluating predictive performance. Accordingly, the full set of training data is $\calD=\calD_r \cup \calD_f$ with $N_r=129,144$ and $N_f=38,569$.  For the subsampled remaining data $\widetilde{\calD}_r$, we randomly draw $\tn_r$ reviews from $\calD_r$ with $\tn_r/N_r=0.1$. 

For text representation, following the methodology in~\citet{izzo2021approximate},  we use a separate sample of reviews outside the unlearning dataset to construct a vocabulary of 1,500 most frequent words. We then employ this vocabulary to build a bag-of-words representation, yielding a $p=1500$ dimensional feature vector for each review. 

We evaluate the prediction performance of the proposed $\hat{\btheta}_r^{(\uls)}$ and $\hat{\btheta}_r^{(\uls+)}$ against four benchmarks: the retrained estimate $\mathring{\btheta}_r^{(\rtr)}$, the pre-trained estimate $\hat{\btheta}_p$, the GradDiff estimate $\hat{\btheta}_r^{(\textup{diff})}$, the subsampled OLS $\tilde{\btheta}_r^{(\ols)}$ and the so-called projective residual update (PRU) estimate introduced in \citet{izzo2021approximate}. Performance is evaluated using the Mean Prediction Error (MPE) on the test set: $\sum_{i=1}^{n_\text{test}}(y_i-\boldsymbol{x}_i^{\top}\hat{\btheta})^2 /n_\text{test}$ for a generic estimator $\hat{\btheta}$.


The prediction results are reported in Figure \ref{fig:third} and the running time cost is reported in Table \ref{tab:time}. Notably, our proposed ULS and ULS+ methods achieve prediction accuracy closely tracking the oracle re-trained estimator, while maintaining high computational efficiency. In contrast, the other unlearning benchmarks, $\hat{\btheta}_p$, $\hat{\btheta}_r^{(\textup{diff})}$, and $\tilde{\btheta}_r^{(\ols)}$ suffer large prediction errors due to the large model discrepancy or the relatively small subsample size. $\hat{\btheta}^{(\textup{pru})}_r$ has comparable performance as $\mathring{\btheta}_r^{(\rtr)}$ when the number of forgotten samples is large, but at a substantial computational cost. Even with matrix acceleration, it is still nearly $300$ times slower than the computation time of other estimators. Indeed, it has been discussed in \citet{izzo2021approximate} that $\hat{\btheta}_r^{(\textup{pru})}$ requires computations involving $N_f^2p$ and $N_f^3$ terms, making it impractical when $N_f$ is large. To provide a comprehensive evaluation, we present corresponding results for $\tilde{n}_r/N_r\in\{0.2, 0.3\}$ in Figures~\ref{fig:yelp-1-supp} and \ref{fig:yelp-2-supp} of the supplements.

\begin{table}[t]
\centering
\small
\setlength{\tabcolsep}{8pt} 
\begin{tabular}{cccccccc}
\hline
 Estimator & $\mathring{\btheta}_r^{(\rtr)}$ & $\hat{\btheta}_r^{(\textup{pru})}$ &$\hat{\btheta}_r^{(\textup{diff})}$ &$\tilde{\btheta}_r^{(\ols)}$ &$\hat{\btheta}_r^{(\uls)}$ & $\hat{\btheta}_r^{(\uls+)}$\\
\hline
Running Time (s)&0.44 & 90.08 & 25.60 & 0.12 & 0.25 & 25.20 \\
\hline
\end{tabular}
\caption{Average running time per estimator over 20 folds on the Yelp dataset. For $\hat{\btheta}_r^{(\textup{diff})}$ and $\hat{\btheta}_r^{(\uls+)}$, running time includes 5-fold cross-validation over 20 candidates for tuning parameter selection. All experiments were conducted on an Intel Xeon Gold 5418Y CPU with 48 physical cores.
\label{tab:time}}
\end{table}

\begin{figure}[!t]
\centering
\includegraphics[width=0.6\textwidth]{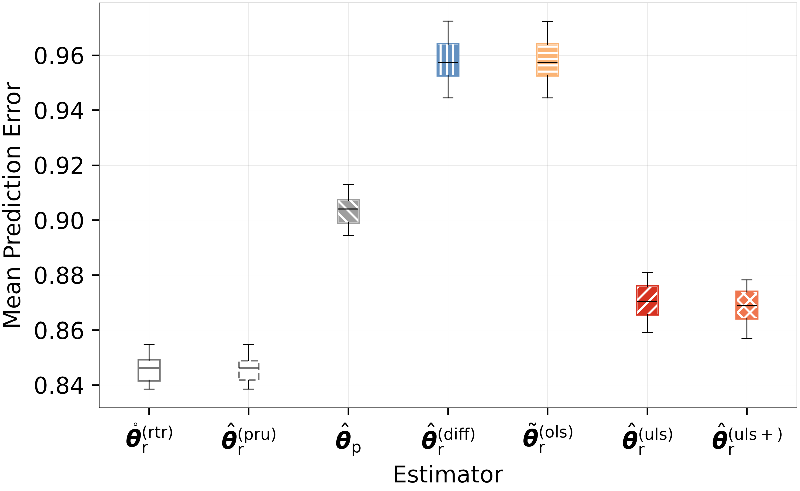}
\caption{Boxplots of mean prediction errors for the Yelp review rating unlearning task under the setting $\tilde{n}_r/N_r=0.1$. Each boxplot is based on 20 random splits of training and test samples.}
\label{fig:third}
\end{figure}


\subsection{UK Biobank dataset unlearning}
Given the pressing legal and ethical mandates for machine unlearning in medical informatics, we apply our framework to analyze inpatient lengths of stay using hospital episode records in 2022 from the UK Biobank repository~\citep{sudlow2015uk}. In clinical settings, predictive models for inpatient length of stay are crucial for operational resource planning and bed management. The predictors include admission method, management strategy, operative status, and responsible clinician specialty, which are highly informative for predicting duration.

Before constructing predictive model, we perform initial data cleaning by removing missing values and duplicate records to ensure data quality and integrity, resulting in a final dataset of $7,602$ records and 6 categorical predictors. These covariates are pre-processed to a design matrix with $p=16$ columns. See Section \ref{supp-data} in the supplements for more details.
For the response variable $y$, it is highly right-skewed as shown in the left panel of Figure~\ref{fig:logy} and we apply a $\log_{10}(y)$ transformation for reliable regression modeling.

Hospital episode records frequently include atypical or administratively irregular cases.  As the left panel of Figure~\ref{fig:logy} illustrates, certain episodes report extremely prolonged durations that do not represent routine inpatient trajectories. To systematically align the model with typical admissions, we employ the standard Interquartile Range (IQR) rule \citep{dekking2005modern} to define the outlier forget set. Specifically, we compute $\text{IQR}=q_3-q_1$ where $q_1$ and $q_3$ are the empirical quantiles of the original hospital stay duration. Episodes satisfying $y_i > q_3 + 1.5\,\text{IQR}$ are assigned to $\mathcal{D}_f$ (corresponding to a threshold of 38 days). The remaining records constitute the data to be retained. For robust evaluation, we employ the same partitioning and replication scheme as the Yelp dataset to construct the remaining set $\calD_r$ and the independent test set. This results in a forget set size of $N_f=577$ and remaining set  size of $N_r=5619$. 

\begin{figure}[!t]
\centering
\includegraphics[width=0.6\textwidth]{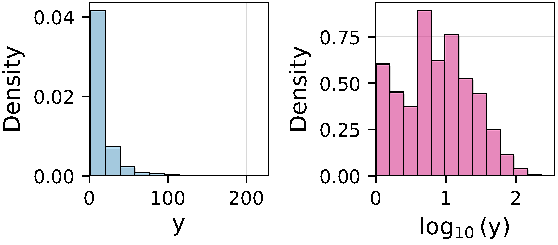}
\caption{Distribution of the response variable $y$ within the full dataset $\calD$. The left panel shows the distribution of raw response, which is severely right-skewed. The right panel displays the transformed response $\log_{10}(y)$, illustrating a significantly reduced scale.
\label{fig:logy}}
\end{figure}
We evaluate the predictive performance using the MPE of the log-transformed responses on the independent test set: $\sum_{i=1}^{n_\text{test}}(\log_{10}(y_i)-\boldsymbol{x}_i^{\top}\hat{\btheta})^2 /n_\text{test}$ for a generic estimator $\hat{\btheta}$. The results are reported in Figure~\ref{fig:fourth}. Our proposed ULS and ULS+ outperform the three baselines methods, demonstrating the lowest predictive error. Additional performance evaluations for $\tilde{n}_r/N_r \in \{0.2, 0.3\}$ are available in Figures~\ref{fig:ukb-uls-supp1} and \ref{fig:ukb-uls-supp2} of the supplement.

\begin{figure}[!t]
\centering
\includegraphics[width=0.6\textwidth]{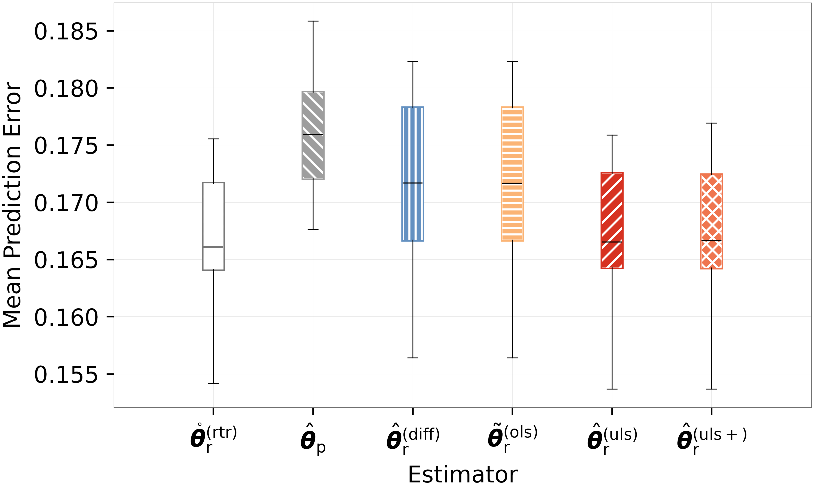}
\caption{Boxplots of mean prediction errors for the UK Biobank hospital episode unlearning task under the setting $\tilde{n}_r/N_r=0.1$.  Each boxplot is based on 20 random splits of training and test samples.
\label{fig:fourth}}
\end{figure}
We further conduct statistical inference for the coefficient corresponding to the ``one or more operative procedures performed" covariate. We compare the 95\% confidence intervals produced by our ULS method defined in (\ref{eq-ci_uls}) with that of OLS method based on $\widetilde{\calD}_r$ define in (\ref{eq-ci_ols}). As reported in Table~\ref{tab:ukb_inf}, both methods demonstrate high statistical significance, yielding strictly positive intervals that exclude zero. This indicates a positive association between operative procedures and hospital length of stay, which aligns perfectly with established clinical expectations. Crucially, the proposed ULS method exhibits superior statistical efficiency, yielding signficantly narrower confidence intervals compared to the subsampled OLS approach, thereby allowing for more precise clinical insights.

\begin{table}[t]
\centering
\small
\setlength{\tabcolsep}{8pt} 
\begin{tabular}{cccc}
\hline
 $\tn_r/N_r$ & 0.1 & 0.2 & 0.3 \\
\hline
ULS & $0.158\pm0.029$ & $0.163\pm0.026$ & $0.159\pm0.025$ \\
OLS & $0.181\pm0.078$ & $0.172\pm0.053$ & $0.147\pm0.044$ \\
\hline
\end{tabular}
\caption{Inference results for the coefficient of the ``one or more operative procedures performed'' covariate. The table reports the 95\% confidence intervals produced by ULS and OLS method across $\tn_r/N_r \in \{0.1, 0.2, 0.3\}$.
\label{tab:ukb_inf}}
\end{table}

\section{Discussion}
\label{sec-diss}
In this work, we provide a rigorous statistical framework for machine unlearning under generic loss functions. We develop computationally efficient methods for exact bias correction under distribution shifts and formally establish their minimax optimality. Furthermore, by proving the asymptotic normality of the unlearning estimator, we enable valid statistical inference and uncertainty quantification without the prohibitive cost of full retraining.   

The core methodology introduced in this work relies on the smoothness of loss functions, adapting this approach to non-smooth optimization landscapes (e.g., Lasso or quantile regression) presents a substantial technical challenge. Looking forward, developing minimax-optimal unlearning procedures for non-smooth problems, nonparametric models, and deep neural networks remains a vital frontier for future research. 


\vspace{-0.1in}

\section*{Supplementary materials}

Supplement to ``Efficient machine unlearning with minimax optimality''.
We provide the proofs of theorems and further results on simulations and data applications.

\vspace{-0.1in}

\section*{Competing interests}
No competing interest is declared.

\vspace{-0.1in}

\section*{Acknowledgments}
SL is supported in part by funds from the National Natural Science Foundation of China (No. 12571314). LZ is partly supported by ``AI for Math" Fund and NSF CAREER DMS-2340241.   
\bibliographystyle{abbrvnat}

\bibliography{mu-jrssb.bib}

\clearpage

\clearpage
\appendix

\begin{center}
{\large \bfseries Supplementary Material\par}
\vspace{0.3em}
\end{center}

\section{Convergence analysis of Algorithm \ref{alg-gd}}
\label{ap-loss}
For any $\btheta\in\R^p$, define $B(\btheta,d)=\{\btheta'\in\R^p:\|\btheta'-\btheta\|_2\leq d\}$.
\begin{condition}[$\lambda$-strong convexity]
\label{cond1-supp}
There exists some $\lambda>0$ such that  with probability at least $1-\exp\{-c_1\log \tn_r\}$,
\[
  \ell(\btheta;\widetilde{\calD}_r)-\ell(\btheta';\widetilde{\calD}_r)-\langle \dot{\ell}(\btheta';\widetilde{\calD}_r),\btheta-\btheta'\rangle\geq \frac{\lambda\tn_r}{2}\|\btheta-\btheta'\|_2^2
\]
for all $\btheta,\btheta'\in B(\btheta_r,d)$ with some constant $d>0$.
\end{condition}
\begin{condition}[$\mu$-smoothness]
\label{cond2-supp}
There exists some $\mu>0$ such that  with probability at least $1-\exp\{-c_1\log \tn_r\}$,
\[
  \ell(\btheta;\widetilde{\calD}_r)-\ell(\btheta';\widetilde{\calD}_r)-\langle \dot{\ell}(\btheta';\widetilde{\calD}_r),\btheta-\btheta'\rangle\leq \frac{\mu\tn_r}{2}\|\btheta-\btheta'\|_2^2
\]
for all $\btheta,\btheta'\in B(\btheta_r,d)$ with some constant $d>0$.
\end{condition}
\begin{condition}[Gradient smoothness]
\label{cond5-supp}
With probability at least $1-\exp\{-c_1\log \tn_r\}$,
\[
   \|\dot{\ell}(\btheta;\widetilde{\calD}_r)-\dot{\ell}(\btheta';\widetilde{\calD}_r)\|_2\leq C\tn_r\|\btheta-\btheta'\|_2
\]
for all $\btheta,\btheta'\in B(\btheta_r,d)$ with some constant $d>0$.
\end{condition}
\begin{condition}[Sub-exponential gradient vector]
\label{cond3-supp}
Assume that
$\dot{\ell}(\btheta_p;\bx^{(r)}_i,y^{(r)}_i)-\dot{\ell}(\btheta_r;\bx^{(r)}_i,y^{(r)}_i)$, $i=1,\dots,N_r$
are independent sub-Gaussian vector with sub-exponential norm bounded by $C\|\btheta_p-\btheta_r\|_2$. Moreover, assume $\|\btheta_p-\btheta_r\|_2\leq C$ for some positive constant $C$.
\end{condition}
\begin{condition}[Convergence rate of full-sample estimate]
\label{cond4-supp}
With probability at least $1-\exp\{-c_1\log \tn_r\}$,
\[
  \|\hat{\btheta}_p-\btheta_p\|_2\leq c_2\sqrt{\frac{p+\log \tn_r}{N}}~\text{and}~\|\mathring{\btheta}_r^{(\rtr)}-\btheta_r\|_2\leq c_2\sqrt{\frac{p+\log\tn_r}{N_r}}
\]
for some positive constants $c_1$ and $c_2$. 
\end{condition}
Conditions \ref{cond1-supp}-\ref{cond5-supp} are standard conditions for establishing the error contraction of a generic gradient descent algorithm \citep{balakrishnan2017statistical,nesterov2018lectures}. Conditions \ref{cond3-supp} and \ref{cond4-supp} are the regularity conditions for the establishing the convergence rate of the limit $\hat{\btheta}^{(\unl)}_{r,T}$ when $T\rightarrow\infty$.
 For squared loss and cross-entropy loss, these conditions are satisfied under standard sub-Gaussian conditions on the covariates.
\begin{theorem}
\label{thm-ap1}
Assume Condition \ref{cond1-supp}-\ref{cond4-supp}. For step size $\alpha=\frac{2}{N_r(\lambda+\mu)}$, it holds that
\[
  \|\hat{\btheta}_{r,t}^{(\unl)}-\mathring{\btheta}_r^{(\rtr)}\|_2\leq \frac{|\mu-\lambda|^t}{|\lambda +\mu|^t}\|\btheta_r-\btheta_p\|_2+C\|\btheta_r-\btheta_p\|_2\sqrt{\frac{p}{\tn_r}}+C'\sqrt{\frac{p}{N_r}}
\]
for some positive constant $C$ with probability at least $1-\exp\{-c_1\min\{p,\log\tn_r\}\}$.
\end{theorem}
\begin{proof}[Proof of Theorem \ref{thm-ap1}]
Let $\hat{\bu}_{t}=\hat{\btheta}_{r,t}^{(\unl)}-\mathring{\btheta}_r^{(\rtr)}$.
By (\ref{eq-gd}),
\begin{align*}
\hat{\bu}_{t}&=\hat{\bu}_{t-1}-\alpha\{\frac{N_r}{\tn_r}\dot{\ell}(\hat{\btheta}^{(\textup{unl})}_{r,t-1};\widetilde{\calD}_r)-\frac{N_r}{\tn_r}\dot{\ell}(\hat{\btheta}_p;\widetilde{\calD}_r)-\dot{\ell}(\hat{\btheta}_p;\calD_f)\}\\
&=\underbrace{\hat{\bu}_{t-1}-\alpha\frac{N_r}{\tn_r}\{\dot{\ell}(\hat{\btheta}^{(\textup{unl})}_{r,t-1};\widetilde{\calD}_r)-\dot{\ell}(\mathring{\btheta}_r^{(\rtr)};\widetilde{\calD}_r)\}}_{\widehat{B}_{t-1}}-\underbrace{\alpha\{\frac{N_r}{\tn_r}\dot{\ell}(\mathring{\btheta}_r^{(\rtr)};\widetilde{\calD}_r)-\frac{N_r}{\tn_r}\dot{\ell}(\hat{\btheta}_p;\widetilde{\calD}_r)-\dot{\ell}(\hat{\btheta}_p;\calD_f)\}}_{\widehat{E}}.
\end{align*}
Therefore,
\[
   \|\hat{\bu}_{t}\|_2\leq \|\widehat{B}_{t-1}\|_2+\|\widehat{E}\|_2.
\]
For the first term,
\begin{align}
\|\widehat{B}_{t-1}\|_2^2=\|\hat{\bu}_{t-1}\|_2^2+\frac{\alpha^2N_r^2}{\tn_r^2}\|\dot{\ell}(\hat{\btheta}^{(\textup{unl})}_{r,t-1};\widetilde{\calD}_r)-\dot{\ell}(\mathring{\btheta}^{(\rtr)};\widetilde{\calD}_r)\|_2^2-2\alpha\frac{N_r}{\tn_r}\langle\hat{\bu}_{t-1},\dot{\ell}(\hat{\btheta}^{(\textup{unl})}_{r,t-1};\widetilde{\calD}_r)-\dot{\ell}(\mathring{\btheta}_r^{(\rtr)};\widetilde{\calD}_r)\rangle.\label{oracle-inequality}
\end{align}
To apply Conditions \ref{cond1-supp} and \ref{cond2-supp}, we first show that
\begin{align}
\label{event1-supp}
   \P(\hat{\btheta}^{(\textup{unl})}_{r,t-1},\mathring{\btheta}_r^{(\rtr)}\in B(\btheta_r,d),~\forall t=1,2,\dots)\geq 1-\exp\{-c_1\log \tn_r\}.
\end{align}
By Condition \ref{cond4-supp}, it is easy to see $\mathring{\btheta}_r^{(\rtr)}\in B(\btheta_r,d)$ with high probability.
By induction, we only need to show that $\hat{\btheta}^{(\textup{unl})}_{r,0}\in B(\btheta_r,d)$ due to the error contraction property developed below. As $\hat{\btheta}^{(\textup{unl})}_{r,0}=\hat{\btheta}_p$, we know that
\[
   \|\hat{\btheta}^{(\textup{unl})}_{r,0}-\btheta_r\|_2\leq \|\hat{\btheta}_p-\btheta_p\|_2+\|\btheta_p-\btheta_r\|_2\leq C+o(1)
\]
with probability at least $1-\exp\{-c_1\log \tn_r\}$. Hence, for $d\geq C+1$, we have shown (\ref{event1-supp}).

In this event, Conditions \ref{cond1-supp} and \ref{cond2-supp} can be applied and by classical results such as \citet{nesterov2018lectures}, with probability at least $1-\exp\{-c_2\log\tn_r\}$,
\begin{align*}
\langle\hat{\bu}_{t-1},\dot{\ell}(\hat{\btheta}^{(\textup{unl})}_{r,t-1};\widetilde{\calD}_r)-\dot{\ell}(\mathring{\btheta}_r^{(\rtr)};\widetilde{\calD}_r)\rangle\geq \frac{\lambda \mu \tn_r}{\lambda +\mu}\|\hat{\bu}_{t-1}\|_2^2+\frac{1}{\tn_r(\lambda+\mu)}\|\dot{\ell}(\hat{\btheta}^{(\textup{unl})}_{r,t-1};\widetilde{\calD}_r)-\dot{\ell}(\mathring{\btheta}_r^{(\rtr)};\widetilde{\calD}_r)\|_2^2.
\end{align*}
Therefore,
\begin{align*}
\|\widehat{B}_{t-1}\|_2^2\leq (1-\frac{2\alpha\lambda\mu N_r}{\lambda +\mu})\|\hat{\bu}_{t-1}\|_2^2+(\alpha^2\frac{N_r^2}{\tn_r^2}-\frac{2\alpha N_r}{(\lambda+\mu)\tn^2_r})\|\dot{\ell}(\hat{\btheta}^{(\textup{unl})}_{r,t-1};\widetilde{\calD}_r)-\dot{\ell}(\mathring{\btheta}_r^{(\rtr)};\widetilde{\calD}_r)\|_2^2
\end{align*}
Hence, for $\alpha=\frac{2}{N_r(\lambda+\mu)}$,
\begin{align*}
\|\widehat{B}_{t-1}\|_2^2\leq \frac{(\mu-\lambda)^2}{(\lambda +\mu)^2}\|\hat{\bu}_{t-1}\|_2^2.
\end{align*}

As $\dot{\ell}(\hat{\btheta}_p;\calD_r)+\dot{\ell}(\hat{\btheta}_p;\calD_f)=0$, we have
\begin{align*}
\widehat{E}&=\alpha\{\frac{N_r}{\tn_r}\dot{\ell}(\mathring{\btheta}_r^{(\rtr)};\widetilde{\calD}_r)-\frac{N_r}{\tn_r}\dot{\ell}(\hat{\btheta}_p;\widetilde{\calD}_r)+\dot{\ell}(\hat{\btheta}_p;\calD_r)\}\\
&=\frac{\alpha N_r}{\tn_r}\dot{\ell}(\mathring{\btheta}_r^{(\rtr)};\widetilde{\calD}_r)-\frac{N_r}{\tn_r}\dot{\ell}(\hat{\btheta}_p;\widetilde{\calD}_r)-\alpha\{\dot{\ell}(\mathring{\btheta}_r^{(\rtr)};\calD_r)-\dot{\ell}(\hat{\btheta}_p;\calD_r)\}\\
&=\underbrace{\frac{\alpha N_r}{\tn_r}\dot{\ell}(\btheta_r;\widetilde{\calD}_r)-\frac{\alpha N_r}{\tn_r}\dot{\ell}(\btheta_p;\widetilde{\calD}_r)-\alpha\dot{\ell}(\btheta_r;\calD_r)+\alpha\dot{\ell}(\btheta_p;\calD_r)}_{\widehat{E}_1}+ rem_E.
\end{align*}
By Condition \ref{cond3-supp},  by Bernstein's inequality
\[
   \P\left(\|\widehat{E}_1\|_2\geq C\alpha\|\btheta_r-\btheta_p\|_2\sqrt{\frac{p}{\tn_r}}\right)\leq \exp\{-c_1 p\}.
\]
For the remainder term $rem_{E}$, by Condition \ref{cond5-supp},
\begin{align*}
&\frac{N_r}{\tn_r}\|\dot{\ell}(\mathring{\btheta}_r^{(\rtr)};\widetilde{\calD}_r)-\dot{\ell}(\btheta_r;\widetilde{\calD}_r)\|_2\leq CN_r\|\mathring{\btheta}_r^{(\rtr)}-\btheta_r\|_2\\
&\|\dot{\ell}(\mathring{\btheta}_r^{(\rtr)};\calD_r)-\dot{\ell}(\btheta_r;\calD_r)\|_2\leq CN_r\|\mathring{\btheta}_r^{(\rtr)}-\btheta_r\|_2\\
&\frac{N_r}{\tn_r}\|\dot{\ell}(\hat{\btheta}_p;\widetilde{\calD}_r)-\dot{\ell}(\btheta_p;\widetilde{\calD}_r)\|_2\leq CN_r\|\hat{\btheta}_p-\btheta_p\|_2\\
&\|\dot{\ell}(\hat{\btheta}_p;\calD_r)-\dot{\ell}(\btheta_p;\calD_r)\|_2\leq CN_r\|\hat{\btheta}_p-\btheta_p\|_2.
\end{align*}
Hence, with probability at least $1-\exp\{-c_1p\}$,
\begin{align*}
  \|rem_E\|_2&\leq C\alpha N_r ( \|\hat{\btheta}_p-\btheta_p\|_2+\|\mathring{\btheta}_r^{(\rtr)}-\btheta_r\|_2)\\
  &\leq C \sqrt{\frac{p}{N_r}}.
\end{align*}
To summarize, with probability at least $1-\exp\{-c_1p\}$,
\begin{align*}
\|\hat{\bu}_t\|_2\leq \frac{|\mu-\lambda|}{|\lambda +\mu|}\|\hat{\bu}_{t-1}\|_2+C\|\btheta_r-\btheta_p\|_2\sqrt{\frac{p}{\tn_r}}+C'\sqrt{\frac{p}{N_r}}.
\end{align*}
By induction,
\begin{align*}
\|\hat{\bu}_t\|_2\leq \frac{|\mu-\lambda|^t}{|\lambda +\mu|^t}\|\btheta_p-\btheta_r\|_2+\frac{1-\frac{|\mu-\lambda|^t}{|\lambda +\mu|^t}}{1-\frac{|\mu-\lambda|}{|\lambda +\mu|}}(C\|\btheta_r-\btheta_p\|_2\sqrt{\frac{p}{\tn_r}}+C'\sqrt{\frac{p}{N_r}})
\end{align*}
\end{proof}

\section{Proofs}
Define the empirical covariance matrix and marginal statistics for each data set
\begin{align*}
\widehat{\Sig}_r=\frac{1}{N_r}(X^{(r)})^{\top}X^{(r)},~\widehat{M}_r=\frac{1}{N_r}(X^{(r)})^{\top}\by^{(r)},\widehat{E}_r=\frac{1}{N_r}(X^{(r)})^{\top}(\bm{y}^{(r)}-X^{(r)}\btheta_r)\\
\widehat{\Sig}_f=\frac{1}{N_f}(X^{(f)})^{\top}X^{(f)},~\widehat{M}_f=\frac{1}{N_f}(X^{(f)})^{\top}\by^{(f)},\widehat{E}_f=\frac{1}{N_f}(X^{(f)})^{\top}(\bm{y}^{(f)}-X^{(f)}\btheta_f)\\
\widetilde{\Sig}_r=\frac{1}{\tilde{n}_r}(\tilde{X}^{(r)})^{\top}\tilde{X}^{(r)},~\widetilde{M}_r=\frac{1}{\tilde{n}_r}(\tilde{X}^{(r)})^{\top}\tilde{\by}^{(r)},\widetilde{E}_r=\frac{1}{\tn_r}(\tilde{X}^{(r)})^{\top}(\tilde{\bm{y}}^{(r)}-\tilde{X}^{(r)}\btheta_r).
\end{align*}
\subsection{Proof of (\ref{ul-loss})}
\label{sec-proof1}
By (\ref{unlearn-est}),
\begin{align*}
\widetilde{\Sig}_r\hat{\btheta}^{(\uls)}_r=\widetilde{\Sig}_r\hat{\btheta}_p-\frac{\omega_f}{\omega_r}(\widehat{M}_f-\widehat{\Sig}_f\hat{\btheta}_p).
\end{align*}
We know that for a loss function
\[
   \btheta^{\top}A\btheta-2\btheta^{\top}\bm{b},
\]
where $A$ and $\bm{b}$ are given, the first-order condition gives $A\hat{\btheta}=\bm{b}$.
Therefore, the loss function of $\hat{\btheta}^{(\uls)}_r$ is
\begin{align*}
&\btheta^{\top}\widetilde{\Sig}_r\btheta-2\btheta^{\top}(\widetilde{\Sig}_r\hat{\btheta}_p-\frac{\omega_f}{\omega_r}(\widehat{M}_f-\widehat{\Sig}_f\hat{\btheta}_p))\\
=&\frac{1}{\omega_r}(\omega_r\btheta^{\top}\widetilde{\Sig}_r\btheta-2\btheta^{\top}(\omega_r\widetilde{\Sig}_r+\omega_f\widehat{\Sig}_f)\hat{\btheta}_p+2\omega_f\btheta^{\top}\widehat{M}_f)\\
\propto&\btheta^{\top}(\omega_r\widetilde{\Sig}_r+\omega_f\widehat{\Sig}_f)\btheta-2\btheta^{\top}(\omega_r\widetilde{\Sig}_r+\omega_f\widehat{\Sig}_f)\hat{\btheta}_p+\omega_f(2\btheta^{\top}\widehat{M}_f-\btheta^{\top}\widehat{\Sig}_f\btheta)\\
\propto &-\frac{\omega_f}{N_f}\ell(\btheta;\calD_f)+(\hat{\btheta}_p-\btheta)^{\top}(\omega_r\widetilde{\Sig}_r+\omega_f\widehat{\Sig}_f)(\hat{\btheta}_p-\btheta).
\end{align*}
\subsection{Proof of Lemma \ref{tlem2}}
\begin{proof}[Proof of Lemma \ref{tlem2}]
We only need to show the first and the fourth results. The second and third results follow from standard OLS theory.

For $\widehat{\Sig}_p=\omega_f\widehat{\Sig}_f+\omega_r\widehat{\Sig}_r$, we have
\begin{align*}
\hat{\btheta}_p-\btheta_r&=\widehat{\Sig}_p^{-1}(\omega_f\widehat{M}_f+\omega_r\widehat{M}_r-\widehat{\Sig}_p\btheta_r)\\
&=\widehat{\Sig}_p^{-1}(\omega_f\widehat{M}_f+\omega_r\widehat{E}_r-\omega_f\widehat{\Sig}_f\btheta_r)\\
&=\omega_f\widehat{\Sig}_p^{-1}(\widehat{M}_f-\widehat{\Sig}_f\btheta_r)+\omega_r\widehat{\Sig}_p^{-1}\widehat{E}_r\\
&=\omega_f\widehat{\Sig}_p^{-1}\widehat{\Sig}_f(\btheta_f-\btheta_r)+\omega_r\widehat{\Sig}_p^{-1}\widehat{E}_r+\omega_f\widehat{\Sig}_p^{-1}\widehat{E}_f.
\end{align*}
Define an event 
\[
  \mathcal{E}_0=\left\{\Lambda_{\min}(\widetilde{\Sig}_r)\geq c_{\Sig}^{-1}, \Lambda_{\min}(\widehat{\Sig}_p)\geq c_{\Sig}^{-1},\Lambda_{\max}(\widehat{\Sig}_f)\leq c_{\Sig},\|\omega_r\widehat{E}_r+\omega_f\widehat{E}_f\|_2\leq C\sqrt{p/N}\right\}.
\]
We first show that $\P(\mathcal{E}_0)\rightarrow 1$.
By Exercise 4.7.3 in \citet{vershynin2018high}, we know that
\begin{align}
&\P\left(\|\widehat{\Sig}_f-\Sig_f\|_2\leq C\sqrt{\frac{p}{N_f}}\|\Sig_f\|_2\right)\geq 1-2e^{-p}.\nonumber\\
&\P\left(\|\widetilde{\Sig}_r-\Sig_r\|_2\leq C\sqrt{\frac{p}{\tilde{n}_r}}\|\Sig_r\|_2\right)\geq 1-2e^{-p}.\nonumber\\
&\P\left(\|\widehat{\Sig}_r-\Sig_r\|_2\leq C\sqrt{\frac{p}{N_r}}\|\Sig_r\|_2\right)\geq 1-2e^{-p}.\label{concen1}
\end{align}
For the last statement of $\mathcal{E}_0$, using the sub-exponential property of $\bx_i^{(r)}\eps_i^{(r)}$ and $\bx_i^{(f)}\eps_i^{(f)}$, we have for any given $\|\bm{u}\|_2=1$,
\[
   \P(|\bm{u}^{\top}(\omega_r\widehat{E}_r+\omega_f\widehat{E}_f)|>t)\leq \exp\{-\min\{Nt^2,Nt\}\}.
\]
As 
\[
\|\omega_r\widehat{E}_r+\omega_f\widehat{E}_f\|_2=\sup_{\bm{u}\in\R^p:\|\bm{u}\|_2=1}\bm{u}^{\top}(\omega_r\widehat{E}_r+\omega_f\widehat{E}_f),
\]
taking a covering over $\{\bm{u}\in\R^p:\|\bm{u}\|_2=1\}$, we have
\begin{align*}
\P\left(\|\omega_r\widehat{E}_r+\omega_f\widehat{E}_f\|_2\geq t\right)\leq \exp\{-\min\{Nt^2,Nt\}+Cp\}.
\end{align*}
Hence, for $t=C\sqrt{p/N}$, we arrive at desired results.

In event $\mathcal{E}_0$ for some positive constant $C$, we have
\begin{align*}
\|\hat{\btheta}_p-\btheta_r\|_2&\leq \omega_f\|\widehat{\Sig}_p^{-1}\widehat{\Sig}_f(\btheta_f-\btheta_r)\|_2+\|\widehat{\Sig}_p^{-1}(\omega_r\widehat{E}_r+\omega_f\widehat{E}_f)\|_2\\
&\leq C\omega_f\delta+C\sqrt{\frac{p}{N}}
\end{align*}
with probability at least $1-\exp\{-c_1 p\}$.

For the transfer learning estimate, we know that
\begin{align*}
\hat{\btheta}^{(\textup{tl})}_r=(\widetilde{\Sig}_r+\lambda^{(\textup{tl})} I_p)^{-1}(\widetilde{M}_r+\lambda^{(\textup{tl})}\hat{\btheta}_p),
\end{align*}
which is a weighted average of $\tilde{\btheta}_r^{(\ols)}$ and $\hat{\btheta}_p$. Taking $\lambda^{(\textup{tl})} =C\sqrt{p/\tn_r}/(\sqrt{p/N}+\omega_f\delta)$ for some positive constant $C$, we have
\[
  \| \hat{\btheta}^{(\textup{tl})}_r-\btheta_r\|_2\leq C'\min\{\sqrt{\frac{p}{\tn_r}},\sqrt{\frac{p}{N}}+\omega_f\delta\}
\]
with probability at least $1-\exp\{-c_1p\}$.
\end{proof}
\subsection{Proof of Theorem \ref{thm2}}
\begin{lemma}[A technical lemma]
\label{tlem1}
If  $\|A^{-1}(\widehat{A}-A)\|_2=o(1)$, then
\begin{align*}
\| \widehat{A}^{-1}\widehat{B}-A^{-1}B\|_2 
&\leq (1+o(1))\|A^{-1}(A-\widehat{A})A^{-1}B\|_2+(1+o(1))\|A^{-1}(\widehat{B}-B)\|_2.
\end{align*}
\end{lemma}
\begin{proof}[Proof of Lemma \ref{tlem1}]
\begin{align*}
  \widehat{A}^{-1}\widehat{B}-A^{-1}B&=(\widehat{A}^{-1}-A^{-1})\widehat{B}+A^{-1}(\widehat{B}-B)\\
 &=A^{-1}(A-\widehat{A})\widehat{A}^{-1}\widehat{B}+A^{-1}(\widehat{B}-B).
\end{align*}
Therefore,
\begin{align*}
\| \widehat{A}^{-1}\widehat{B}-A^{-1}B\|_2&\leq \|A^{-1}(A-\widehat{A})\widehat{A}^{-1}\widehat{B}\|_2+\|A^{-1}(\widehat{B}-B)\|_2\\
&\leq \|A^{-1}(A-\widehat{A})A^{-1}B\|_2+\|A^{-1}(\widehat{B}-B)\|_2+\| \widehat{A}^{-1}\widehat{B}-A^{-1}B\|_2\|A^{-1}(A-\widehat{A})\|_2.
\end{align*}
If $\|A^{-1}(\widehat{A}-A)\|_2=o(1)$, then 
\begin{align*}
\| \widehat{A}^{-1}\widehat{B}-A^{-1}B\|_2\leq (1+o(1))\|A^{-1}(A-\widehat{A})A^{-1}B\|_2+(1+o(1))\|A^{-1}(\widehat{B}-B)\|_2.
\end{align*}
\end{proof}

\begin{proof}[Proof of Theorem \ref{thm2}]
We know that
\[
  \|\hat{\btheta}^{(\uls)}_r-\btheta_r\|_2\leq\|\hat{\btheta}^{(\uls)}_r-\mathring{\btheta}_r^{(\text{rtr})}\|_2+\|\mathring{\btheta}_r^{(\text{rtr})}-\btheta_r\|_2.
\]

\begin{align}
\hat{\btheta}^{(\uls)}_r-\mathring{\btheta}_r^{(\text{rtr})}&=(\omega_r\widetilde{\Sig}_r)^{-1}(\widetilde{\Sig}_p\hat{\btheta}_p-\omega_f\widehat{M}_f)-\mathring{\btheta}_r^{(\text{rtr})}\nonumber\\
&=(\omega_r\widetilde{\Sig}_r)^{-1}[\widetilde{\Sig}_p\hat{\btheta}_p-\omega_f\widehat{M}_f-\omega_r\widetilde{\Sig}_r\mathring{\btheta}_r^{(\text{rtr})}]\nonumber\\
&=(\omega_r\widetilde{\Sig}_r)^{-1}[\widehat{\Sig}_p\hat{\btheta}_p-\omega_f\widehat{M}_f-\omega_r\widetilde{\Sig}_r\mathring{\btheta}_r^{(\text{rtr})}+\omega_r(\widetilde{\Sig}_r-\widehat{\Sig}_r)\hat{\btheta}_p]\nonumber\\
&=(\omega_r\widetilde{\Sig}_r)^{-1}[\omega_r\widehat{M}_r-\omega_r\widetilde{\Sig}_r\mathring{\btheta}_r^{(\text{rtr})}+\omega_r(\widetilde{\Sig}_r-\widehat{\Sig}_r)\hat{\btheta}_p]\nonumber\\
&=(\omega_r\widetilde{\Sig}_r)^{-1}\omega_r(\widehat{\Sig}_r-\widetilde{\Sig}_r)(\mathring{\btheta}_r^{(\text{rtr})}-\hat{\btheta}_p)\nonumber\\
&=\omega_f\widetilde{\Sig}_r^{-1}(\widehat{\Sig}_r-\widetilde{\Sig}_r)\widehat{\Sig}_p^{-1}(\widehat{\Sig}_f\mathring{\btheta}_r^{(\text{rtr})}-\widehat{M}_f).\label{decomp1}
\end{align}
Define an event 
\[
  \mathcal{E}_1=\mathcal{E}_0\cap \left\{\|\widehat{E}_r\|_2\leq C\sqrt{\frac{p}{N_r}},\|\widehat{E}_f\|_2\leq C\sqrt{\frac{p}{N_f}}\right\}.
\]
Analogous to the proof of Lemma \ref{tlem2}, we can show that $\P(\mathcal{E}_1)\geq 1-\exp\{-c_1p\}$.

By (\ref{decomp1}), in event $\mathcal{E}_1$, we have
\begin{align*}
\|\hat{\btheta}^{(\uls)}_r-\mathring{\btheta}_r^{(\text{rtr})}\|_2&\leq \omega_f\|\widetilde{\Sig}_r^{-1}\|_2\|\widetilde{\Sig}_r-\Sig_r\|_2\|\widehat{\Sig}_p^{-1}(\widehat{\Sig}_f\mathring{\btheta}_r^{(\text{rtr})}-\widehat{M}_f)\|_2\\
&\leq \omega_f\|\widetilde{\Sig}_r^{-1}\|_2\|\widetilde{\Sig}_r-\widehat{\Sig}_r\|_2\|\widehat{\Sig}_p^{-1}\|_2\|\widehat{\Sig}_f\mathring{\btheta}_r^{(\text{rtr})}-\widehat{M}_f\|_2\\
&\leq \frac{\omega_f}{c_0^2}\sqrt{\frac{p}{\tilde{n}_r}}\|\widehat{\Sig}_f\mathring{\btheta}_r^{(\text{rtr})}-\widehat{M}_f\|_2.
\end{align*}

As $\mathring{\btheta}_r^{(\text{rtr})}=\btheta_r+\widehat{\Sig}_r^{-1}\widehat{E}_r$, we have in event $\mathcal{E}_1$,
\begin{align*}
\|\widehat{\Sig}_f\mathring{\btheta}_r^{(\text{rtr})}-\widehat{M}_f\|_2&=\|\widehat{\Sig}_f(\btheta_r-\btheta_f)+\widehat{\Sig}_f\widehat{\Sig}_r^{-1}\widehat{E}_r-\widehat{E}_f\|_2\\
&\leq \|\widehat{\Sig}_f\|_2\delta+\|\widehat{\Sig}_f\widehat{\Sig}_r^{-1}\widehat{E}_r\|_2+\|\widehat{E}_f\|_2\\
&\leq C\left(\delta+\sqrt{\frac{p}{\min\{N_f,N_r\}}}\right).
\end{align*}

To summarize, we arrive at with probability at least $1-\exp\{-c_1p\}$,
\begin{align*}
\|\hat{\btheta}^{(\uls)}_r-\mathring{\btheta}_r^{(\text{rtr})}\|_2&\leq C\omega_f\sqrt{\frac{p}{\tilde{n}_r}}\left(\delta+\sqrt{\frac{p}{\min\{N_f,N_r\}}}\right)\\
&\leq C\omega_f\sqrt{\frac{p}{\tilde{n}_r}}\left(\delta+\sqrt{\frac{p}{N_f}}+\sqrt{\frac{p}{N_r}}\right)\\
&\leq C\omega_f\sqrt{\frac{p}{\tilde{n}_r}}\left(\delta+\sqrt{\frac{p}{N_f}}\right)+o(1)\sqrt{\frac{p}{N_r}},
\end{align*}
where the last step is due to $p=o(\tn_r)$.
Note that
\begin{align*}
\omega_f\sqrt{\frac{p}{\tilde{n}_r}}\sqrt{\frac{p}{N_f}}=\frac{N_fp}{N\sqrt{\tn_rN_f}}=\sqrt{\frac{p}{N}}\sqrt{\frac{p}{\tn_r}\omega_f}=o(1)\sqrt{\frac{p}{N}}= o(1)\sqrt{\frac{p}{N_r}}.
\end{align*}
To summarize, with probability at least $1-\exp\{-c_1p\}$,
\[
  \|\hat{\btheta}^{(\uls)}_r-\btheta_r\|_2\leq c_2\sqrt{\frac{p}{N_r}}+c_2\omega_f\delta\sqrt{\frac{p}{\tn_r}}.
\]
\end{proof}

\subsection{Proof of Theorem \ref{thm0-gd}}
\begin{proof}[Proof of Theorem \ref{thm0-gd}]
Let $\hat{\bu}_t=\hat{\btheta}^{(\uls)}_{r,t}-\mathring{\btheta}_r^{(\text{rtr})}$ and
\[
\widehat{G}_{t-1}=N(\widetilde{\Sig}_p\hat{\btheta}_p-\omega_f\widehat{M}_f-\omega_r\widetilde{\Sig}_r\hat{\btheta}^{(\uls)}_{r,t-1}).
\] Then
\begin{align*}
\hat{\bu}_t&=\hat{\bu}_{t-1}+\alpha\widehat{G}_{t-1}\\
&=\hat{\bu}_{t-1}+\alpha N(\widetilde{\Sig}_p\hat{\btheta}_p-\omega_f\widehat{M}_f-\omega_r\widetilde{\Sig}_r\mathring{\btheta}_r^{(\text{rtr})})-\alpha N_r\widetilde{\Sig}_r\hat{\bu}_{t-1}\\
&=(I_p-\alpha N_r\widetilde{\Sig}_r)\hat{\bu}_{t-1}+\alpha N(\widehat{\Sig}_p\hat{\btheta}_p-\omega_f\widehat{M}_f-\omega_r\widetilde{\Sig}_r\mathring{\btheta}_r^{(\text{rtr})})+\alpha N(\widetilde{\Sig}_p-\widehat{\Sig}_p)\hat{\btheta}_p\\
&=\underbrace{(I_p-\alpha N_r\widetilde{\Sig}_r)}_{A}\hat{\bu}_{t-1}+\underbrace{\alpha N_r(\widehat{M}_r-\widetilde{\Sig}_r\mathring{\btheta}_r^{(\text{rtr})})+\alpha N_r(\widetilde{\Sig}_r-\widehat{\Sig}_r)\hat{\btheta}_p}_{\zeta}.
\end{align*}
By induction, we arrive at
\begin{align*}
\hat{\bu}_t&=A^t\hat{\bu}_0+\sum_{k=0}^{t-1}A^k\zeta\\
&=A^t(\hat{\btheta}_p-\mathring{\btheta}_r^{(\text{rtr})})+(I_p-A^t)(I_p-A)^{-1}\zeta.
\end{align*}
Hence,
\begin{align*}
\|\hat{\bu}_t\|_2\leq \underbrace{\|A^t(\hat{\btheta}_p-\mathring{\btheta}_r^{(\text{rtr})})\|_2}_{T_1}+\underbrace{\|(I_p-A^t)(I_p-A)^{-1}\zeta\|_2}_{T_2}.
\end{align*}
We bound the two terms separately.
For $T_1$,
\begin{align*}
T_1&\leq \|A\|_2^t\|\hat{\btheta}_p-\mathring{\btheta}_r^{(\text{rtr})}\|_2\\
&\leq (1-\alpha N_r\Lambda_{\min}(\widetilde{\Sig}_r))^t\|\hat{\btheta}_p-\mathring{\btheta}_r^{(\text{rtr})}\|_2.
\end{align*}
We will choose $\alpha$ such that
\[
1-\alpha N_r\Lambda_{\min}(\widetilde{\Sig}_r)<c_{\alpha} ~\text{and}~\Lambda_{\min}(A)\geq 0,
\]
which gives $0< \alpha < \frac{1-c_{\alpha}}{N_r\Lambda_{\max}(\widetilde{\Sig}_r)}$.  Moreover,
\begin{align*}
\|\hat{\btheta}_p-\mathring{\btheta}_r^{(\text{rtr})}\|_2&=\|\widehat{\Sig}_p^{-1}(\omega_r\widehat{M}_r+\omega_f\widehat{M}_f-\omega_r\widehat{M}_r-\omega_f\widehat{\Sig}_f\mathring{\btheta}_r^{(\text{rtr})})\|_2\\
&=\|\omega_f\widehat{\Sig}_p^{-1}\widehat{\Sig}_f(\btheta_f-\btheta_r)\|_2+\|\omega_f\widehat{\Sig}_p^{-1}(\widehat{E}_f+\widehat{\Sig}_f\widehat{\Sig}_r^{-1}\widehat{E}_r)\|_2\\
&\leq C\omega_f\delta+C\omega_f\sqrt{\frac{p}{\min\{N_f,N_r\}}}.
\end{align*}

For $T_2$,
\begin{align*}
T_2\leq \|(I_p-A)^{-1}\zeta\|_2=\|(\alpha N_r\widetilde{\Sig}_r)^{-1}\zeta\|_2.
\end{align*}
Note that
\begin{align*}
\zeta&=\alpha N_r(\widehat{M}_r-\widetilde{\Sig}_r\mathring{\btheta}_r^{(\text{rtr})})+\alpha N_r(\widetilde{\Sig}_r-\widehat{\Sig}_r)\mathring{\btheta}_r^{(\text{rtr})}-\alpha N_r(\widetilde{\Sig}_r-\widehat{\Sig}_r)(\mathring{\btheta}_r^{(\text{rtr})}-\hat{\btheta}_p)\\
&=\alpha N_r(\widehat{\Sig}_r-\widetilde{\Sig}_r)(\mathring{\btheta}_r^{(\text{rtr})}-\hat{\btheta}_p)\\
&=\alpha N_r(\widehat{\Sig}_r-\widetilde{\Sig}_r)\widehat{\Sig}_p^{-1}\omega_f(\widehat{\Sig}_f\mathring{\btheta}_r^{(\text{rtr})}-\widehat{\Sig}_f\btheta_f-\widehat{E}_f).
\end{align*}
Hence,
\begin{align*}
T_2\leq \|(\alpha N_r\widetilde{\Sig}_r)^{-1}\zeta\|_2=\omega_f\|\widetilde{\Sig}_r^{-1}(\widehat{\Sig}_r-\widetilde{\Sig}_r)\widehat{\Sig}_p^{-1}(\widehat{\Sig}_f\mathring{\btheta}_r^{(\text{rtr})}-\widehat{\Sig}_f\btheta_f-\widehat{E}_f)\|_2.
\end{align*}
In view of (\ref{decomp1}) and the following proof, we can show that
\[
\P\left(T_2\leq \omega_f\delta\sqrt{\frac{p}{\tilde{n}_r}}\right)\geq 1-\exp\{-c_1p\}.
\]
Combining the upper bounds for $T_1$ and $T_2$, we have
\begin{align*}
\|\hat{\bu}_t\|_2\leq c_1(1-\alpha N_r\Lambda_{\min}(\widetilde{\Sig}_r))^t(\omega_f\delta+\omega_f\sqrt{\frac{p}{\min\{N_f,N_r\}}})+c_1\omega_f\delta\sqrt{\frac{p}{\tilde{n}_r}}
\end{align*}
with probability at least $1-\exp\{-c_2p\}$.
\end{proof}

%

\subsection{Proofs in Section \ref{sec5}}
\begin{proof}[Proof of Theorem \ref{thm-inf}]
By (\ref{decomp1}), we have
\begin{align*}
\bm{v}^{\top}(\hat{\btheta}^{(\uls)}_r-\btheta_r)&=\bm{v}^{\top}(\mathring{\btheta}_r^{(\text{rtr})}-\btheta_r)+\omega_f\bm{v}^{\top}\widetilde{\Sig}_r^{-1}(\widehat{\Sig}_r-\widetilde{\Sig}_r)\widehat{\Sig}_p^{-1}(\widehat{\Sig}_f\mathring{\btheta}_r^{(\text{rtr})}-\widehat{M}_f)\\
&=\underbrace{\bm{v}^{\top}\Sig_r^{-1}\widehat{E}_r}_{T_3} +\underbrace{\omega_f\bm{v}^{\top}\Sig_r^{-1}(\widehat{\Sig}_r-\widetilde{\Sig}_r)\Sig_p^{-1}\Sig_f(\btheta_r-\btheta_f)}_{T_4}+rem,
\end{align*}
where  $\Sig_p=\omega_r\Sig_r+\omega_f\Sig_f$ and
\begin{align*}
  rem&=\underbrace{\bm{v}^{\top}(\widehat{\Sig}_r^{-1}-\Sig_r^{-1})\widehat{E}_r}_{rem_{1}}\\
  &+
  \underbrace{\omega_f\bm{v}^{\top}\widetilde{\Sig}_r^{-1}(\widehat{\Sig}_r-\widetilde{\Sig}_r)\widehat{\Sig}_p^{-1}(\widehat{\Sig}_f\mathring{\btheta}_r^{(\text{rtr})}-\widehat{M}_f)-\omega_f\bm{v}^{\top}\Sig_r^{-1}(\widehat{\Sig}_r-\widetilde{\Sig}_r)\Sig_p^{-1}\Sig_f(\btheta_r-\btheta_f)}_{rem_{2}}.
\end{align*}

For $rem_1$, we have
\begin{align*}
rem_1&\leq \|\bm{v}^{\top}(\widehat{\Sig}_r^{-1}-\Sig_r^{-1})\widehat{E}_r\|_2\\
&\leq \|\bm{v}^{\top}\Sig_r^{-1}(\widehat{\Sig}_r-\Sig_r)\|_2\|\widehat{\Sig}_r^{-1}(X^{(r)})^{\top}\bm\eps^{(r)}/N_r\|_2.
\end{align*}
Using the Bernstein's inequality for sub-exponential random variables, we know that
\begin{align*}
&\P(\|\bm{v}^{\top}\Sig_r^{-1}(\widehat{\Sig}_r-\Sig_r)\|_2\geq t_1)\leq e^p\exp\left\{-\min\{\frac{N_r^2t_1^2}{N_r\|\bm{v}\|_2^2},\frac{N_rt_1}{\|\bm{v}\|_2}\}\right\}\\
&\P(\|\widehat{E}_r\|_2\geq t_2)\leq e^p\exp\left\{-\min\{\frac{N_r^2t_2^2}{N_r},N_rt_2\}\right\}.
\end{align*}
Taking $t_1=C\|\bm{v}\|_2\sqrt{(p+(\log N_r)^{1/2})/N_r}$ and $t_2=C\sqrt{(p+(\log N_r)^{1/2})/N_r}$, we have
\[
   \P\left(rem_1\geq C\|\bm{v}\|_2\frac{p+\sqrt{\log N_r}}{N_r}\right)\leq \exp\{-\log N_r\}.
\]
For $rem_2$, note that
\begin{align*}
rem_2&\leq\omega_f|\bm{v}^{\top}(\Sig_r^{-1}-\widetilde{\Sig}_r^{-1})(\widehat{\Sig}_r-\widetilde{\Sig}_r)\widehat{\Sig}_p^{-1}(\widehat{\Sig}_f\mathring{\btheta}_r^{(\text{rtr})}-\widehat{M}_f)|\\
&\quad +\omega_f|\bm{v}^{\top}\Sig_r^{-1}(\widehat{\Sig}_r-\widetilde{\Sig}_r)\{\widehat{\Sig}_p^{-1}-\Sig_p^{-1}\}\Sig_f(\btheta_r-\btheta_f)|\\
&\quad+\omega_f|\bm{v}^{\top}\Sig_r^{-1}(\widehat{\Sig}_r-\widetilde{\Sig}_r)[\widehat{\Sig}_p^{-1}(\widehat{\Sig}_f\widehat{\Sig}_r^{-1}\widehat{E}_r-\widehat{E}_f)]|.
\end{align*}
By (\ref{concen1}) and event $\mathcal{E}_1$,
\begin{align*}
rem_2&\leq \frac{p\omega_f\|\bm{v}\|_2}{\tn_r}\|\btheta_r-\btheta_f\|_2+\frac{p\omega_f\|\bm{v}\|_2}{\sqrt{\tn_rN}}\|\btheta_r-\btheta_f\|_2+\omega_f\|\bm{v}\|_2\sqrt{\frac{p}{\tn_r}}\sqrt{\frac{p}{N_r\wedge N_f}}\\
&=\frac{\omega_f\|\bm{v}\|_2\|\btheta_r-\btheta_f\|_2}{\sqrt{\tn_r}}\frac{p+\sqrt{\log N_r}}{\sqrt{\min\{\tn_r,N_r,N_f\}}}
\end{align*}
with probability at least $1-\exp\{-c_1\log N_r\}$.

By our assumption that $p^2=o(\min\{\tn_r,N_f\})$, we have 
\begin{align}
\label{rem-order}
   rem_1=o\left(\frac{\|\bm{v}\|_2}{\sqrt{N_r}}\right)~\text{and}~rem_2=o\left(\frac{\omega_f\|\bm{v}\|_2\delta}{\sqrt{\tn_r}}\right)
\end{align}
with probability at least $1-\exp\{-c_1\log N_r\}$.

For $T_3$, note that
\[
 T_3=\frac{1}{N_r}\bm{v}^{\top}\Sig_r^{-1}(\tilde{X}^{(r)})^{\top}\tilde{\bm\eps}^{(r)}+\frac{1}{N_r}\bm{v}^{\top}\Sig_r^{-1}(\check{X}^{(r)})^{\top}\check{\bm\eps}^{(r)}.
\]
\begin{align*}
T_4&=\frac{\tn_r-N_r}{\tn_rN_r}\omega_f\bm{v}^{\top}\Sig_r^{-1}((\widetilde{X}^{(r)})^{\top}\widetilde{X}^{(r)}-\tn_r\Sig_r)\Sig_p^{-1}\Sig_f(\btheta_r-\btheta_f)\\
&\quad+\frac{1}{N_r}\omega_f\bm{v}^{\top}\Sig_r^{-1}((\check{X}^{(r)})^{\top}\check{X}^{(r)}-(N_r-\tn_r)\Sig_r)\Sig_p^{-1}\Sig_f(\btheta_r-\btheta_f).
\end{align*}

By definition of $a_i^{(r)}$ and $b_i^{(r)}$,
 \[
    T_3+T_4=\sum_{i\in\tilde{\mathcal{N}}_r}(\frac{1}{N_r}a^{(r)}_i+\frac{\tn_r-N_r}{\tn_rN_r}b^{(r)}_i)+\sum_{i\notin\tilde{\mathcal{N}}_r}(\frac{1}{N_r}a^{(r)}_i+\frac{1}{N_r}b^{(r)}_i),
 \]
$\E[T_3+T_4]=0$, and $\textup{Var}(T_3+T_4)=V_r$. 

As $\E[a_i^{(r)}]=0$ and $\E[b_i^{(r)}]=0$ and $a_i^{(r)}+cb_i^{(r)}$ are independent of each other for any constant $c$. We will apply Lyapunov's central limit theorem. We first compute
\begin{align*}
\textup{Var}(T_3+T_4)&=\frac{1}{N_r^2}\sum_{i\in\tilde{\mathcal{N}}_r}\E[(a^{(r)}_i+\frac{N_r-\tn_r}{\tn_r}b^{(r)}_i)^2]+\frac{1}{N_r^2}\sum_{i\notin\tilde{\mathcal{N}}_r}\E[(a^{(r)}_i+b_i^{(r)})^2].
\end{align*}

By Condition \ref{cond-inf},
\[
\textup{Var}(T_3+T_4)\geq \frac{1-\rho}{N_r^2}\sum_{i=1}^{N_r}\E[(a_i^{(r)})^2]+\frac{(1-\rho)(N_r-\tn_r)^2}{\tn_r^2N_r^2}\sum_{i\in\tilde{\mathcal{N}}_r}\E[(b^{(r)}_i)^2]+\frac{1-\rho}{N_r^2}\sum_{i\notin\tilde{\mathcal{N}}_r}\E[(b_i^{(r)})^2].
\]
As $\tn_r\leq cN_r$ for some constant $0<c<1$, we have
\[
\textup{Var}(T_3+T_4)\geq \frac{1-\rho}{N_r^2}\sum_{i=1}^{N_r}\E[(a_i^{(r)})^2]+\frac{c(1-\rho)}{\tn_r^2}\sum_{i\in\tilde{\mathcal{N}}_r}\E[(b^{(r)}_i)^2]+\frac{1-\rho}{N_r^2}\sum_{i\notin\tilde{\mathcal{N}}_r}\E[(b_i^{(r)})^2].
\]
Using Condition \ref{cond-inf} again, we have
\begin{align}
\label{var-lowerbound}
\textup{Var}(T_3+T_4)\geq \frac{c_1\|\bm{v}\|_2^2}{N_r}+\frac{c_2\|\bm{v}\|_2^2\omega_f^2\|\btheta_r-\btheta_f\|_2^2}{\tn_r}.
\end{align}
By (\ref{rem-order}), we have
\[
   rem_1+rem_2=o(1)\textup{Var}^{1/2}(T_3+T_4).
\]
We only need to check the fourth-moment conditions. Let
\[
 s_4=\sum_{i\in\tilde{\mathcal{N}}_r}\E[(\frac{1}{N_r}a^{(r)}_i+\frac{\tn_r-N_r}{\tn_rN_r}b^{(r)}_i)^4]+\sum_{i\not\in\tilde{\mathcal{N}}_r}\E[(\frac{1}{N_r}a^{(r)}_i+\frac{1}{N_r}b^{(r)}_i)^4].
\]
Note that
\begin{align*}
s_4\leq \frac{8}{N_r^4}\sum_{i=1}^{N_r}\E[(a_i^{(r)})^4]+8(\frac{\tn_r-N_r}{\tn_rN_r})^4\sum_{i\in\tilde{\mathcal{N}}_r}\E[(b^{(r)}_i)^4]+\frac{8}{N_r^4}\sum_{i\notin\tilde{\mathcal{N}}_r}\E[(b^{(r)}_i)^4].
\end{align*}
Using the sub-Gaussian properties of $\bx_i^{(r)}$ and $\eps_i^{(r)}$, we know that
\[
  \E[(a_i^{(r)})^4]\leq C\|\bm{v}\|_2^4~\text{and}~\E[(b_i^{(r)})^4]\leq C\omega_f^4\|\bm{v}\|_2^4\delta^4.
\]
Hence,
\[
  s_4\leq \frac{C\|\bm{v}\|_2^4}{N_r^3}+\frac{\omega_f^4\|\bm{v}\|_2^4\delta^4}{\tn_r^3}.
\]
By (\ref{var-lowerbound}), 
\begin{align}
\label{var-lowerbound2}
\frac{s_4}{\textup{Var}^2(T_3+T_4)}\leq \frac{1}{\tn_r}.
\end{align}

Therefore, by Lyapunov's central limit theorem, we have
\begin{align*}
\frac{\bm{v}^{\top}(\hat{\btheta}^{(\uls)}_r-\btheta_r)}{\sqrt{V_r}}\xrightarrow{D}N(0,1).
\end{align*}

\end{proof}

\begin{proof}[Proof of Lemma \ref{lem-var}]

Note that 
\[
 b_i^{(r)}=\bm{v}^{\top}\Sig_r^{-1}(\bx_i^{(r)}(\bx_i^{(r)})^{\top}-\Sig_r)(\btheta_r-\btheta_p).
\]
Let
\begin{align*}
\widetilde{V}_r=\frac{1}{N_r^2}\sum_{i\in\widetilde{\mathcal{N}}_r}(a^{(r)}_i+\frac{\tn_r-N_r}{\tn_r}b^{(r)}_i)^2+\frac{N_r-\tn_r}{N_r^2\tn_r}\sum_{i\in \widetilde{\mathcal{N}}_r}(a^{(r)}_i+b^{(r)}_i)^2.
\end{align*}
Let
\begin{align*}
\check{V}_r=\frac{1}{N_r^2}\sum_{i\in\widetilde{\mathcal{N}}_r}(a^{(r)}_i+\frac{\tn_r-N_r}{\tn_r}b^{(r)}_i)^2+\frac{1}{N_r^2}\sum_{i\notin \widetilde{\mathcal{N}}_r}(a^{(r)}_i+b^{(r)}_i)^2.
\end{align*}
Consider the decomposition
\begin{align}
|\widehat{V}_r-V_r|\leq \underbrace{|\widehat{V}_r-\widetilde{V}_r|}_{T_5}+\underbrace{|\widetilde{V}_r-\check{V}_r|}_{T_6}+\underbrace{|\check{V}_r-V_r|}_{T_7}.
\end{align}
We will bound each term separately.

For $T_7$, by Chebyshev's inequality
\begin{align*}
\P(|\check{V}_r-V_r|/V_r\geq t)\leq \frac{\E[\check{V}^2_r]}{V_r^2t^2}.
\end{align*}
By (\ref{var-lowerbound2}), 
\[
  \E[\check{V}_r^2]=s_4\leq CV_r^2/\tn_r.
\]
Hence,
\begin{align*}
\P(T_7/V_r\geq \tn_r^{-1/4})\leq \frac{1}{\sqrt{\tn_r}}.
\end{align*}
For $T_6$, note that
\begin{align*}
  |\widetilde{V}_r-\check{V}_r|&=\frac{N_r-\tn_r}{N_r^2\tn_r}\sum_{i\in \widetilde{\mathcal{N}}_r}(a^{(r)}_i+b^{(r)}_i)^2-\frac{1}{N_r^2}\sum_{i \notin \widetilde{\mathcal{N}}_r}(a^{(r)}_i+b^{(r)}_i)^2\\
  &=\frac{N_r-\tn_r}{N_r^2\tn_r}\sum_{i=1}^{N_r}(a^{(r)}_i+b^{(r)}_i)^2\mathbb{1}(i\in\widetilde{\mathcal{N}}_r)-\frac{1}{N_r^2}\sum_{i=1}^{N_r}(a^{(r)}_i+b^{(r)}_i)^2\mathbb{1}(i\not\in\widetilde{\mathcal{N}}_r)\\
  &=\frac{N_r}{N_r^2\tn_r}\sum_{i=1}^{N_r}(a^{(r)}_i+b^{(r)}_i)^2\mathbb{1}(i\in\widetilde{\mathcal{N}}_r)-\frac{1}{N_r^2}\sum_{i=1}^{N_r}(a^{(r)}_i+b^{(r)}_i)^2.
\end{align*}
As $\widetilde{\calD}_r$ is a random sample from $\calD_r$, we know that
\[
  \E[ \frac{N_r}{N_r^2\tn_r}\sum_{i\in \mathcal{N}_r}(a^{(r)}_i+b^{(r)}_i)^2\mathbb{1}(i\in\widetilde{\mathcal{N}}_r)|\calD_r]=\frac{1}{N_r^2}\sum_{i\in \mathcal{N}_r}(a^{(r)}_i+b^{(r)}_i)^2.
\]
By Bernstein's inequality, we have
\begin{align*}
\P(T_6\geq t|\calD_r)\leq \exp\left\{-\frac{t^2}{\frac{1}{N_r^2\tn_r^2}\sum_{i=1}^{N_r}(a_i^{(r)}+b_i^{(r)})^4\P(i\in\widetilde{\mathcal{N}}_r|\calD_r)+\frac{1}{N_r\tn_r}\max_{i\leq N_r}(a_i^{(r)}+b_i^{(r)})^2t}\right\}.
\end{align*}
We know that $\P(i\in\widetilde{\mathcal{N}}_r|\calD_r)=\tn_r/N_r$.
For 
\[
t\geq C\max\left\{\frac{\sqrt{\tn_r\log N_r\sum_{i=1}^{N_r}(a_i^{(r)}+b_i^{(r)})^4/N_r}}{N_r\tn_r},\frac{\log N_r\max_{i\leq N_r}(a_i^{(r)}+b_i^{(r)})^2}{N_r\tn_r}\right\},
\]
we have
\[
   \P(T_6\geq t|\calD_r)\leq \exp\{-c_1\log N_r\}.
\]
Using the sub-Gaussian property of $\bx_i^{(r)}$ and $\eps_i^{(r)}$, we know that
\[
   \max_{i\leq \mathcal{N}_r}(a_i^{(r)}+b_i^{(r)})^2\leq C\|\bm{v}\|_2^2\log N_r(1+\omega_f^2\|\bdelta\|_2^2)
\]
with probability at least $1-\exp\{-c_1 N_r\}$. Hence,
\begin{align*}
\P\left(T_6\geq \frac{\|\bm{v}\|_2^2(1+\omega_f^2\|\btheta_r-\btheta_f\|_2^2)\log N_r}{N_r\sqrt{\tn_r}}\right)\leq \exp\{-c_1\log N_r\}.
\end{align*}
By (\ref{var-lowerbound}), we arrive at
\[
  \P\left(\frac{T_6}{V_r}\geq \frac{C\log N_r}{\sqrt{\tn_r}}+\frac{C\sqrt{\tn_r}\log N_r}{N_r}\right)\leq \exp\{-c_1\log N_r\}
\]
and $\log N_r/\sqrt{\tn_r}+\sqrt{\tn_r}\log N_r/N_r=o(1)$.

It is left to bound $T_5$.
\begin{align*}
\widehat{V}_r-\widetilde{V}_r&=\underbrace{\frac{1}{N_r^2}\sum_{i\in\widetilde{\mathcal{N}}_r}(\hat{a}^{(r)}_i+\frac{\tn_r-N_r}{\tn_r}\hat{b}^{(r)}_i)^2-\frac{1}{N_r^2}\sum_{i\in\widetilde{\mathcal{N}}_r}(a^{(r)}_i+\frac{\tn_r-N_r}{\tn_r}b^{(r)}_i)^2}_{T_{5,1}}\\
&\quad+\underbrace{\frac{N_r-\tn_r}{N_r^2\tn_r}\sum_{i\in \widetilde{\mathcal{N}}_r}(\hat{a}^{(r)}_i+\hat{b}^{(r)}_i)^2-\frac{N_r-\tn_r}{N_r^2\tn_r}\sum_{i\in \widetilde{\mathcal{N}}_r}(a^{(r)}_i+b^{(r)}_i)^2}_{T_{5,2}}.
\end{align*}
We first bound $T_{5,1}$. Note that
\begin{align*}
T_{5,1}&=\frac{1}{N_r^2}\sum_{i\in\widetilde{\mathcal{N}}_r}\{\hat{a}_i^{(r)}-a_i^{(r)}+\frac{\tn_r-N_r}{\tn_r}(\hat{b}_i^{(r)}-b_i^{(r)})\}^2\\
&\quad+\frac{2}{N_r^2}\sum_{i\in\widetilde{\mathcal{N}}_r}\{\hat{a}_i^{(r)}-a_i^{(r)}+\frac{\tn_r-N_r}{\tn_r}(\hat{b}_i^{(r)}-b_i^{(r)})\}\{a_i^{(r)}+\frac{\tn_r-N_r}{\tn_r}b_i^{(r)}\}\\
&\leq \frac{1}{N_r^2}\sum_{i\in\widetilde{\mathcal{N}}_r}\{\hat{a}_i^{(r)}-a_i^{(r)}+\frac{\tn_r-N_r}{\tn_r}(\hat{b}_i^{(r)}-b_i^{(r)})\}^2\\
&\quad+\widetilde{V}^{1/2}_r\sqrt{\frac{1}{N_r^2}\sum_{i\in\widetilde{\mathcal{N}}_r}\{\hat{a}_i^{(r)}-a_i^{(r)}+\frac{\tn_r-N_r}{\tn_r}(\hat{b}_i^{(r)}-b_i^{(r)})\}^2}.
\end{align*}
As we have shown that $|\widetilde{V}_r-V_r|/V_r=o_P(1)$, to show that $T_{5,1}/V_r=o_P(1)$, it suffices to show that
\begin{align}
\label{toshow1}
   \frac{1}{N_r^2}\sum_{i\in\widetilde{N}_r}\{\hat{a}_i^{(r)}-a_i^{(r)}\}^2/V_r=o_P(1)~\text{and} ~\frac{1}{N_r^2}(\frac{\tn_r-N_r}{\tn_r})^2\sum_{i\in\widetilde{N}_r}\{\hat{b}_i^{(r)}-b_i^{(r)}\}^2/V_r=o_P(1).
\end{align}

Similarly, for $T_{5,2}$, we have
\begin{align*}
T_{5,2}
&\leq \frac{N_r-\tn_r}{N_r^2\tn_r}\sum_{i\in\widetilde{\mathcal{N}}_r}\{\hat{a}_i^{(r)}-a_i^{(r)}+\hat{b}_i^{(r)}-b_i^{(r)}\}^2\\
&\quad+\widetilde{V}^{1/2}_r\sqrt{\frac{N_r-\tn_r}{N_r^2\tn_r}\sum_{i\in\widetilde{\mathcal{N}}_r}\{\hat{a}_i^{(r)}-a_i^{(r)}+\hat{b}_i^{(r)}-b_i^{(r)}\}^2}.
\end{align*}
To show that $T_{5,2}/V_r=o_P(1)$, it suffices to show that
\begin{align}
\label{toshow2}
   \frac{N_r-\tn_r}{N_r^2\tn_r}\sum_{i\in\widetilde{\mathcal{N}}_r}\{\hat{a}_i^{(r)}-a_i^{(r)}\}^2/V_r=o_P(1)~\text{and}~ \frac{N_r-\tn_r}{N_r^2\tn_r}\sum_{i\in\widetilde{\mathcal{N}}_r}\{\hat{b}_i^{(r)}-b_i^{(r)}\}^2/V_r=o_P(1).
\end{align}
In view of (\ref{toshow1}) and (\ref{toshow2}), to show $(T_{5,1}+T_{5,2})/V_r=o(1)$, it suffices to show that
\begin{align}
\label{toshow3}
   \frac{1}{N_r\tn_r}\sum_{i\in\widetilde{\mathcal{N}}_r}\{\hat{a}_i^{(r)}-a_i^{(r)}\}^2/V_r=o_P(1)~\text{and} \frac{1}{\tn^2_r}\sum_{i\in\widetilde{\mathcal{N}}_r}\{\hat{b}_i^{(r)}-b_i^{(r)}\}^2/V_r=o_P(1).
\end{align}
Let $\hat{\eps}_i^{(r)}=y_i^{(r)}-\bx_i^{(r)}\hat{\btheta}_r^{(\ols)}$, $i=1,\dots,N_r$. 
By definition,
\begin{align*}
&\sum_{i\in\widetilde{\mathcal{N}}_r}(\hat{a}_i^{(r)}-a_i^{(r)})^2=\sum_{i\in\widetilde{\mathcal{N}}_r}\{\bm{v}^{\top}(\widetilde{\Sig}_r^{-1}\bx_i^{(r)}\hat{\eps}^{(r)}_i-\Sig_r^{-1}\bx_i^{(r)}\eps_i^{(r)})\}^2\\
&\leq 2\sum_{i\in\widetilde{\mathcal{N}}_r}\{\bm{v}^{\top}(\widetilde{\Sig}_r^{-1}-\Sig_r^{-1})\bx_i^{(r)}\eps^{(r)}_i\}^2+2\sum_{i\in\widetilde{\mathcal{N}}_r}\{\bm{v}^{\top}\widetilde{\Sig}_r^{-1}\bx_i^{(r)}(\hat{\eps}_i^{(r)}-\eps_i^{(r)})\}^2\\
&\leq 2\|\bm{v}\|_2^2\sum_{i\in\widetilde{\mathcal{N}}_r}\|\bx_i^{(r)}\eps^{(r)}_i\|_2^2\|\widetilde{\Sig}_r^{-1}-\Sig_r^{-1}\|_2^2+2\|\bm{v}\|_2^2\max_{i\in\tilde{\mathcal{N}}_r}\|(\bx_i^{(r)})^{\top}(\hat{\btheta}_r^{(\uls)}-\btheta_r)\|_2^2\sum_{i\in\tilde{\mathcal{N}}_r}\|(\bx_i^{(r)})^{\top}\widetilde{\Sig}_r^{-1}\bm{v}\|_2^2,
\end{align*}
We know that $\E[|\bx_i^{(r)}\eps^{(r)}_i\|_2^2]\leq Cp$.
By Markov's inequality,
\[
  \P\left(\sum_{i\in\widetilde{\mathcal{N}}_r}\|\bx_i^{(r)}\eps^{(r)}_i\|_2^2\geq C\tn_r p\log \tn_r\right)\leq \frac{C}{\log \tn_r}.
  \]
  On the other hand,
  \[
     \sum_{i\in\tilde{\mathcal{N}}_r}\|\bx_i^{(r)}\widetilde{\Sig}_r^{-1}\bm{v}\|_2^2= \sum_{i\in\tilde{\mathcal{N}}_r}\bm{v}^{\top}\widetilde{\Sig}_r^{-1}\bx_i^{(r)}(\bx_i^{(r)})^{\top}\widetilde{\Sig}_r^{-1}\bm{v}=\tn_r\bm{v}^{\top}\widetilde{\Sig}_r^{-1}\bm{v}.
  \]
  We have shown that
\[
  \P\left(\|\widetilde{\Sig}_r^{-1}-\Sig_r^{-1}\|_2\leq \sqrt{\frac{p+\log \tn_r}{\tn_r}},~\|\hat{\btheta}_r^{(\uls)}-\btheta_r\|_2\leq \sqrt{\frac{p+\log \tn_r}{N_r}}+\omega_f\delta\sqrt{\frac{p+\log \tn_r}{\tn_r}} \right)\geq 1-\exp\{-c_1\log \tn_r\}.
\]
  Hence,
  \[
    \P\left(  \sum_{i\in\tilde{\mathcal{N}}_r}\|\bx_i^{(r)}\widetilde{\Sig}_r^{-1}\bm{v}\|_2^2\geq C\tn_r\right)\leq \exp\{-c_1 \tn_r\}.
  \]
Moreover,
\[
  \P\left(\max_{i\in\tilde{\mathcal{N}}_r}\|\bx_i\|_2^2\geq p\log \tn_r\right)\leq \exp\{-c_1\log \tn_r\}.
\]
To summarize, with probability at least $1-1/\log \tn_r$,
\begin{align*}
\sum_{i\in\widetilde{\mathcal{N}}_r}(\hat{a}_i^{(r)}-a_i^{(r)})^2
&\leq 2\|\bm{v}\|_2^2(p\log \tn_r(p+\log \tn_r)+\tn_rp\log \tn_r\|\hat{\btheta}^{(\uls)}_r-\btheta_r\|_2^2)\\
\frac{N_r-n_r}{N_r^2n_r}\sum_{i\in\widetilde{\mathcal{N}}_r}(\hat{a}_i^{(r)}-a_i^{(r)})^2
&\leq C\frac{\|\bm{v}\|_2^2(p+\log \tn_r)p\log \tn_r}{N_r\tn_r}.
\end{align*}
By (\ref{var-lowerbound}), the first equation in (\ref{toshow3}) holds assuming $\max\{p,\log\tn_r\}^2=o(\tn_r)$.

Similarly,
\begin{align*}
\sum_{i\in\widetilde{N}_r}(\hat{b}_i^{(r)}-b_i^{(r)})^2&\leq 4\sum_{i\in\widetilde{\mathcal{N}}_r}\{\bm{v}^{\top}(\widetilde{\Sig}_r^{-1}-\Sig_r^{-1})(\bx_i^{(r)}(\bx_i^{(r)})^{\top}-\Sig_r)(\hat{\btheta}_r^{(\uls)}-\hat{\btheta}_p)\}^2\\
&\quad+ 4\sum_{i\in\widetilde{\mathcal{N}}_r}\{\bm{v}^{\top}\widetilde{\Sig}_r^{-1}(\widetilde{\Sig}_r-\Sig_r)(\hat{\btheta}_r^{(\uls)}-\hat{\btheta}_p)\}^2\\
&+4\sum_{i\in\widetilde{\mathcal{N}}_r}\{\bm{v}^{\top}\Sig_r^{-1}(\bx_i^{(r)}(\bx_i^{(r)})^{\top}-\Sig_r)(\hat{\btheta}_r^{(\uls)}-\hat{\btheta}_p-\btheta_r+\btheta_p)\}^2\\
&\leq C\|\hat{\btheta}_r^{(\uls)}-\hat{\btheta}_p\|_2^2\|\bm{v}^{\top}(\widetilde{\Sig}_r^{-1}-\Sig_r^{-1})\|_2^2\sup_{\|\bm{u}\|_2=1}\sum_{i\in\widetilde{\mathcal{N}}_r}|\bm{u}^{\top}(\bx_i^{(r)}(\bx_i^{(r)})^{\top}-\Sig_r)\bm{u}|^2\\
&\quad+C\tn_r\|\bm{v}\|_2^2\|\widetilde{\Sig}_r-\Sig_r\|_2^2\|\hat{\btheta}_r^{(\uls)}-\hat{\btheta}_p\|_2^2\\
&\quad +C\|\bm{v}\|_2^2\|\hat{\btheta}_r^{(\uls)}-\hat{\btheta}_p-\btheta_r+\btheta_p\|_2^2\sup_{\|\bm{u}\|
_2=1}\sum_{i\in\widetilde{\mathcal{N}}_r}|\bm{u}^{\top}(\bx_i^{(r)}(\bx_i^{(r)})^{\top}-\Sig_r)\bm{u}|^2.
\end{align*}

By (\ref{unlearn-est}),  with probability at least $1-\exp\{-c_1\min\{N_f,\tn_r\}\}$,
\begin{align*}
\|\hat{\btheta}_r^{(\uls)}-\hat{\btheta}_p\|_2=\frac{N_f}{N_r}\|\widetilde{\Sig}_r^{-1}\widehat{\Sig}_f(\hat{\btheta}_f-\hat{\btheta}_p)\|_2\leq C\omega_f\delta.
\end{align*}
By Theorem \ref{thm2} and Lemma \ref{tlem2}, with probability at least $1-\exp\{-c_1\log n\}$,
\begin{align*}
\|\hat{\btheta}_r^{(\uls)}-\hat{\btheta}_p-\btheta_r+\btheta_p\|_2&\leq \|\hat{\btheta}_r^{(\uls)}-\btheta_r\|_2+\|\hat{\btheta}_p-\btheta_p\|_2\\
&\leq C\sqrt{\frac{p+\log \tn_r}{N_r}}+\omega_f\delta\sqrt{\frac{p+\log\tn_r}{\tn_r}}.
\end{align*}

For any given $\bm{u}\in\R^p$ such that $\|\bm{u}\|_2=1$, $\bm{u}^{\top}\bx_i^{(r)}(\bx_i^{(r)})^{\top}\bm{u}$ is sub-exponential. Using the concentration inequality for sub-Weibull random variables \citep{kuchibhotla2022moving}, we have
\begin{align*}
\P\left(
\sum_{i\in\widetilde{\mathcal{N}}_r}|\bm{u}^{\top}(\bx_i^{(r)}(\bx_i^{(r)})^{\top}-\Sig_r)\bm{u}|^2\geq t\right)\leq \exp\left\{-C\min\{\frac{t^2}{\tn_r},\sqrt{t}\}\right\}.
\end{align*}
Taking a uniform over the unit ball, we have
\[
   \P\left(
\sum_{i\in\widetilde{\mathcal{N}}_r}|\bm{u}^{\top}(\bx_i^{(r)}(\bx_i^{(r)})^{\top}-\Sig_r)\bm{u}|^2\geq \sqrt{\tn_rp}+p^2+\log\tn_r\right)\leq \exp\left\{-C\sqrt{\log \tn_r}\right\}.
\]
Combining above inequalities, we arrive at
\begin{align*}
\sum_{i\in\widetilde{N}_r}(\hat{b}_i^{(r)}-b_i^{(r)})^2&\lesssim\|\bm{v}\|_2^2\omega_f^2\delta^2\frac{p+\log\tn_r}{\tn_r}(\sqrt{\tn_rp}+p^2+\log\tn_r)+\|\bm{v}\|_2^2\omega_f^2\delta^2(p+\log \tn_r)\\
&\quad+\|\bm{v}\|_2^2(\frac{p+\log\tn_r}{N_r}+\omega_f^2\delta^2\frac{p+\log\tn_r}{\tn_r})(\sqrt{\tn_rp}+p^2+\log\tn_r)\\
&\lesssim\|\bm{v}\|_2^2\omega_f^2\delta^2(p+\log \tn_r)+\|\bm{v}\|_2^2\frac{\tn_r(p+\log\tn_r)}{N_r}
\end{align*}
given that $\tn_r\gg p^2$.

By (\ref{var-lowerbound}), we know that the second equation in (\ref{toshow3}) holds with probability at least $1-\exp\{-c_1\log \tn_r\}$.

Hence, we have shown that $T_5/V_r=o_P(1)$. Together with previous arguments, the proof is complete.
\end{proof}

\subsection{Proofs in Section \ref{sec6}}
\begin{proof}[Proof of Theorem \ref{thm-combine}]

The first-order condition gives
\[
  \omega_f(\widehat{M}_f-\widehat{\Sig}_f\hat{\btheta}^{(\uls+)}_r)-\widetilde{\Sig}_p(\hat{\btheta}_p-\hat{\btheta}^{(\uls+)}_r)-\lambda(\widetilde{M}_r-\widetilde{\Sig}_r\hat{\btheta}^{(\uls+)}_r)=0.
\]
Therefore,
\begin{align*}
\hat{\btheta}_r^{(\uls+)}&=\{(\omega_r+\lambda)\widetilde{\Sig}_r\}^{-1}(\widetilde{\Sig}_p\hat{\btheta}_p+\lambda\widetilde{M}_r-\omega_f\widehat{M}_f)\\
\hat{\btheta}^{(\uls+)}_r-\btheta_r&=\{(\omega_r+\lambda)\widetilde{\Sig}_r\}^{-1}(\widetilde{\Sig}_p\hat{\btheta}_p-\omega_f\widehat{M}_f-\omega_r\widetilde{\Sig}_r\btheta_r+\lambda\widetilde{E}_r)\\
&=\{(\omega_r+\lambda)\widetilde{\Sig}_r\}^{-1}(\omega_r\widetilde{\Sig}_r(\hat{\btheta}^{(\uls)}_r-\btheta_r)+\lambda\widetilde{E}_r).
\end{align*}
Using the results of Theorem \ref{thm2}, we know that with probability at least $1-\exp\{-c_1p\}$, we have
\begin{align*}
\|\hat{\btheta}^{(\uls+)}_r-\btheta_r\|_2&\leq C\frac{\omega_r}{\omega_r+\lambda}(\sqrt{\frac{p}{N_r}}+\omega_f\delta\sqrt{\frac{p}{\tilde{n}_r}})+C\frac{\lambda}{\omega_r+\lambda}\sqrt{\frac{p}{\tn_r}}\\
&\leq C\sqrt{\frac{p}{N_r}}+C\sqrt{\frac{p}{\tn_r}}(\frac{\omega_r}{\omega_r+\lambda}\omega_f\delta+\frac{\lambda}{\omega_r+\lambda}).
\end{align*}
Hence, for $\lambda=C\omega_r\omega_f\delta$ for any positive constant $C$, we have
\[
(\frac{\omega_r}{\omega_r+\lambda}\omega_f\delta+\frac{\lambda}{\omega_r+\lambda})\leq (1+C)\min \{\omega_f\delta,1\}.
\]
The proof is complete now.

\end{proof}

\begin{proof}[Proof of Theorem \ref{thm-gd}]
Let $\lambda^{(\textup{diff})}\geq 2\Lambda_{\max}(\Sig_f)/\Lambda_{\min}(\Sig_r)$  so that $\lambda^{(\textup{diff})}\widetilde{\Sig}_r-\widehat{\Sig}_f$ is positive definite with probability $1-\exp\{-c_1 p\}$. In this case,
\[
    \Lambda_{\min}(\lambda^{(\textup{diff})}\widetilde{\Sig}_r-\widehat{\Sig}_f)\geq c\lambda^{(\textup{diff})} \Lambda_{\min}(\Sig_r)
\]
for some positive constant $c$ with probability $1-\exp\{-c_1 p\}$.
\begin{align*}
\hat{\btheta}^{(\textup{diff})}_r=(\lambda^{(\textup{diff})}\widetilde{\Sig}_r-\widehat{\Sig}_f)^{-1}(\lambda^{(\textup{diff})}\widetilde{M}_r-\widehat{M}_f).
\end{align*}
Hence,
\begin{align*}
\hat{\btheta}^{(\textup{diff})}_r-\btheta_r&=(\lambda\widetilde{\Sig}_r-\widehat{\Sig}_f)^{-1}(\lambda\widetilde{M}_r-\lambda^{(\textup{diff})}\widetilde{\Sig}_r\btheta_r-\widehat{M}_f+\widehat{\Sig}_f\btheta_r)\\
&=\underbrace{(\lambda^{(\textup{diff})}\widetilde{\Sig}_r-\widehat{\Sig}_f)^{-1}(\lambda^{(\textup{diff})}\widetilde{E}_r-\widehat{E}_f)}_{T_8}+\underbrace{(\lambda^{(\textup{diff})}\widetilde{\Sig}_r-\widehat{\Sig}_f)^{-1}\widehat{\Sig}_f(\btheta_r-\btheta_f)}_{T_9}.
\end{align*}
We know that
\[
  \|T_9\|_2\leq \Lambda_{\min}^{-1}(\lambda^{(\textup{diff})}\widetilde{\Sig}_r-\widehat{\Sig}_f)\|\widehat{\Sig}_f\|_2\delta\leq C\Lambda_{\min}^{-1}(\lambda^{(\textup{diff})}\widetilde{\Sig}_r-\widehat{\Sig}_f)\delta,
\]
with probability at least $1-\exp\{-c p\}$.

For $T_8$, with probability at least $1-\exp\{-c p\}$, we have
\begin{align*}
\|T_8\|_2\leq C\Lambda_{\min}^{-1}(\lambda^{(\textup{diff})}\widetilde{\Sig}_r-\widehat{\Sig}_f)(\lambda^{(\textup{diff})} \sqrt{\frac{p}{\tn_r}}+\sqrt{\frac{p}{N_f}}).
\end{align*}
Therefore, with probability at least $1-\exp\{-c p\}$,
\begin{align*}
\|\hat{\btheta}^{(\textup{diff})}_r-\btheta_r\|_2&\leq C\Lambda_{\min}^{-1}(\lambda^{(\textup{diff})}\widetilde{\Sig}_r-\widehat{\Sig}_f)(\lambda^{(\textup{diff})} \sqrt{\frac{p}{\tn_r}}+\sqrt{\frac{p}{N_f}}+\delta)\\
&\leq C\Lambda^{-1}_{\min}(\Sig_r)( \sqrt{\frac{p}{\tn_r}}+\frac{\sqrt{\frac{p}{N_f}}+\delta}{\lambda^{(\textup{diff})}}).
\end{align*}
 For 
 \begin{align*}
   \lambda^{(\textup{diff})}&=\max\left\{\frac{\sqrt{\frac{p}{N_f}}+\delta}{\sqrt{p/\tn_r}},2\Lambda_{\max}(\Sig_f)/\Lambda_{\min}(\Sig_r)\right\}\\
   &=\max\left\{\sqrt{\frac{\tilde{\omega}_r}{\omega_f}}+\sqrt{\frac{\tn_r}{p}}\delta,2\Lambda_{\max}(\Sig_f)/\Lambda_{\min}(\Sig_r)\right\},
 \end{align*}
 we have
 \begin{align*}
\|\hat{\btheta}^{(\textup{diff})}_r-\btheta_r\|_2
&\leq C\sqrt{\frac{p}{\tn_r}}+C\min\left\{\sqrt{\frac{p}{N_f}}+\delta,\sqrt{\frac{p}{\tn_r}}\right\}\\
&\leq C'\sqrt{\frac{p}{\tn_r}}.
\end{align*}

\end{proof}

\subsection{Proof of Theorem \ref{thm-minimax} and Theorem \ref{thm-minimax2}}
We see that Theorem \ref{thm-minimax} is a special case of Theorem \ref{thm-minimax2}. Hence, it suffices to prove the more general theorem, Theorem \ref{thm-minimax2}.
\begin{proof}[Proof of Theorem \ref{thm-minimax2}]

\underline{\textbf{Preliminaries.}}
We will prove this bound based on Theorem 8 of \citet{ma2024high}.  Specifically, we will first construct $B_p(\delta)\subset \Theta_p(\delta)$ and let $D=\max_{\bbeta,\bbeta'\in B_p(\delta)} \|E_{1:p,}(\bbeta-\bbeta')\|_2$ for $E_{1:p,}=(I_p, 0)\in\R^{p\times \text{dim}(\bbeta)}$.
We show that
\[
    \inf_{\hat{\btheta}_r\in\mathcal{F}(\hat{\btheta}_p,\calD_f,\widetilde{\calD}_r)}\sup_{B_p(\delta)} \E[\|\hat{\btheta}_r-\btheta_r\|_2^2]\geq \rho^2 D^2
\]
for some constant $\rho>0$. Then Theorem 8 of \citet{ma2024high} implies that
\begin{align}
\label{lower-prob}
  \inf_{\hat{\btheta}_r\in\mathcal{F}(\hat{\btheta}_p,\calD_f,\widetilde{\calD}_r)}\sup_{\Theta_p(\delta)} \P(\||\hat{\btheta}_r-\btheta_r\|_2^2\geq c^2D^2)\geq \rho^2/2
\end{align}
for some small enough constant $c$. 

We first describe the construction of $B_p(\delta)$.
Let $\Phi=\{\bm\phi^{(1)},\dots,\bm\phi^{(M)}\}=\{0,1\}^p$ for $M=2^p$. 
Let $\bbeta^{(j)}=(\btheta_{r}^{(j)},\btheta_{f}^{(j)},\Sig_r^{(j)},\Sig_f^{(j)},(\sig^{(j)}_r)^2,\sig^2_f)$. With parameter $\bbeta^{(j)}$, we specify the data distribution as
\begin{align*}
 y^{(r)}_i&=(\bx^{(r)}_{i})^{\top}\btheta^{(j)}_r+\eps^{(r)}_{i},~\bx^{(r)}_{i}\sim N(0,\Sig^{(j)}_r),~\eps^{(r)}_{i}\sim_{i.i.d} N(0,(\sig_r^{(j)})^2),~i=1,\dots,N_r,\\
y^{(f)}_i&=(\bx^{(f)}_{i})^{\top}\btheta^{(j)}_f+\eps^{(f)}_{i},~\bx^{(f)}_{i}\sim N(0,\Sig^{(j)}_f),~\eps^{(f)}_{i}\sim_{i.i.d} N(0,\sig_f^2),~i=1,\dots,N_f.
\end{align*}
We set  $\sig_f^2=1$ throughout the proof. 

Define $d_j(\bbeta,\bbeta')=|e_j^{\top}(\bbeta-\bbeta')|$ and $g(x)=x^2$. Then 
\[
      \|\btheta_{r}-\btheta_{r}'\|_2^2=\sum_{j=1}^pg(d_j(\bbeta,\bbeta')).
\]
Following the notations in \citet{ma2024high}. For $\bm\phi,\bm\phi'\in\Phi$, write $\bm\phi\sim\bm\phi'$ whenever $\bm\phi$ and $\bm\phi'$ differ in precisely one coordinate, and $\bm\phi\sim_j\bm\phi'$ when that coordinate is the $j$-th.

\underline{\textbf{Part 1.}} 
 Let
\[
\btheta_{f}^{(j)}=\btheta_{r}^{(j)}=c_rN^{-1/2}\bm\phi^{(j)}~\text{and}~ \Sig_r^{(j)}=\Sig_f^{(j)}=I_p,
\] where $c_r$ is a constant determined later. Let $B_p(\delta)=\{\bbeta^{(1)},\dots,\bbeta^{(M)}\}$. We note that
$B_p(\delta)\subseteq\Theta_p(\delta)$ for any $\delta\geq 0$.

We first show that there exists some constant $C>0$ such that
\begin{align}
  \inf_{\hat{\btheta}_r\in\mathcal{F}(\hat{\btheta}_p,\calD_f,\widetilde{\calD}_r)}\sup_{B_p(\delta)} \E[\|\hat{\btheta}_r-\btheta_r\|_2^2]\geq \frac{Cp}{N}.\label{lb1}
\end{align}
As $\hat{\btheta}_p$ is a function of $\calD_r$ and $\calD_f$ and $\widetilde{\calD}_r\subseteq \calD_r$, we know that
$\mathcal{F}(\hat{\btheta}_p,\calD_f,\widetilde{\calD}_r)\subseteq\mathcal{F}(\calD_f,\calD_r)$. Hence,
\begin{align*}
  \inf_{\hat{\btheta}_r\in\mathcal{F}(\hat{\btheta}_p,\calD_f,\widetilde{\calD}_r)}\sup_{B_p(\delta)} \E[\|\hat{\btheta}_r-\btheta_r\|_2^2]\geq   \inf_{\hat{\btheta}_r\in\mathcal{F}(\calD_f,\calD_r)}\sup_{B_p(\delta)} \E[\|\hat{\btheta}_r-\btheta_r\|_2^2]
\end{align*}
and it suffices to show that
\begin{align}
  \inf_{\hat{\btheta}_r\in\mathcal{F}(\calD_f,\calD_r)}\sup_{B_p(\delta)} \E[\|\hat{\btheta}_r-\btheta_r\|_2^2]\geq \frac{Cp}{N}.\label{lb1-1}
\end{align}

 Let $\alpha_j=d_j(\bbeta,\bbeta')$ for $\bbeta,\bbeta'\in B_p(\delta)$ and $\bbeta\sim _j\bbeta'$. By our construction of $\bbeta^{(j)}$, $j=1,\dots,M$, $\alpha_j=c_rN^{-1/2}$ and
\[
   \sum_{j=1}^pg(\alpha_j)=p(c_rN^{-1/2})^2=c_r^2p/N.
\]

By Lemma 23 in \citet{ma2024high}, we know that
\begin{align}
\inf_{\hat{\btheta}_r\in\mathcal{F}(\calD_f,\calD_r)}\sup_{B_p(\delta)} \E[\|\hat{\btheta}_r-\btheta_r\|_2^2]\geq \frac{1}{2A}\{1-\max_{\bbeta,\bbeta'\in B_p(\delta):\bbeta\sim \bbeta'}\textup{TV}(p_{\bbeta}(\calD_f,\calD_r),p_{\bbeta'}(\calD_f,\calD_r))\}\sum_{j=1}^pg(\alpha_j),\label{tv0}
\end{align}
where $\textup{TV}(p_{\bbeta}(\calD_f,\calD_r),p_{\bbeta'}(\calD_f,\calD_r))$ denotes the total variation distance between the density of $(\calD_f,\calD_r)$ under $\bbeta$ and that under $\bbeta'$. 

It is left to bound $\max_{\bbeta,\bbeta'\in B_p(\delta):\bbeta\sim \bbeta'}\textup{TV}(p_{\bbeta}(\calD_f,\calD_r),p_{\bbeta'}(\calD_f,\calD_r))$.
Note that the retain samples and forget samples are \textit{i.i.d.} under $P_{\bbeta^{(j)}}$ such that 
\[
    ((\bx^{(f)}_i)^{\top},y^{(f)}_i)\sim_{\bbeta^{(j)}}((\bx^{(r)}_i)^{\top},y^{(r)}_i)\sim_{\bbeta^{(j)}} N(0,\Omega^{(j)})~\text{where}~\Omega^{(j)}=\begin{pmatrix}
   \|\btheta_r^{(j)}\|_2^2+1 & (\btheta_r^{(j)})^{\top}\\
   \btheta_r^{(j)} & I_p\end{pmatrix}.
\]

Hence,
\begin{align}
\textup{TV}^2(p_{\bbeta^{(j)}}(\calD_f,\calD_r),p_{\bbeta^{(k)}}(\calD_f,\calD_r))&\leq \textup{KL}(p_{\bbeta^{(j)}}(\calD_f,\calD_r),p_{\bbeta^{(k)}}(\calD_f,\calD_r))\nonumber\\
&=\frac{N}{2}\{\textup{Tr}((\Omega^{(k)})^{-1}(\Omega^{(j)}-\Omega^{(k)}))-\log det((\Omega^{(k)})^{-1}\Omega^{(j)})\}.\label{tv1}
\end{align}
Let $\lambda_1\geq \dots\geq \lambda_p$ denote the singular values of $(\Omega^{(k)})^{-1}( \Omega^{(k)}-\Omega^{(j)})$. Given that $\lambda_1\leq 1/2$, we have
\begin{align}
&\textup{Tr}((\Omega^{(k)})^{-1}(\Omega^{(j)}-\Omega^{(k)}))-\log det((\Omega^{(k)})^{-1}\Omega^{(j)})=\sum_{l=1}^p\{-\lambda_l-\log(1-\lambda_l)\}\nonumber\\
&\quad\leq \sum_{l=1}^p \sum_{k=2}^{\infty}\lambda_l^k/k\leq\sum_{l=1}^p\lambda_l^2=\|(\Omega^{(k)})^{-1}(\Omega^{(j)}-\Omega^{(k)})\|_F^2.\label{tv2}
\end{align}

We can calculate that
\begin{align*}
(\Omega^{(k)})^{-1}&=\begin{pmatrix}
1 &-\btheta_{r}^{(k)}\\
-\btheta_{r}^{(k)} &I_p+\btheta_{r}^{(k)}(\btheta_{r}^{(k)})^{\top}
\end{pmatrix}\\
\Omega^{(j)}-\Omega^{(k)}&=\begin{pmatrix}
\|\btheta_{r}^{(k)}\|_2^2- \|\btheta_{r}^{(j)}\|_2^2& (\btheta_{r}^{(k)}-\btheta_{r}^{(j)})^{\top}\\
\btheta_{r}^{(k)}-\btheta_{r}^{(j)} & 0
\end{pmatrix} \\
(\Omega^{(k)})^{-1}(\Omega^{(j)}-\Omega^{(k)})&=\begin{pmatrix}
(\btheta_{r}^{(k)}-\btheta_{r}^{(j)})^{\top}\btheta_{r}^{(j)}& (\btheta_{r}^{(k)}-\btheta_{r}^{(j)})^{\top}\\
-\btheta_{r}^{(k)}((\btheta_{r}^{(k)})^{\top}\btheta_{r}^{(j)}- \|\btheta_{r}^{(j)}\|_2^2)+\btheta_{r}^{(k)}-\btheta_{r}^{(j)}&-\btheta_{r}^{(k)}(\btheta_{r}^{(k)}-\btheta_{r}^{(j)})^{\top}
\end{pmatrix}.
\end{align*}
For  $\btheta^{(k)}_{r}\sim \btheta^{(j)}_{r}$, $\btheta^{(k)}_{r}$ and $\btheta^{(j)}_{r}$ only differ at one coordinate. Suppose they differ at the $l$-th coordinate. Then $\btheta_{r}^{(k)}-\btheta_{r}^{(j)}=\pm \frac{c_r}{\sqrt{N}}\bm{e}_l$ and
\begin{align*}
(\Omega^{(k)})^{-1}(\Omega^{(j)}-\Omega^{(k)})&=\begin{pmatrix}
\pm\frac{c_r}{\sqrt{N}}\bm{e}_l^{\top}\btheta_{r}^{(j)}&\frac{c_r}{\sqrt{N}}\bm{e}_l^{\top}\\
-\frac{c_r}{\sqrt{N}}(I_p-\btheta_{r}^{(k)}(\btheta_{r}^{(j)})^{\top})\bm{e}_l&\frac{c_r}{\sqrt{N}}\btheta_{r}^{(k)}\bm{e}_l^{\top}
\end{pmatrix}.
\end{align*}
Hence,
\begin{align*}
\|(\Omega^{(k)})^{-1}(\Omega^{(j)}-\Omega^{(k)})\|_F^2&=\frac{c_r^2}{N}(\bm{e}_l^{\top}\btheta_{r}^{(j)})^2+\frac{c_r^2}{N}+\frac{c_r^2}{N}\|(I_p-\btheta_{r}^{(k)}(\btheta_{r}^{(j)})^{\top})\bm{e}_l\|_2^2+\frac{c_r^2}{N}\|\btheta_r^{(k)}\|_2^2\\
&\leq \frac{c_r^2(1+\max_{j\leq M}\|\btheta_r^{(j)}\|_2^2)}{N}\leq \frac{c_r^2(1+c_r^2p/N)}{N}\leq \frac{Cc_r^2}{N}.
\end{align*}
where the last line is due to $\|(I_p-\btheta_{r}^{(k)}(\btheta_{r}^{(j)})^{\top})\bm{e}_l\|^2_2\leq 1$ and $p=o(N)$.

By (\ref{tv1}) and (\ref{tv2}), we arrive at
\begin{align*}
\max_{1\leq j\leq k\leq M}\textup{TV}^2(p_{\bbeta^{(j)}}(\calD_f,\calD_r),p_{\bbeta^{(k)}}(\calD_f,\calD_r))\leq \frac{N}{2}\|(\Omega^{(k)})^{-1}(\Omega^{(j)}-\Omega^{(k)})\|_F^2\leq Cc_r^2\leq 1/2
\end{align*}
for some small enough constant $c_r$.

In view of (\ref{tv0}), by choosing $c_r$ to be a small enough constant, we arrive at
\begin{align*}
\inf_{\hat{\btheta}_r\in\mathcal{F}(\calD_f,\calD_r)}\sup_{B_p(\delta)} \E[\|\hat{\btheta}_r-\btheta_r\|_2^2]\geq \frac{Cp}{N},
\end{align*}
which concludes the proof of (\ref{lb1}).

On the other hand,
\begin{align*}
D=\max_{\bbeta,\bbeta'\in B_p(\delta)}\|E_{1:p,.}(\bbeta-\bbeta')\|_2=c_r\sqrt{p/N}.
\end{align*}
By (\ref{lower-prob}), we arrive at
\begin{align*}
  \inf_{\hat{\btheta}_r\in\mathcal{F}(\hat{\btheta}_p,\calD_f,\widetilde{\calD}_r)}\sup_{\Theta_p(\delta)} \P(\||\hat{\btheta}_r-\btheta_r\|_2^2\geq \frac{c_1p}{N})\geq c_2
\end{align*}
for some small enough positive constants $c_1$ and $c_2$.

\underline{\textbf{Part 2.}} Next, we show that when $\tn_r\leq c_1N$,
\begin{align}
  \inf_{\hat{\btheta}_r\in\mathcal{F}(\hat{\btheta}_p,\calD_f,\widetilde{\calD}_r)}\sup_{\Theta_p(\delta)} \E[\|\hat{\btheta}_r-\btheta_r\|_2^2]\geq \min\{\omega_f^2\delta^2,1\}\frac{p}{\tilde{n}_r}.\label{lb2}
\end{align}

Analogous to part (i), we will use Assouad's lemma. 
Let $\bm\psi^{(j)}=\frac{c_r}{\sqrt{\tn_r}\delta}\bm\phi^{(j)}$, 
\begin{align}
\Sig_{r}^{(j)}&=I_p+\bm\psi^{(j)}\btheta_f^{\top}+\btheta_f(\bm\psi^{(j)})^{\top}, ~~\Sig_f^{(j)}=I_p,\nonumber\\
\btheta_{r}^{(j)}&=\btheta_f-(\omega_r\Sig_{r}^{(j)})^{-1}\Sig_{p}^{(j)}\btheta_f,~\btheta_{f}^{(j)}=\btheta_f=\frac{\omega_r\min\{\delta,c_0\omega_f^{-1}\}}{\sqrt{p}}\mathbf{1}_p\nonumber\\
(\sig_r^{(j)})^2&=5-(\btheta^{(j)}_r)^{\top}\Sig_r^{(j)}\btheta_r^{(j)},~\sig^2_f=1.\label{theta-def}
\end{align}
for some constant $0<c_0\leq 1$.
It is easy to verify the following facts: For any $1\leq j\leq k\leq M$,
\begin{align}
&\Sig_r^{(j)}\btheta_r^{(j)}=\Sig_r^{(k)}\btheta_r^{(k)}=\frac{\omega_f}{\omega_r}\Sig_f\btheta_f\nonumber\\
&\omega_f\Sig_f\btheta_f+\omega_r\Sig_r^{(j)}\btheta_r^{(j)}=0\nonumber\\
&\|\btheta_r^{(j)}\|_2=\frac{\omega_f}{\omega_r}\|(\Sig_r^{(j)})^{-1}\Sig_f\btheta_f\|_2\leq\frac{\omega_f}{\omega_r}\ \|\btheta_f\|_2\leq  \min\{\omega_f\delta,1\},\label{facts}
\end{align}
where the last line is due to $\Sig_r^{(j)}\succeq I_p$ and $\|(\Sig_{r}^{(j)})^{-1}\|_2\leq 1$. As
\[
(\btheta_r^{(j)})^{\top}\Sig_r^{(j)}\btheta_r^{(j)}\leq \|\Sig_r^{(j)}\btheta_r^{(j)}\|_2^2\leq \min\{\omega_f^2\delta^2,1\}\leq 1,\]
we know that $(\sig_r^{(j)})^2=5-(\btheta_r^{(j)})^{\top}\Sig_r^{(j)}\btheta_r^{(j)}\geq1$ for all $1\leq j\leq M$.
We now verify that
$B_p(\delta)=\{\bbeta^{(1)},\dots,\bbeta^{(M)}\}\subseteq\Theta_p(\delta)$ for any $\delta\geq 0$. 
By (\ref{theta-def}) and (\ref{facts}),
\begin{align*}
  & \max_{1\leq j\leq M}\|\btheta_{r}^{(j)}-\btheta_{f}^{(j)}\|_2\leq  \max_{j\leq M}\|\btheta_r^{(j)}\|_2+\|\btheta_f\|_2\leq \|\btheta_f\|_2/\omega_r\leq \delta.
\end{align*}
We can also check that for any $c_1>1$ there exists some small enough $c_r>0$ such that for all $1\leq j\leq M$,
\[
\Lambda_{\min}(\Sig_r^{(j)})\geq 1-2\|\bm\psi^{(j)}\|_2\|\btheta_f\|_2\geq 1-\frac{c_r\sqrt{p}}{\sqrt{\tn_r}}\geq 1/c_1
\] and 
\[
 \Lambda_{\max}(\Sig_r^{(j)})\leq 1+2\|\bm\psi^{(j)}\|_2\|\btheta_f\|_2\leq 1+\frac{c_r\sqrt{p}}{\sqrt{\tn_r}}\leq c_1.
 \]
We have verified that $B_p(\delta)=\{\bbeta^{(1)},\dots,\bbeta^{(M)}\}\subseteq\Theta_p(\delta)$ for any $\delta\geq 0$.

\textbf{Part 2(i).}  By Lemma \ref{lem-mini0}, we know that
\begin{align*}
   \sum_{j=1}^pg(\alpha_j)\geq \frac{c_r^2p\min\{(\omega_f\delta)^2,1\}}{\tn_r}.
  \end{align*}

\textbf{Part 2(ii). Upper bound total variation distance}.

In view of (\ref{tv0}) is left to show that
\begin{align*}
\max_{\bbeta^{(j)},\bbeta^{(k)}\in B_p(\delta):\bbeta^{(j)}\sim \bbeta^{(k)}}\textup{TV}(p_{\bbeta^{(j)}}(\hat{\btheta}_p,\calD_f,\widetilde{\calD}_r),p_{\bbeta^{(k)}}(\hat{\btheta}_p,\calD_f,\widetilde{\calD}_r))\leq 1/2.
\end{align*}
Define
\begin{align}
\label{eq-z}
\bm{z}_r=\bm{y}_r-X_r\btheta_p~\text{and} ~\bm{z}_f=\bm{y}_f-X_f\btheta_p.
\end{align}

Note that by the definition of $\btheta_p$,
\begin{align}
\hat{\btheta}_p&=\btheta_p+\frac{1}{N}\widehat{\Sig}_p^{-1}(X_r^{\top}\bm{y}_r/N+X_f^{\top}\bm{y}_f/N-\omega_r\widehat{\Sig}_r\btheta_p-\omega_f\widehat{\Sig}_f\btheta_p)\nonumber\\
&=\btheta_p+\frac{1}{N}\widehat{\Sig}_p^{-1}(X_r^{\top}\bz_r+X_f^{\top}\bz_f)\nonumber\\
&:=\underbrace{\btheta_p+\frac{1}{N}\Sig_p^{-1}(X_r^{\top}\bz_r+X_f^{\top}\bz_f)}_{\mathring{\btheta}_p}+\hat{\bu},\label{eq-thetap-r}
\end{align}
for any $1\leq j\leq p$, where 
\[
   \hat{\bu}=\frac{1}{N}(\widehat{\Sig}_p^{-1}-\Sig_p^{-1})(X_r^{\top}\bz_r+X_f^{\top}\bz_f).
\]

Therefore,
\begin{align}
&\max_{\bbeta^{(j)},\bbeta^{(k)}\in B_p(\delta):\bbeta^{(j)}\sim\bbeta^{(k)}}\textup{TV}(p_{\bbeta^{(j)}}(\hat{\btheta}_p,\calD_f,\widetilde{\calD}_r),p_{\bbeta^{(k)}}(\hat{\btheta}_p,\calD_f,\widetilde{\calD}_r))\nonumber\\
&\quad\leq\max_{\bbeta^{(j)},\bbeta^{(k)}\in B_p(\delta):\bbeta^{(j)}\sim\bbeta^{(k)}}\textup{TV}(p_{\bbeta^{(j)}}(\mathring{\btheta}_p,\calD_f,\widetilde{\calD}_r),p_{\bbeta^{(k)}}(\mathring{\btheta}_p,\calD_f,\widetilde{\calD}_r))\nonumber\\
&\quad\quad+2\max_{\bbeta^{(j)}\in B_p(\delta)}\textup{TV}(p_{\bbeta^{(j)}}(\hat{\btheta}_p,\calD_f,\widetilde{\calD}_r),p_{\bbeta^{(j)}}(\mathring{\btheta}_p,\calD_f,\widetilde{\calD}_r)).\label{tv-decomp0}
\end{align}
We now bound the first term of (\ref{tv-decomp0}). Note that
\begin{align}
&\max_{\bbeta^{(j)}\sim\bbeta^{(k)}}\textup{TV}(p_{\bbeta^{(j)}}(\mathring{\btheta}_p,\calD_f,\widetilde{\calD}_r),p_{\bbeta^{(k)}}(\mathring{\btheta}_p,\calD_f,\widetilde{\calD}_r))\nonumber\\
&=\max_{\bbeta^{(j)}\sim\bbeta^{(k)}}\frac{1}{2}\int | p_{\bbeta^{(j)}}(\mathring{\btheta}_p|\calD_f,\widetilde{\calD}_r)p_{\bbeta^{(j)}}(\calD_f,\widetilde{\calD}_r)-p_{\bbeta^{(k)}}(\mathring{\btheta}_p|\calD_f,\widetilde{\calD}_r)p_{\bbeta^{(k)}}(\calD_f,\widetilde{\calD}_r)|d(\mathring{\btheta}_p,\calD_f,\widetilde{\calD}_r)\nonumber\\
&\leq  \max_{\bbeta^{(j)}\sim\bbeta^{(k)}}\frac{1}{2}\int |p_{\bbeta^{(j)}}(\mathring{\btheta}_p|\calD_f,\widetilde{\calD}_r)-p_{\bbeta^{(k)}}(\mathring{\btheta}_p|\calD_f,\widetilde{\calD}_r)|p_{\bbeta^{(j)}}(\calD_f,\widetilde{\calD}_r)d(\mathring{\btheta}_p,\calD_f,\widetilde{\calD}_r)\nonumber\\
&\quad+\max_{\bbeta^{(j)}\sim\bbeta^{(k)}}\frac{1}{2}\int p_{\bbeta^{(k)}}(\mathring{\btheta}_p|\calD_f,\widetilde{\calD}_r)|p_{\bbeta^{(j)}}(\calD_f,\widetilde{\calD}_r)-p_{\bbeta^{(k)}}(\calD_f,\widetilde{\calD}_r)|d(\hat{\btheta}_p,\calD_f,\widetilde{\calD}_r)\nonumber\\
&= \max_{\bbeta^{(j)}\sim\bbeta^{(k)}}\E_{\bbeta^{(j)}}[\textup{TV}(p_{\bbeta^{(j)}}(\mathring{\btheta}_p|\calD_f,\widetilde{\calD}_r),p_{\bbeta^{(k)}}(\mathring{\btheta}_p|\calD_f,\widetilde{\calD}_r))]+ \max_{\bbeta^{(j)}\sim\bbeta^{(k)}}\textup{TV}(p_{\bbeta^{(j)}}(\widetilde{\calD}_r),p_{\bbeta^{(k)}}(\widetilde{\calD}_r)),\label{tv-decomp}
\end{align}
where the last step  is due to the distribution of $\calD_f$ is unchanged under $\bbeta^{(j)}$, $j=1,\dots,M$. We will bound each term separately.

(ii-1) First, we bound the second term of (\ref{tv-decomp}). In Lemma \ref{lem-mini1}, we show that
\[
\max_{\bbeta^{(j)},\bbeta^{(k)}\in B_p(\delta):\bbeta^{(j)}\sim \bbeta^{(k)}}\textup{TV}(p_{\bbeta^{(j)}}(\widetilde{\calD}_r),p_{\bbeta^{(k)}}(\widetilde{\calD}_r))\leq Cc_r.
\]

(ii-2) Next, we bound the first term of (\ref{tv-decomp}). We first find the distribution of $\mathring{\btheta}_p$ conditioning on $\calD_r$ and $\calD_f$.
 Let $\check{N}_r=N_r-\tn_r$. Let $\check{X}_r\in\R^{\check{N}_r\times p}$ and $\check{\bm{z}}_r\in\R^{\check{N}_r}$ denote the samples in $\calD_r\setminus \widetilde{\calD}_r$. 
Let $\check{\omega}_r=\check{N}_r/N$.

As  $\btheta^{(j)}_{p}=0$, for $j=1,\dots, M$, we have
\begin{align*}
&z_{r,i}\sim_{\bbeta^{(j)}} N(0,(\btheta_r^{(j)})^{\top}\Sig_r^{(j)}\btheta_r^{(j)}+(\sig_r^{(j)})^2)=N(0,5), ~j=1,\dots,M.\\
&z_{f,i}\sim_{\bbeta^{(j)}} N(0,\|\btheta_f\|_2^2+1),~j=1,\dots,M.
\end{align*}
We see that the distribution of $\bm{z}_r$ and $\bm{z}_f$ are invariant under different hypotheses.
Moreover, $\check{\bz}_r$ is independent of $\calD_f$ and $\widetilde{\calD}_r$.
Hence,
\begin{align*}
&\max_{\bbeta^{(j)}\sim\bbeta^{(k)}}\E_{\bbeta^{(j)}}[\textup{TV}(p_{\bbeta^{(j)}}(\mathring{\btheta}_p|\calD_f,\widetilde{\calD}_r),p_{\bbeta^{(k)}}(\mathring{\btheta}_p|\calD_f,\widetilde{\calD}_r))]\\
&\quad\leq \E_{\bbeta^{(j)}}[\textup{TV}(p_{\bbeta^{(j)}}(\mathring{\btheta}_p|\calD_f,\widetilde{\calD}_r,\check{\bz}_r),p_{\bbeta^{(k)}}(\mathring{\btheta}_p|\calD_f,\widetilde{\calD}_r,\check{\bz}_r))].
\end{align*}
By Lemma \ref{lem-T1}, we arrive at
\begin{align*}
\max_{\bbeta^{(j)}\sim\bbeta^{(k)}}\E_{\bbeta^{(j)}}[\textup{TV}(p_{\bbeta^{(j)}}(\mathring{\btheta}_p|\calD_f,\widetilde{\calD}_r),p_{\bbeta^{(k)}}(\mathring{\btheta}_p|\calD_f,\widetilde{\calD}_r))]\leq Cc_r\sqrt{\frac{p}{\tn_r}}.
\end{align*}


\textbf{Part 2(iii). Upper bound the second term of (\ref{tv-decomp0})}.
By Lemma \ref{lem-T2},
\[
\max_{j\leq M}\textup{TV}(p_{\bbeta^{(j)}}(\hat{\btheta}_p,\calD_f,\widetilde{\calD}_r),p_{\bbeta^{(j)}}(\mathring{\btheta}_p,\calD_f,\widetilde{\calD}_r))\leq C'\sqrt{\frac{p(p+\log\tn_r)}{N}}+\frac{1}{\tn_r}.
\]

\textbf{Part 2(iv). Summarization.}

By (\ref{tv-decomp0}), and the results in Part 2(ii) and Part 2(iii), we know that
\begin{align*}
\max_{\bbeta^{(j)}\sim\bbeta^{(k)}}\textup{TV}(p_{\bbeta^{(j)}}(\hat{\btheta}_p,\calD_f,\widetilde{\calD}_r),p_{\bbeta^{(k)}}(\hat{\btheta}_p,\calD_f,\widetilde{\calD}_r))\leq C(c_r+c_r\sqrt{\frac{p}{\tn_r}}+\sqrt{\frac{p(p+\log\tn_r)}{N}}+\frac{1}{\tn_r}).
\end{align*}
As we assume $p\leq c_0\min\{\tn_r,\sqrt{N}\}$ for some small enough constant $c_0$, there exists some small enough constant $c_r$, such that
\[
\max_{\bbeta^{(j)}\sim\bbeta^{(k)}}\textup{TV}(p_{\bbeta^{(j)}}(\hat{\btheta}_p,\calD_f,\widetilde{\calD}_r),p_{\bbeta^{(k)}}(\hat{\btheta}_p,\calD_f,\widetilde{\calD}_r))\leq1/2.
\]
On the other hand,
\begin{align*}
D=\max_{\bbeta,\bbeta'\in B_p(\delta)}\|E_{1:p,.}(\bbeta-\bbeta')\|_2=\max_{j,k\leq M}\|\btheta^{(j)}-\btheta^{(k)}\|_2=c_r\sqrt{\frac{p}{\tn_r}}\min\{\omega_f\delta,1\}.
\end{align*}
By (\ref{lower-prob}), we arrive at
\begin{align*}
  \inf_{\hat{\btheta}_r\in\mathcal{F}(\hat{\btheta}_p,\calD_f,\widetilde{\calD}_r)}\sup_{\Theta_p(\delta)} \P\left(\||\hat{\btheta}_r-\btheta_r\|_2^2\geq c_1\frac{p}{\tn_r}\min\{\omega_f^2\delta^2,1\}\right)\geq c_2
\end{align*}
for some small enough positive constants $c_1$ and $c_2$.
\end{proof}
\begin{lemma}
\label{lem-mini0}
Assume the conditions of Theorem \ref{thm-minimax2}.  Under the distribution configuration in (\ref{theta-def}), it holds that
\begin{align*}
   \sum_{j=1}^pg(\alpha_j)= \sum_{j=1}^p\alpha_j^2\geq \frac{c_r^2p\min\{(\omega_f\delta)^2,1\}}{\tn_r}.
  \end{align*}
\end{lemma}
\begin{proof}[Proof of Lemma \ref{lem-mini0}]
We know that
\begin{align}
\alpha_l&=|\bm{e}_l^{\top}(\btheta_r^{(j)}-\btheta_r^{(k)})|\nonumber\\
&=|\bm{e}_l^{\top}\{(\omega_r\Sig_{r}^{(j)})^{-1}\Sig_{p}^{(j)}\btheta_f-(\omega_r\Sig_{r}^{(k)})^{-1}\Sig_p^{(k)}\btheta_f\}|\nonumber\\
   &=\frac{\omega_f}{\omega_r}|\bm{e}_l^{\top}\{(\Sig_{r}^{(j)})^{-1}-(\Sig_{r}^{(k)})^{-1}\}\btheta_f|.\label{eq-alphaj}
\end{align}
We can calculate that
\begin{align*}
&\{(\Sig_{r}^{(j)})^{-1}-(\Sig_{r}^{(k)})^{-1}\}\Sig_f\btheta_f\\
&= (\Sig_r^{(j)})^{-1}(\Sig_r^{(k)}-\Sig_r^{(j)})(\Sig_r^{(k)})^{-1}\btheta_f\\
&=(\Sig_r^{(k)}-\Sig_r^{(j)})\btheta_f+(\Sig_r^{(j)})^{-1}(\Sig_r^{(k)}-\Sig_r^{(j)})((\Sig_r^{(k)})^{-1}-I_p)\btheta_f+((\Sig_r^{(j)})^{-1}-I_p)(\Sig_r^{(k)}-\Sig_r^{(j)})\btheta_f.
\end{align*}
Therefore,
\begin{align*}
\alpha_l&\geq \underbrace{\frac{\omega_f}{\omega_r}|\bm{e}_l^{\top}(\Sig_r^{(k)}-\Sig_r^{(j)})\btheta_f|}_{R_{1}}-\underbrace{\frac{\omega_f}{\omega_r}\|(\Sig_r^{(j)})^{-1}(\Sig_r^{(k)}-\Sig_r^{(j)})(I_p-\Sig_r^{(k)})(\Sig_r^{(k)})^{-1}\btheta_f\|_2}_{R_{2}}\\
&\quad-\underbrace{\frac{\omega_f}{\omega_r}\|(\Sig_r^{(j)})^{-1}(I_p-\Sig_r^{(j)})(\Sig_r^{(k)}-\Sig_r^{(j)})\btheta_f\|_2}_{R_{3}}.
\end{align*}
As 
\begin{align}
\Sig_r^{(k)}-\Sig_r^{(j)}=(\bm\psi^{(k)}-\bm\psi^{(j)})\btheta_f^{\top}+\btheta_f(\bm\psi^{(k)}-\bm\psi^{(j)})^{\top},\label{eq-diff}
\end{align} 
we have 
\begin{align*}
R_{1}&=\frac{\omega_f}{\omega_r}|\bm{e}_l^{\top}(\bm\psi^{(k)}-\bm\psi^{(j)})\|\btheta_f\|_2^2+\bm{e}_l^{\top}\btheta_f(\bm\psi^{(k)}-\bm\psi^{(j)})^{\top}\btheta_f|\\
&\geq\frac{c_r\omega_f|\|\btheta_f\|_2^2+(\bm{e}_l^{\top}\btheta_f)^2|}{\omega_r\sqrt{\tn_r}\delta}\geq \frac{c_r\omega_f^2\min\{\delta^2,\omega_f^{-2}\}}{\sqrt{\tn_r}\omega_f\delta}\geq \frac{\min\{\omega_f\delta,1\}}{\sqrt{\tn_r}} \\
R_{2}&\leq \frac{\omega_f}{\omega_r}\|\btheta_f\|_2\|\Sig_r^{(k)}-\Sig_r^{(j)}\|_2\|I_p-\Sig_r^{(k)}\|_2\\
&\leq \omega_f\delta \|\bm\psi^{(k)}-\bm\psi^{(j)}\|_2\|\btheta_f\|_2\|\bm\psi^{(k)}\|_2\|\btheta_f\|_2\\
&\leq \omega_f\delta\|\btheta_f\|_2^2|\bm{e}_l^{\top}(\bm\psi^{(k)}-\bm\psi^{(j)})|\|\bm\psi^{(k)}\|_2\\
&\leq \frac{\omega_f^3\min\{\delta^3,\omega_f^{-3}\}\sqrt{p}}{\tn_r\omega_f^2\delta^2}=\frac{\min\{\omega_f,1\}\sqrt{p}}{\tn_r}=o(1)\frac{1}{\sqrt{\tn_r}}\\
R_{3}&\leq \frac{\omega_f}{\omega_r}\|\btheta_f\|_2\|\Sig_r^{(k)}-\Sig_r^{(j)}\|_2\|I_p-\Sig_r^{(j)}\|_2=o(1)\frac{1}{\sqrt{\tn_r}}
\end{align*}
where we use the facts that $p=o(\tn_r)$ and $\omega_r\geq c_0>0$.

Hence, $\min_{l\leq p}\alpha_l\geq \frac{c_r\min\{\omega_f\delta,1\}}{\sqrt{\tn_r}} (1-o(1))$ and
\begin{align*}
   \sum_{j=1}^pg(\alpha_j)&= \sum_{j=1}^p\alpha_j^2\geq  \sum_{j=1}^p(R_1-R_2-R_3)^2\geq\frac{c_r^2p\min\{(\omega_f\delta)^2,1\}}{\tn_r}.
  \end{align*}

\end{proof}

\begin{lemma}
\label{lem-KL}
Let $P^{(j)}= N(\bm\mu^{(j)},\Omega^{(j)})$ where $\Sig^{(j)}$ are positive definite. Given that $\|(\Omega^{(k)})^{-1}(\Omega^{(j)}-\Omega^{(k)})\|_2\leq 1/2$, then
\[
  \textup{TV}(P^{(j)},P^{(k)})\leq \frac{1}{2}(\bm\mu^{(j)}-\bm{\mu}^{(k)})^{\top}(\Omega^{(k)})^{-1}(\bm\mu^{(j)}-\bm{\mu}^{(k)})+\frac{1}{2}\|(\Omega^{(k)})^{-1}(\Omega^{(j)}-\Omega^{(k)})\|_F^2.
\]
\end{lemma}
\begin{proof}[Proof of Lemma \ref{lem-KL}]
\begin{align*}
\textup{TV}^2(P^{(j)},P^{(k)})&\leq \textup{KL}(P^{(j)},P^{(k)})\\
&=\frac{1}{2}(\bm\mu^{(j)}-\bm{\mu}^{(k)})^{\top}(\Omega^{(k)})^{-1}(\bm\mu^{(j)}-\bm{\mu}^{(k)})+\frac{1}{2}\textup{Tr}((\Omega^{(k)})^{-1}(\Omega^{(j)}-\Omega^{(k)}))-\frac{1}{2}\log det((\Omega^{(k)})^{-1}\Omega^{(j)}).
\end{align*}
Let $\lambda_1\geq \dots\geq \lambda_p$ denote the singular values of $(\Omega^{(k)})^{-1}( \Omega^{(k)}-\Omega^{(j)})$. Given that $\lambda_1\leq 1/2$, we have
\begin{align}
&\textup{Tr}((\Omega^{(k)})^{-1}(\Omega^{(j)}-\Omega^{(k)}))-\log det((\Omega^{(k)})^{-1}\Omega^{(j)})=\sum_{l=1}^p\{-\lambda_l-\log(1-\lambda_l)\}\nonumber\\
&\quad\leq \sum_{l=1}^p \sum_{k=2}^{\infty}\lambda_l^k/k\leq\sum_{l=1}^p\lambda_l^2=\|(\Omega^{(k)})^{-1}(\Omega^{(j)}-\Omega^{(k)})\|_F^2.\label{tv2}
\end{align}
\end{proof}

\begin{lemma}
\label{lem-T1}
Assume the conditions of Theorem \ref{thm-minimax2}. Under the distribution configuration in (\ref{theta-def}), it holds that
\[
   \max_{\bbeta^{(j)}\sim \bbeta^{(k)}}\E_{\bbeta^{(j)}}[\textup{TV}(p_{\bbeta^{(j)}}(\mathring{\btheta}_p|\calD_f,\widetilde{\calD}_r,\check{\bz}_r),p_{\bbeta^{(k)}}(\mathring{\btheta}_p|\calD_f,\widetilde{\calD}_r,\check{\bz}_r))]\leq Cc_r\sqrt{\frac{p}{\tn_r}}.
\]
\end{lemma}
\begin{proof}[Proof of Lemma \ref{lem-T1}]
By the definition of $\bz_r$ in (\ref{eq-z}), conditioning on $\{\calD_f,\widetilde{\calD}_r,\check{\bz}_r\}$ is equivalent to conditioning on $\{\calD_f,\widetilde{\calD}_r,\bz_r\}$.
Hence,
\begin{align*}
   &\mathring{\btheta}_p|\calD_f,\widetilde{\calD}_r,\check{\bz}_r\sim_{\bbeta^{(j)}} \mathring{\btheta}_p|\calD_f,\widetilde{\calD}_r,\bz_r\sim_{\bbeta^{(j)}} N(\mathring{\bm\mu}^{(j)},\mathring{\Sig}^{(j)}),
\end{align*}
where by (\ref{eq-thetap-r}),
\begin{align*}
 \mathring{\bm\mu}^{(j)}  &=\frac{1}{N}(\Sig_{p}^{(j)})^{-1}\{X_f^{\top}\bz_f+\tilde{X}_r^{\top}\tilde{\bz}_r+\frac{\|\check{\bz}_r\|_2^2}{\E[z_r^2]}\Sig_r^{(j)}(\btheta^{(j)}_r-\btheta^{(j)}_{p})\}\\
 &=\frac{1}{N}(\Sig_{p}^{(j)})^{-1}\{X_f^{\top}\bz_f+\tilde{X}_r^{\top}\tilde{\bz}_r-\E[X_f^{\top}\bz_f+\tilde{X}_r^{\top}\tilde{\bz}_r]+\frac{(\|\check{\bz}_r\|_2^2-\check{n}_r\E[z_r^2])}{\E[z_r^2]}\Sig_r^{(j)}(\btheta^{(j)}_r-\btheta^{(j)}_{p})\}\\
\mathring{\Sig}^{(j)}&= \frac{\|\check{\bz}_r\|_2^2}{N^2}(\Sig_{p}^{(j)})^{-1}\left(\Sig_r^{(j)}-\frac{\Sig_r^{(j)}(\btheta^{(j)}_r-\btheta^{(j)}_{p})(\btheta^{(j)}_r-\btheta^{(j)}_{p})^{\top}\Sig_r^{(j)}}{\E[z_r^2]}\right) (\Sig_{p}^{(j)})^{-1}.
\end{align*}
In the above equations, we use the fact that the distribution of $\calD_f$ and $\bz_r$ are unchanged under each hypothesis. 

We first show that $\mathring{\Sig}^{(j)}$ is positive definite. First, 
\[
\Lambda_{\min}((\Sig_r^{(j)})^{-1})\geq \|\Sig_r^{(j)}\|_2^{-1}\geq \frac{1}{1+\|\btheta_f\|_2\|\bpsi^{(j)}\|_2}\geq \frac{1}{1+\frac{\min\{\omega_f\delta,1\}\sqrt{p}}{\omega_f\delta\sqrt{\tn_r}}}\geq \frac{1}{2}.
\]
As $\omega_r>1/2$ and $\E[z_r^2]=5$, by (\ref{facts}),
\begin{align*}
  \min_{1\leq j\leq M}\Lambda_{\min}\left(\Sig_r^{(j)}-\frac{\Sig_r^{(j)}(\btheta^{(j)}_r-\btheta^{(j)}_{p})(\btheta^{(j)}_r-\btheta^{(j)}_{p})^{\top}\Sig_r^{(j)}}{\E[z_r^2]}\right)&\geq \min_{1\leq j\leq M}\Lambda_{\min}(\Sig_r^{(j)})-\frac{\omega_f^2\|\btheta_f\|_2^2}{\omega_r^2\E[z_r^2]}\\
  &\geq 1-\frac{4\min\{\omega_f^2\delta^2,1\}}{5}\geq 1/5.
\end{align*}

Hence, the distribution  $N(\mathring{\bmu}^{(j)},\mathring{\Sig}^{(j)})$ is well-defined. We move on to bound \newline $\textup{KL}(N(\mathring{\bm\mu}^{(j)},\mathring{\Xi}^{(j)}),N(\mathring{\bm\mu}^{(k)},\mathring{\Xi}^{(k)})))$.

By Lemma \ref{lem-KL}, if $\|(\mathring{\Sig}^{(k)})^{-1}(\mathring{\Sig}^{(j)}-\mathring{\Sig}^{(k)})\|_2\leq 1/2$,
\begin{align*}
&\textup{TV}^2(p_{\bbeta^{(j)}}(\mathring{\btheta}_p|\calD_f,\widetilde{\calD}_r,\check{\bz}_r),p_{\bbeta^{(k)}}(\mathring{\btheta}_p|\calD_f,\widetilde{\calD}_r,\check{\bz}_r))
\leq \textup{KL}(N(\mathring{\bm\mu}^{(j)},\mathring{\Sig}^{(j)}),N(\mathring{\bm\mu}^{(k)},\mathring{\Sig}^{(k)}))\\
&\quad\quad\leq \frac{1}{2}(\mathring{\bmu}^{(j)}-\mathring{\bmu}^{(k)})^{\top}(\mathring{\Sig}^{(k)})^{-1}(\mathring{\bmu}^{(j)}-\mathring{\bmu}^{(k)})+\frac{1}{2}\|(\mathring{\Sig}^{(k)})^{-1}(\mathring{\Sig}^{(j)}-\mathring{\Sig}^{(k)})\|_F^2.
\end{align*}

As $\Sig_r^{(j)}\btheta_r^{(j)}=\Sig_r^{(k)}\btheta_r^{(k)}=-\frac{\omega_f}{\omega_r}\btheta_f$ and $\btheta^{(j)}_{p}=0$, we have
\begin{align*}
 \mathring{\bm\mu}^{(j)} - \mathring{\bm\mu}^{(k)}&=(\Sig_p^{(k)})^{-1}(\Sig_p^{(k)}-\Sig_{p}^{(j)})\mathring{\bm\mu}^{(j)}\\
 \mathring{\Sig}^{(j)}-\mathring{\Sig}^{(k)}&=\frac{\|\check{\bz}_r\|_2^2}{N^2}(\Sig_{p}^{(j)})^{-1}\left(\Sig_r^{(j)}-\frac{\omega_f^2\btheta_f\btheta_f^{\top}}{\omega_r^2\E[z_r^2]}\right) (\Sig_{p}^{(j)})^{-1}-\frac{\|\check{\bz}_r\|_2^2}{N^2}(\Sig_p^{(k)})^{-1}\left(\Sig_r^{(k)}-\frac{\omega_f^2\btheta_f\btheta_f^{\top}}{\omega_r^2\E[z_r^2]}\right) (\Sig_p^{(k)})^{-1}.
\end{align*}

We first check that  $\|(\mathring{\Sig}^{(k)})^{-1}(\mathring{\Sig}^{(j)}-\mathring{\Sig}^{(k)})\|_2\leq 1/2$. 
Note that
\begin{align*}
&\{\mathring{\Sig}^{(k)}\}^{-1}(\mathring{\Sig}^{(j)}-\mathring{\Sig}^{(k)})= \Sig_p^{(k)}\left(\Sig_r^{(k)}-\frac{\omega_f^2\btheta_f\btheta_f^{\top}}{\omega_r^2\E[z_r^2]}\right)^{-1}\Sig_p^{(k)}(\Sig_{p}^{(j)})^{-1}\left(\Sig_r^{(j)}-\frac{\omega_f^2\btheta_f\btheta_f^{\top}}{\omega_r^2\E[z_r^2]}\right) (\Sig_{p}^{(j)})^{-1}-I_p\\
&\quad=\Sig_p^{(k)}\left(\Sig_r^{(k)}-\frac{\omega_f^2\btheta_f\btheta_f^{\top}}{\omega_r^2\E[z_r^2]}\right)^{-1}\left(\Sig_r^{(j)}-\frac{\omega_f^2\btheta_f\btheta_f^{\top}}{\omega_r^2\E[z_r^2]}\right) (\Sig_{p}^{(j)})^{-1}-I_p\\
&\quad\quad+\underbrace{\Sig_p^{(k)}\left(\Sig_r^{(k)}-\frac{\omega_f^2\btheta_f\btheta_f^{\top}}{\omega_r^2\E[z_r^2]}\right)^{-1}\{(\Sig_p^{(k)}(\Sig_{p}^{(j)})^{-1}-I_p\}\left(\Sig_r^{(j)}-\frac{\omega_f^2\btheta_f\btheta_f^{\top}}{\omega_r^2\E[z_r^2]}\right) (\Sig_{p}^{(j)})^{-1}}_{A_1^{(j,k)}}\\
&\quad=A_1^{(j,k)}+\Sig_p^{(k)}(\Sig_{p}^{(j)})^{-1}-I_p+\underbrace{\Sig_p^{(k)}\{\left(\Sig_r^{(k)}-\frac{\omega_f^2\btheta_f\btheta_f^{\top}}{\omega_r^2\E[z_r^2]}\right)^{-1}\left(\Sig_r^{(j)}-\frac{\omega_f^2\btheta_f\btheta_f^{\top}}{\omega_r^2\E[z_r^2]}\right)-I_p\} (\Sig_{p}^{(j)})^{-1}}_{A_2^{(j,k)}}.
\end{align*}

As $\Sig_{p}^{(j)}$ and $\Sig_p^{(k)}$ have finite positive eigenvalues, it is easy to show that
\begin{align*}
 & \max_{\bbeta^{(j)}\sim \bbeta^{(k)}}\|(\mathring{\Sig}^{(k)})^{-1}(\mathring{\Sig}^{(j)}-\mathring{\Sig}^{(k)})\|_2\leq C\|\Sig_{p}^{(j)}-\Sig_p^{(k)}\|_2\leq C\omega_r\|\Sig_r^{(j)}-\Sig_r^{(k)}\|_2\\
  &\quad\leq 2C\omega_r\|\bpsi^{(k)}-\bpsi^{(j)}\|_2\|\btheta_f\|_2\leq \frac{c_r\min\{\delta,1/\omega_f\}}{\sqrt{\tn_r}\delta}\leq \frac{c_r}{\sqrt{\tn_r}}.
\end{align*}

Similarly,
\begin{align*}
& \|\mathring{\bm\mu}^{(j)} - \mathring{\bm\mu}^{(k)}\|_2\leq C\omega_r\|\mathring{\bm\mu}^{(j)}\|_2\|\Sig_r^{(j)}-\Sig_r^{(k)}\|_2\leq \frac{Cc_r\|\mathring{\bm\mu}^{(j)}\|_2}{\sqrt{\tn_r}}\\
 & \|(\mathring{\Sig}^{(k)})^{-1/2}(\mathring{\bm\mu}^{(j)} - \mathring{\bm\mu}^{(k)})\|_2^2\leq Cc_r^2\frac{N^2\|\mathring{\bm\mu}^{(j)}\|_2^2}{\|\check{\bz}_r\|_2^2\tn_r}.
\end{align*}
Hence,
\begin{align*}
 \textup{TV}^2(N(\mathring{\bm\mu}^{(j)},\mathring{\Sig}^{(j)}),N(\mathring{\bm\mu}^{(k)},\mathring{\Sig}^{(k)}))\leq \frac{Cc_r^2\|\mathring{\bm\mu}^{(j)}\|^2_2N^2}{\|\check{\bz}_r\|_2^2\tn_r}+\frac{Cc_r^2p}{\tn_r}.
\end{align*}
Note that
\[
 \|\mathring{\bm\mu}_j\|_2\leq  \frac{1}{N}\|X_f^{\top}\bz_f+\tilde{X}_r^{\top}\tilde{\bz}_r\|_2+\frac{\omega_f\|\btheta_f\|_2}{\omega_rN}\frac{|\|\check{\bz}_r\|_2^2-\check{N}_r\E[z_r^2]|}{\E[z_r^2]}.
\]
Hence,
\begin{align*}
\E_j[\frac{\|\mathring{\bm\mu}^{(j)}\|_2^2}{\|\check{\bz}_r\|_2^2}]&\leq \frac{2}{N^2} \E[\|X_f^{\top}\bz_f+\tilde{X}_r^{\top}\tilde{\bz}_r-\E[X_f^{\top}\bz_f+\tilde{X}_r^{\top}\tilde{\bz}_r]\|_2^2]\E[\frac{1}{\|\check{\bz}_r\|_2^2}]+ \frac{2\min\{(\omega_f\delta)^2,1\}}{5N^2}\E[\frac{(\|\check{\bz}_r\|_2^2-\check{N}_r\E[z_r^2])^2}{\|\check{\bz}_r\|_2^2}]\\
&\leq\frac{C(\tn_r+N_f)p}{N^2(\check{N}_r-2)}+ \frac{C\min\{(\omega_f\delta)^2,1\}}{N^2}(\check{N}_r-2-2\check{N}_r+\frac{\check{N}_r^2}{\check{N}_r-2})\\
&\leq C\frac{(\tn_r+N_f)p}{N^2\check{N}_r}+\frac{\min\{(\omega_f\delta)^2,1\}}{N^2},
\end{align*}
where we use the fact that $\|\check{\bz}_r\|_2^2/5$ follows chi-squared distribution with parameter $\check{N}_r$.
Therefore,
\begin{align*}
 \E_j[\textup{TV}^2(N(\mathring{\bm\mu}^{(j)},\mathring{\Xi}^{(j)}),N(\mathring{\bm\mu}^{(k)},\mathring{\Xi}^{(k)}))]&\leq \frac{Cc_r^2((\tn_r+N_f)p/\check{N}_r+\min\{(\omega_f\delta)^2,1\})}{\tn_r}+\frac{Cc_r^2p}{\tn_r}\\
 &\leq \frac{Cc_r^2(p+\min\{(\omega_f\delta)^2,1\})}{\tn_r}=\frac{Cc_r^2p}{\tn_r},
\end{align*}
where we use the fact that $\check{N}_r=N_r-\tn_r\geq cN$.
\end{proof}

\begin{lemma}
\label{lem-T2}
Assume the conditions of Theorem \ref{thm-minimax2}. Under the distribution configuration in (\ref{theta-def}), it holds that
\[
  \max_{\bbeta^{(j)}\in B_p(\delta)}\textup{TV}(p_{\bbeta^{(j)}}(\hat{\btheta}_p,\calD_f,\widetilde{\calD}_r),p_{\bbeta^{(j)}}(\mathring{\btheta}_p,\calD_f,\widetilde{\calD}_r))\leq C'\sqrt{\frac{p(p+\log\tn_r)}{N}}+\frac{1}{\tn_r}.
\]
\end{lemma}
\begin{proof}[Proof of Lemma \ref{lem-T2}]
Note that
\begin{align*}
&\textup{TV}(p_{\bbeta^{(j)}}(\hat{\btheta}_p,\calD_f,\widetilde{\calD}_r),p_{\bbeta^{(j)}}(\mathring{\btheta}_p,\calD_f,\widetilde{\calD}_r))\\
&\leq \E_{\bbeta^{(j)}}[\textup{TV}(p_{\bbeta^{(j)}}(\hat{\btheta}_p|\calD_f,\widetilde{\calD}_r,\check{X}_r),p_{\bbeta^{(j)}}(\mathring{\btheta}_p|\calD_f,\widetilde{\calD}_r,\check{X}_r))].
\end{align*}
As shown above,
\begin{align*}
&\mathring{\btheta}_p|\calD_f,\widetilde{\calD}_r,\check{X}_r\sim_j \\
&\quad N\left(\underbrace{\btheta_p+\frac{1}{N}(\Sig_{p}^{(j)})^{-1}(X_f^{\top}\bz_f+\tilde{X}_r^{\top}\tilde{\bz}_r+\check{N}_r\check{\Sig}_r(\btheta_r^{(j)}-\btheta_p))}_{\mathring{\bm\alpha}^{(j)}},\underbrace{\frac{(\sig_r^{(j)})^2\check{N}_r}{N^2}(\Sig_{p}^{(j)})^{-1}\check{\Sig}_r(\Sig_{p}^{(j)})^{-1}}_{\mathring{\Omega}^{(j)}}\right)\\
&\hat{\btheta}_p|\calD_f,\widetilde{\calD}_r,\check{X}_r\sim_j \\
&\quad N\left(\underbrace{\btheta_p+\frac{1}{N}\widehat{\Sig}_p^{-1}(X_f^{\top}\bz_f+\tilde{X}_r^{\top}\tilde{\bz}_r+\check{N}_r\check{\Sig}_r(\btheta_r^{(j)}-\btheta_p))}_{\hat{\bm\alpha}^{(j)}},\underbrace{\frac{(\sig_r^{(j)})^2\check{N}_r}{N^2}\widehat{\Sig}_p^{-1}\check{\Sig}_r\widehat{\Sig}_p^{-1}}_{\widehat{\Omega}^{(j)}}\right).
\end{align*}

We apply Lemma \ref{lem-KL} again. Specifically,
\begin{align*}
\|\{\widehat{\Omega}^{(j)}\}^{-1}(\mathring{\Omega}^{(j)}-\widehat{\Omega}^{(j)})\|_2&\leq \|I_p-\widehat{\Sig}_p\check{\Sig}_r^{-1}\widehat{\Sig}_p(\Sig_{p}^{(j)})^{-1}\check{\Sig}_r(\Sig_{p}^{(j)})^{-1}\|_2\\
&\leq \|\widehat{\Sig}_p\check{\Sig}_r^{-1}(\widehat{\Sig}_p(\Sig_{p}^{(j)})^{-1}-I_p)\check{\Sig}_r\Sig_{p}^{(j)})^{-1}\|_2+\|I_p-\widehat{\Sig}_p(\Sig_{p}^{(j)})^{-1}\|_2.
\end{align*}
Let 
\begin{align}
  \mathcal{E}_{j}&=\left\{\|\widehat{\Sig}_p-\Sig_{p}^{(j)}|_2\leq C\sqrt{\frac{p+\log \tn_r}{N}},\Lambda_{\max}(\check{\Sig}_r)\leq c_1,\Lambda_{\min}(\check{\Sig}_r)\geq c_2\right\}.\label{event-Ej}
\end{align}
In event $\mathcal{E}_{j}$, $\Lambda_{\min}(\widehat{\Omega}^{(j)})\geq c/N$, $\|\{\widehat{\Omega}^{(j)}\}^{-1}(\mathring{\Omega}^{(j)}-\Omega^{(j)})\|_2\leq C\sqrt{(p+\log \tn_r)/N}\leq 1/2$ and
\begin{align*}
\textup{KL}(N(\mathring{\bm\alpha}^{(j)},\mathring{\Omega}^{(j)}),N(\hat{\bm\alpha}^{(j)},\hat{\Omega}^{(j)}))
&\leq N\|\mathring{\bm\alpha}^{(j)}-\hat{\bm\alpha}^{(j)}\|_2^2+\frac{p(p+\log \tn_r)}{N}.
\end{align*}
Moreover, in event $\mathcal{E}_{j}$,
\begin{align*}
\|\mathring{\bm\alpha}^{(j)}-\hat{\bm\alpha}^{(j)}\|_2&\leq \|\widehat{\Sig}_p^{-1}(\widehat{\Sig}_p-\Sig_{p}^{(j)})\|_2\|\mathring{\bm\alpha}^{(j)}\|_2\\
&\leq \frac{C}{N}\|\widehat{\Sig}_p-\Sig_{p}^{(j)}\|_2\|X_f^{\top}\bz_f+\tilde{X}_r^{\top}\tilde{\bz}_r+\check{N}_r\check{\Sig}_r(\btheta^{(j)}_r-\btheta_p)\|_2.
\end{align*}

We arrive at in event $\mathcal{E}_{j}$,
\begin{align*}
&\textup{KL}(N(\mathring{\bm\alpha}^{(j)},\mathring{\Omega}^{(j)}),N(\hat{\bm\alpha}^{(j)},\hat{\Omega}^{(j)}))\\
&\quad\leq \frac{C|\widehat{\Sig}_p-\Sig_{p}^{(j)}\|_2^2}{N\Lambda_{\min}(\mathring{\Omega}^{(j)})}\|X_f^{\top}\bz_f+\tilde{X}_r^{\top}\tilde{\bz}_r+\check{N}_r\check{\Sig}_r(\btheta_r^{(j)}-\btheta_p)\|_2^2\\
&\quad\quad+\frac{p(p+\log \tn_r)}{N}.
\end{align*}

Therefore,
\begin{align*}
\E_j[\textup{TV}(N(\mathring{\bm\alpha}^{(j)},\mathring{\Omega}^{(j)}),N(\hat{\bm\alpha}^{(j)},\hat{\Omega}^{(j)}))]&\leq \underbrace{\E_j[\textup{KL}^{1/2}(N(\mathring{\bm\alpha}^{(j)},\mathring{\Omega}^{(j)}),N(\hat{\bm\alpha}^{(j)},\hat{\Omega}^{(j)}))\mathbb{1}(\mathcal{E}_{j})]}_{T_{1,j}}\\
&\quad+\underbrace{\E_j[\textup{TV}(N(\mathring{\bm\alpha}^{(j)},\mathring{\Omega}^{(j)}),N(\hat{\bm\alpha}^{(j)},\hat{\Omega}^{(j)}))\mathbb{1}(\mathcal{E}_{j}^c)]}_{T_{2,j}}.
\end{align*}
By Lemma \ref{lem-mini2}, 
\[
   \max_{j\leq M}T_{2,j}\leq \max_{j\leq M}\P(\mathcal{E}_{j}^c)\leq \exp\{-c_2\log\tn_r\}.
\]
For $T_{1,j}$, using the Gaussian property of $\bx_f,\bx_r,\bm{z}_r,\bm{z}_f$, we have
\begin{align*}
T_{1,j}&\leq \E_{j}[\frac{C}{N\Lambda_{\min}(\mathring{\Omega}^{(j)})}\|\widehat{\Sig}_p-\Sig_{p}^{(j)}\|_2^2\|X_f^{\top}\bz_f+\tilde{X}_r^{\top}\tilde{\bz}_r+\check{N}_r\check{\Sig}_r(\btheta_r^{(j)}-\btheta_p)\|_2^2\mathbb{1}(\mathcal{E}_{j})]+\frac{p(p+\log \tn_r)}{N}\\
&\leq \frac{C(p+\log\tn_r)p}{N^2}\E_{j}[\|X_f^{\top}\bz_f+\tilde{X}_r^{\top}\tilde{\bz}_r+\check{N}_r\check{\Sig}_r(\btheta_r^{(j)}-\btheta_p)\|_2^2]+\frac{p(p+\log\tn_r)}{N}\\
&\leq \frac{C(p+\log\tn_r)p}{N}+\frac{p(p+\log\tn_r)}{N}\\
&\leq \frac{C'p(p+\log\tn_r)}{N},
\end{align*}
where the second last line is due to $\E[X_f^{\top}\bz_f+\tilde{X}_r^{\top}\tilde{\bz}_r+\check{N}_r\check{\Sig}_r(\btheta_r^{(j)}-\btheta_p)]=0$.

To summarize,
\begin{align*}
& \max_{\bbeta^{(j)}\in B}\textup{TV}(p_{\bbeta^{(j)}}(\hat{\btheta}_p,\calD_f,\widetilde{\calD}_r),p_{\bbeta^{(j)}}(\mathring{\btheta}_p,\calD_f,\widetilde{\calD}_r))\leq \frac{C'\sqrt{p(p+\log\tn_r)}}{\sqrt{N}}+\exp\{-c_2\log\tn_r\}\\
 &\quad\leq C'\sqrt{\frac{p(p+\log\tn_r)}{N}}+\frac{1}{\tn_r}.
\end{align*}
\end{proof}

\begin{lemma}
\label{lem-mini1}
Assume the conditions of Theorem \ref{thm-minimax2}. Under the distribution configuration in (\ref{theta-def}), it holds that
\[
   \max_{\bbeta^{(j)},\bbeta^{(k)}\in B_p(\delta):\bbeta^{(j)}\sim \bbeta^{(k)}}\textup{TV}(p_{\bbeta^{(j)}}(\widetilde{\calD}_r),p_{\bbeta^{(k)}}(\widetilde{\calD}_r))\leq Cc_r.
\]
\end{lemma}
\begin{proof}[Proof of Lemma \ref{lem-mini1}]
Note that 
\[
   ((\bx^{(r)}_i)^{\top},y^{(r)}_i)\sim_{\bbeta^{(j)}} N(0,\Omega^{(j)})~\text{where}~\Omega^{(j)}=\begin{pmatrix}
   (\btheta_r^{(j)})^{\top}\Sig_r^{(j)}\btheta_r^{(j)}+(\sig_r^{(j)})^2 & (\Sig_r^{(j)}\btheta_r^{(j)})^{\top}\\
   \Sig_r^{(j)}\btheta_r^{(j)} & \Sig_r^{(j)}\end{pmatrix}.
\]
Invoking (\ref{tv2}),   given that $\|(\Omega^{(k)})^{-1}(\Omega^{(j)}-\Omega^{(k)})\|_2\leq 1/2$,
 \begin{align}
 \label{eq-tv0}
&\textup{TV}^2(p_{\bbeta^{(j)}}(\widetilde{\calD}_r),p_{\bbeta^{(k)}}(\widetilde{\calD}_r)) \leq\textup{KL}(p_{\bbeta^{(j)}}(\widetilde{\calD}_r),p_{\bbeta^{(k)}}(\widetilde{\calD}_r))\leq \frac{\tilde{n}_r}{2}\|(\Omega^{(k)})^{-1}(\Omega^{(j)}-\Omega^{(k)})\|_F^2.
 \end{align}
By (\ref{theta-def}), we first calculate that
\begin{align*}
&\Omega^{(k)}-\Omega^{(j)}=\begin{pmatrix}
0& (\Sig_{r}^{(k)}\btheta_{r}^{(k)}-\Sig_{r}^{(j)}\btheta_{r}^{(j)})^{\top}\\
\Sig_{r}^{(k)}\btheta_{r}^{(k)}-\Sig_{r}^{(j)}\btheta_{r}^{(j)} & \Sig_{r}^{(k)}-\Sig_{r}^{(j)}
\end{pmatrix} .
\end{align*}
By (\ref{facts}), 
\begin{align*}
\Omega^{(k)}-\Omega^{(j)}=\begin{pmatrix}
0& 0\\
0 & \Sig_{r}^{(k)}-\Sig_{r}^{(j)}
\end{pmatrix} .
\end{align*}
We know that
\begin{align}
\|(\Omega^{(k)})^{-1}(\Omega^{(k)}-\Omega^{(j)})\|_F^2&\leq \Lambda^{-2}_{\min}(\Omega^{(k)})\|\Omega^{(k)}-\Omega^{(j)}\|_F^2\leq\Lambda^{-2}_{\min}(\Omega^{(k)})\|\Sig_{r}^{(k)}-\Sig_{r}^{(j)}\|_F^2.\label{eq-tv1}
\end{align}
We know that $\bm\phi^{(k)}$ and $\bm\phi^{(j)}$ only differ at one coordinate. Suppose they differ at the $l$-th coordinate.  
\begin{align*}
\|\Sig_{r}^{(k)}-\Sig_{r}^{(j)}\|_F^2&\leq \|(\bm\psi^{(j)}-\bm\psi^{(k)})\btheta_f^{\top}+\btheta_f(\bm\psi^{(j)}-\bm\psi^{(k)})^{\top}\|_F^2\\
&\leq 2\|\btheta_f\|_2^2\|\bm\psi^{(j)}-\bm\psi^{(k)}\|_2^2+2(\btheta_f^{\top}(\bm\psi^{(j)}-\bm\psi^{(k)}))^2\leq \frac{4c_r^2\min\{\delta^2,\omega_f^{-2}\}}{\tn_r\delta^2}\leq \frac{Cc_r^2}{\tn_r}.
\end{align*} 
On the other hand,
\[
   \Omega^{(k)}\succeq \begin{pmatrix} 5 & \frac{\omega_f}{\omega_r}\btheta_f\\
          \frac{\omega_f}{\omega_r}\btheta_f& I_p\end{pmatrix}.
\]
To show that $\Lambda_{\min}(\Omega^{(k)})\geq c_1$ for some positive constant $c_1$, it suffices to show that
\[
   \frac{\omega_f^2}{\omega_r^2}\|\btheta_f\|^2_2\leq c_2<1.
\]
By the definition of $\btheta_f$ in (\ref{theta-def}), we have
\[
\frac{\omega_f^2}{\omega_r^2}\|\btheta_f\|^2_2\leq \min\{\omega_f^2\delta^2,c_0^2\}<1.
\]

To summarize, by (\ref{eq-tv0}) and (\ref{eq-tv1}), we have
\begin{align*}
 \max_{\bbeta^{(j)},\bbeta^{(k)}\in B_p(\delta):\bbeta^{(j)}\sim \bbeta^{(k)}}\textup{TV}^2(p_{\bbeta^{(j)}}(\widetilde{\calD}_r),p_{\bbeta^{(k)}}(\widetilde{\calD}_r)) &\leq Cc_r^2.
\end{align*}

\end{proof}

\begin{lemma}
\label{lem-mini2}
Assume the conditions of Theorem \ref{thm-minimax2}. Under the distribution configuration in (\ref{theta-def}), by taking $C$ to be a large enough constant in $\mathcal{E}_j$ defined in (\ref{event-Ej}), it holds that
\[
   \min_{1\leq j\leq M}\P_{\bbeta^{(j)}}(\mathcal{E}_j)\geq 1-\exp\{-c_1\log\tn_r\}.
\]
for some positive constant $c_1$.
\end{lemma}
\begin{proof}[Proof of Lemma \ref{lem-mini2}]
As $\bx^{(r)}_i\sim_{j} N(0,\Sig^{(j)})$ and $\bx^{(f)}_i\sim_j N(0,I_p)$, we know that for some large enough $c_0$
\[
  \P_{\bbeta^{(j)}}\left(\|\widehat{\Sig}_{p}^{(j)}-\Sig_{p}^{(j)}\|_2\geq c_0\sqrt{\frac{p+\log \tn_r}{N}}\right)\leq \exp\{-c_1\log \tn_r\}
\]
for some large enough constant $c_0$. Hence,
\[
     \max_{1\leq j\leq M}\P_{\bbeta^{(j)}}\left(\|\widehat{\Sig}_{p}^{(j)}-\Sig_{p}^{(j)}\|_2\geq c_0\sqrt{\frac{p+\log \tn_r}{N}}\right)\leq \exp\{-c_1\log \tn_r\}
\]
The rest of the statements can be proved similarly.
\end{proof}

\section{Additional numerical experiments}
\subsection{Estimation performance}
\label{sec:est_supp}
To complement the numerical analysis in Section~\ref{sec-simu}, this section provides supplementary experiments comparing the proposed unlearning estimator $\hat{\btheta}_r^{(\uls)}$ with its gradient descent variant $\hat{\btheta}_{r,T}^{(\uls)}$ and its robustified counterpart $\hat{\btheta}_r^{(\uls+)}$.

These experiments follow the same setup and configuration as described in the main text. Specifically, for $\hat{\btheta}_{r,T}^{(\uls)}$, we set the number of iterations at $T=500$ iterations with a constant step size $\alpha=0.05$. For $\hat{\btheta}_r^{(\uls+)}$, we employ cross-validation to adaptively tune $\lambda$ by searching over the range $\lambda \in \left[10^{-4}, 10^4\right]$. The estimation errors are summarized as boxplots in Figure~\ref{fig:first_supp} and Figure~\ref{fig:second_supp}. 

\begin{figure}[H]
\centering
\includegraphics[width=0.99\textwidth]{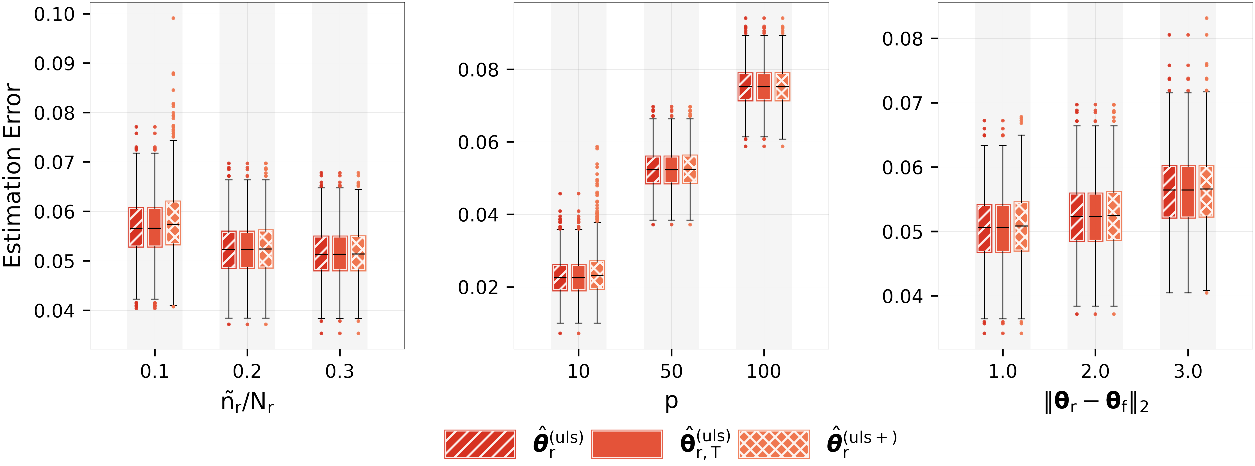}
\caption{Boxplots of the estimation errors for three unlearning estimators: $\hat{\btheta}_r^{(\uls)}$ (yellow with horizontal-line hatch), $\hat{\btheta}_{r,T}^{(\uls)}$ (orange with vertical-line hatch) and $\hat{\btheta}_r^{(\uls+)}$ (dark orange with plus-sign hatch). The three panels correspond to experimental settings (a), (b), and (c), with fixed $N_r =20000$ and $N_f =1000$. Each setting is replicated with 1000 Monte Carlo trials. \label{fig:first_supp}}
\end{figure}

\begin{figure}[H]
\centering
\includegraphics[width=0.99\textwidth]{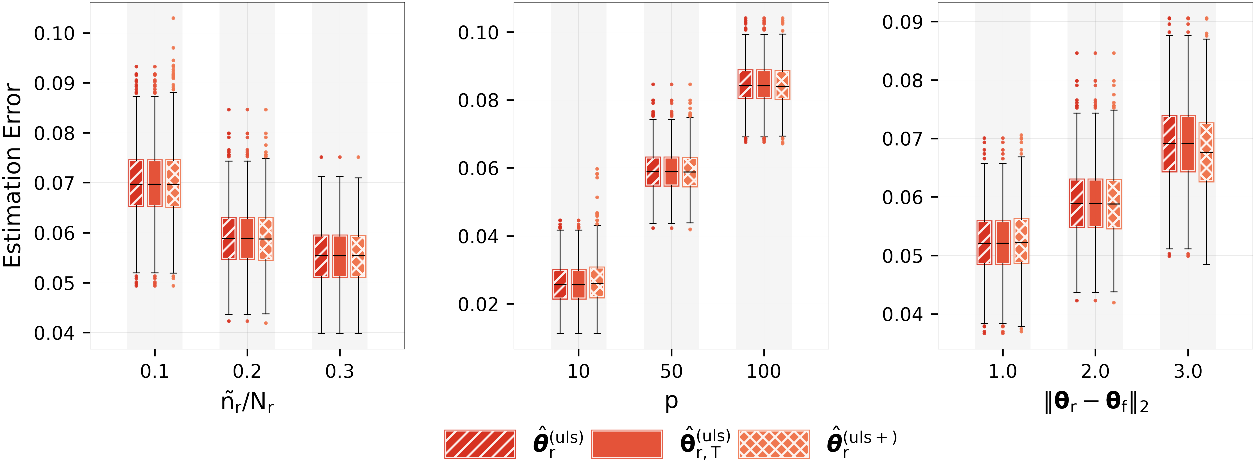}
\caption{Boxplots of the estimation errors for three unlearning estimators: $\hat{\btheta}_r^{(\uls)}$ (red with forward-slash hatch), $\hat{\btheta}_{r,T}^{(\uls)}$ (dark orange solid) and $\hat{\btheta}_r^{(\uls+)}$ (orange with cross-hatch). The three panels correspond to experimental settings (a), (b), and (c), with fixed $N_r =20000$ and $N_f =2000$. Each setting is replicated with 1000 Monte Carlo trials. \label{fig:second_supp}}
\end{figure}

From Figure~\ref{fig:first_supp} and Figure~\ref{fig:second_supp}, we see that $\hat{\btheta}_r^{(uls)}$, $\hat{\btheta}_{r,T}^{(\uls)}$ and $\hat{\btheta}_r^{(\uls+)}$ exhibit consistent trends across varing subsample size ratio $\tilde{n}_r / N_r$, dimension $p$, and parameter discrepancy $\delta$. Regarding the gradient descent variation $\hat{\btheta}_{r,T}^{(\uls)}$, we observe that for sufficiently large $T$, the estimation error of $\hat{\btheta}_{r,T}^{(\uls)}$ becomes virtually indistinguishable from that of $\hat{\btheta}_r^{(\uls)}$. For the robustified version $\hat{\btheta}_r^{(\uls+)}$, the empirical results in Figure~\ref{fig:second_supp} illustrates that $\hat{\btheta}_r^{(\uls+)}$ demonstrates superior performance in scenarios where $\omega_f\delta$ is large. This observation aligns with the convergence rate established in Theorem~\ref{thm-combine}, specifically the term $\min\{\omega_f\delta,1\}$, which characterizes the robustification effect of the retain loss. 

\subsection{Inference results for $N_f=2000$}
\label{sec:inf_supp}
To further investigate the robustness of our inference procedure, we report additional results for a larger $N_f=2000$. Table~\ref{tab:inf_res_supp} summarize the average coverage probabilities and average standard deviations under this setting. These results remain consistent with inference performance observed in Table~\ref{tab:inf_res}.
\begin{table}[!htbp]
\centering
\resizebox{\textwidth}{!}{
\setlength{\tabcolsep}{8pt} 
\begin{tabular}{|c|ccc|ccc|ccc|}
\hline
      & \multicolumn{3}{c|}{$\tn_r/N_r$} 
      & \multicolumn{3}{c|}{$p$} 
      & \multicolumn{3}{c|}{$\delta$} \\

      & 0.1 & 0.2 & 0.3 & 10 & 50 & 100  & 1 & 2 & 3 \\
\hline
\multicolumn{10}{|c|}{\textbf{Average Coverage}} \\
\hline
ULS & 0.954 & 0.950 & 0.951 & 0.948 &  0.950 & 0.942 & 0.957 & 0.950 & 0.943 \\
\hline
OLS & 0.963 & 0.947 & 0.956 & 0.951 & 0.947  & 0.946  & 0.939 &  0.947 & 0.944 \\
\hline
\multicolumn{10}{|c|}{\textbf{Average SD} $(\times 10^{-2})$} \\
\hline
ULS & $0.99 $ & $0.84$ & $0.78$ & $0.83$ & $0.84$ & $0.85$  & $0.75$ & $0.84 $ & $0.97$ \\
\hline
OLS & $2.26$ & $1.59$ & $1.30$ & $1.58$ & $1.59$ & $1.60 $ & $1.59 $ & $1.59$ & $1.59$ \\
\hline
\end{tabular}
}
\caption{Inference results for setting (a), (b) and (c) with $N_r=20000$ and $N_f=2000$. Each setting is replicated with 1000 Monte Carlo trials. 
\label{tab:inf_res_supp}}
\end{table}

\section{Additional details on real data applications}
\subsection{Pre-processing steps for the UK Biobank data}
\label{supp-data}
For predictor construction, we adopt a structured feature engineering pipeline utilizing six representative administrative and clinical covariates, all of which are categorical variables collectively defined by $114$ unique categories.  To handle the varying cardinality of these predictors, we implement a hybrid encoding scheme. The responsible clinician specialty feature, which contains $85$ unique categories, is encoded using target encoding to mitigate the risk of feature explosion. All other features, having fewer than 20 unique categories, are handled using one-hot encoding. 
After removing all-zero columns, feature selection is then conducted using Lasso method~\citep{tibshirani1996regression} with cross-validation on $\calD$, yielding an average final representation of $p=16$ predictors.

\subsection{Additional results for $\tilde{n}_r/N_r \in \{0.2, 0.3\}$}
We report the predictive performance across the Yelp review and the UK Biobank dataset with $\tilde{n}_r/N_r$ set to $0.2$ and $0.3$. As shown, our method is best regardless of the dataset or the value of $\tilde{n}_r/N_r$. These results suggest the robustness of our proposed method. 
\begin{figure}[H]
\centering
\includegraphics[width=0.6\textwidth]{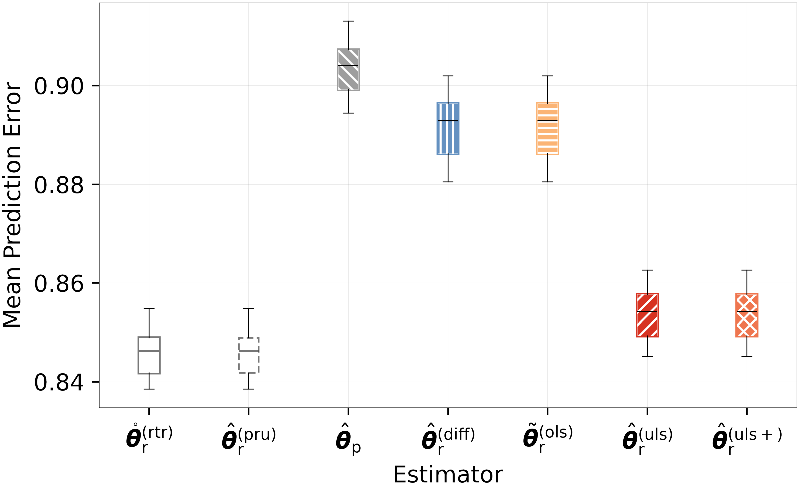}
\caption{Boxplots of mean prediction errors for the Yelp review rating prediction unlearning task. $\hat{\btheta}_{r}^{(\textup{diff})}$, $\tilde{\btheta}_{r}^{(\ols)}$, and $\hat{\btheta}_{r}^{(\uls)}$ are evaluated under the retained subsample proportion $\tilde{n}_r/N_r=0.2$. Each boxplot is based on 20 random splits of training and test samples.}
\label{fig:yelp-1-supp}
\end{figure}

\begin{figure}[H]
\centering
\includegraphics[width=0.6\textwidth]{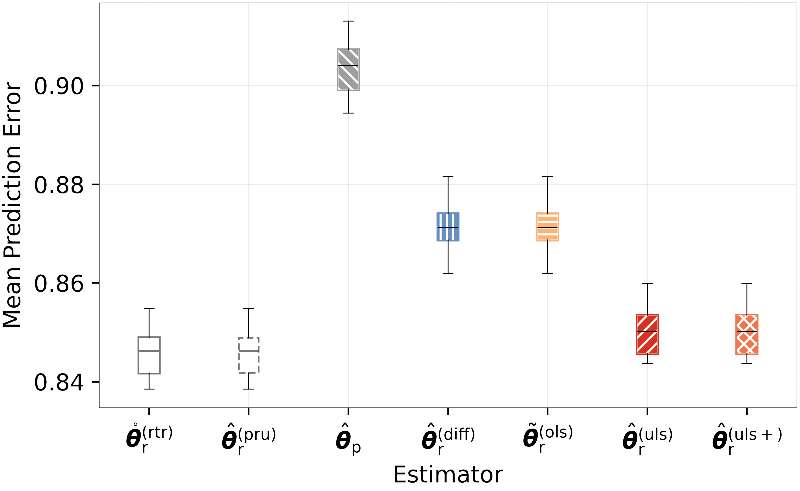}
\caption{Boxplots of mean prediction errors for the Yelp review rating prediction unlearning task. $\hat{\btheta}_{r}^{(\textup{diff})}$, $\tilde{\btheta}_{r}^{(\ols)}$, and $\hat{\btheta}_{r}^{(\uls)}$ are evaluated under the retained subsample proportion $\tilde{n}_r/N_r=0.3$. Each boxplot is based on 20 random splits of training and test samples.}
\label{fig:yelp-2-supp}
\end{figure}

\begin{figure}[H]
\centering
\includegraphics[width=0.6\textwidth]{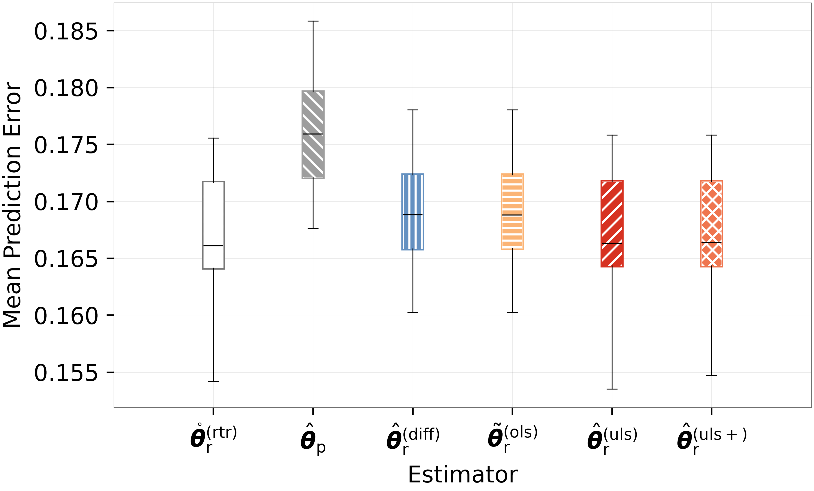}
\caption{Boxplots of mean prediction errors for the UK Biobank hospital episode unlearning task. $\hat{\btheta}_{r}^{(\textup{diff})}$, $\tilde{\btheta}_{r}^{(\ols)}$, and $\hat{\btheta}_{r}^{(\uls)}$ are evaluated under the retained subsample proportion $\tilde{n}_r/N_r=0.2$.  Each boxplot is based on 20 random splits of training and test samples.
\label{fig:ukb-uls-supp1}}
\end{figure}

\begin{figure}[H]
\centering
\includegraphics[width=0.6\textwidth]{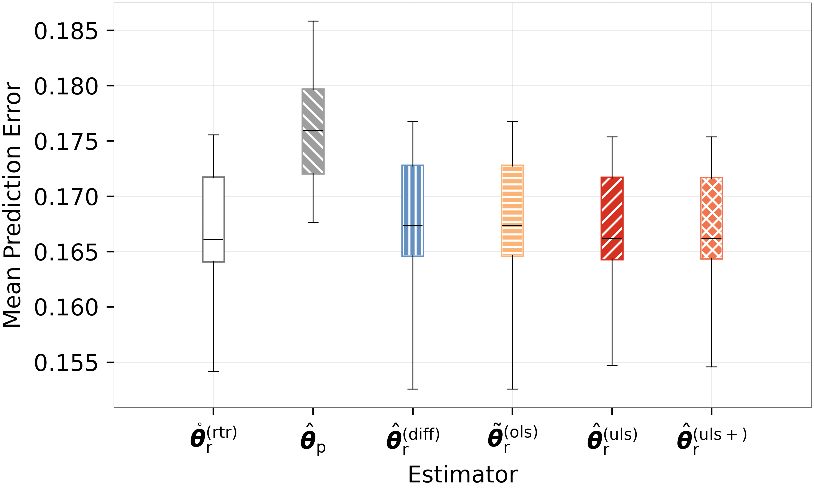}
\caption{Boxplots of mean prediction errors for the UK Biobank hospital episode unlearning task. $\hat{\btheta}_{r}^{(\textup{diff})}$, $\tilde{\btheta}_{r}^{(\ols)}$, and $\hat{\btheta}_{r}^{(\uls)}$ are evaluated under the retained subsample proportion $\tilde{n}_r/N_r=0.3$.  Each boxplot is based on 20 random splits of training and test samples.
\label{fig:ukb-uls-supp2}}
\end{figure}

\bibliographystyle{plainnat}
\bibliography{mu.bib}

\begin{thebibliography}{40}
\providecommand{\natexlab}[1]{#1}
\providecommand{\url}[1]{\texttt{#1}}
\expandafter\ifx\csname urlstyle\endcsname\relax
  \providecommand{\doi}[1]{doi: #1}\else
  \providecommand{\doi}{doi: \begingroup \urlstyle{rm}\Url}\fi

\bibitem[Anjarlekar and Pombra(2025)]{anjarlekar2025llm}
A.~Anjarlekar and S.~Pombra.
\newblock Llm unlearning using gradient ratio-based influence estimation and
  noise injection.
\newblock \emph{arXiv preprint arXiv:2508.06467}, 2025.

\bibitem[Balakrishnan et~al.(2017)Balakrishnan, Wainwright, and
  Yu]{balakrishnan2017statistical}
S.~Balakrishnan, M.~J. Wainwright, and B.~Yu.
\newblock Statistical guarantees for the em algorithm: From population to
  sample-based analysis.
\newblock \emph{The Annals of Statistics}, 45\penalty0 (1):\penalty0 77--120,
  2017.

\bibitem[Brophy and Lowd(2021)]{brophy2021machine}
J.~Brophy and D.~Lowd.
\newblock Machine unlearning for random forests.
\newblock In \emph{International Conference on Machine Learning}, pages
  1092--1104. PMLR, 2021.

\bibitem[Cai and Wei(2021)]{cai2021transfer}
T.~T. Cai and H.~Wei.
\newblock Transfer learning for nonparametric classification.
\newblock \emph{The Annals of Statistics}, 49\penalty0 (1):\penalty0 100--128,
  2021.

\bibitem[Cao and Yang(2015)]{cao2015towards}
Y.~Cao and J.~Yang.
\newblock Towards making systems forget with machine unlearning.
\newblock In \emph{2015 IEEE symposium on security and privacy}, pages
  463--480. IEEE, 2015.

\bibitem[Cook and Weisberg(1982)]{cook1982residuals}
R.~D. Cook and S.~Weisberg.
\newblock Residuals and influence in regression.
\newblock 1982.

\bibitem[Dekking(2005)]{dekking2005modern}
F.~M. Dekking.
\newblock \emph{A Modern Introduction to Probability and Statistics:
  Understanding why and how}.
\newblock Springer Science \& Business Media, 2005.

\bibitem[Fan et~al.(2024)Fan, Liu, Lin, Jia, Zhang, Mei, and
  Liu]{fan2024simplicity}
C.~Fan, J.~Liu, L.~Lin, J.~Jia, R.~Zhang, S.~Mei, and S.~Liu.
\newblock Simplicity prevails: Rethinking negative preference optimization for
  llm unlearning.
\newblock \emph{arXiv preprint arXiv:2410.07163}, 2024.

\bibitem[Ginart et~al.(2019)Ginart, Guan, Valiant, and Zou]{ginart2019making}
A.~Ginart, M.~Guan, G.~Valiant, and J.~Y. Zou.
\newblock Making ai forget you: Data deletion in machine learning.
\newblock \emph{Advances in neural information processing systems}, 32, 2019.

\bibitem[Graves et~al.(2021)Graves, Nagisetty, and Ganesh]{graves2021amnesiac}
L.~Graves, V.~Nagisetty, and V.~Ganesh.
\newblock Amnesiac machine learning.
\newblock In \emph{Proceedings of the AAAI Conference on Artificial
  Intelligence}, volume~35, pages 11516--11524, 2021.

\bibitem[Guo et~al.(2020)Guo, Goldstein, Hannun, and Van
  Der~Maaten]{guo2020certified}
C.~Guo, T.~Goldstein, A.~Hannun, and L.~Van Der~Maaten.
\newblock Certified data removal from machine learning models.
\newblock In \emph{International Conference on Machine Learning}, pages
  3832--3842. PMLR, 2020.

\bibitem[He et~al.(2025)He, Li, Cheng, Huang, and Huang]{he2025towards}
Z.~He, T.~Li, X.~Cheng, Z.~Huang, and X.~Huang.
\newblock Towards natural machine unlearning.
\newblock \emph{IEEE Transactions on Pattern Analysis and Machine
  Intelligence}, 2025.

\bibitem[Izzo et~al.(2021)Izzo, Smart, Chaudhuri, and Zou]{izzo2021approximate}
Z.~Izzo, M.~A. Smart, K.~Chaudhuri, and J.~Zou.
\newblock Approximate data deletion from machine learning models.
\newblock In \emph{International conference on artificial intelligence and
  statistics}, pages 2008--2016. PMLR, 2021.

\bibitem[Jin et~al.(2023)Jin, Chen, Zhang, and Li]{jin2023forgettable}
R.~Jin, M.~Chen, Q.~Zhang, and X.~Li.
\newblock Forgettable federated linear learning with certified data unlearning.
\newblock \emph{arXiv preprint arXiv:2306.02216}, 2023.

\bibitem[Kuchibhotla and Chakrabortty(2022)]{kuchibhotla2022moving}
A.~K. Kuchibhotla and A.~Chakrabortty.
\newblock Moving beyond sub-gaussianity in high-dimensional statistics:
  Applications in covariance estimation and linear regression.
\newblock \emph{Information and Inference: A Journal of the IMA}, 11\penalty0
  (4):\penalty0 1389--1456, 2022.

\bibitem[Li et~al.(2024{\natexlab{a}})Li, Pan, Gopal, Yue, Berrios, Gatti, Li,
  Dombrowski, Goel, Phan, et~al.]{li2024wmdp}
N.~Li, A.~Pan, A.~Gopal, S.~Yue, D.~Berrios, A.~Gatti, J.~D. Li, A.-K.
  Dombrowski, S.~Goel, L.~Phan, et~al.
\newblock The wmdp benchmark: Measuring and reducing malicious use with
  unlearning.
\newblock \emph{arXiv preprint arXiv:2403.03218}, 2024{\natexlab{a}}.

\bibitem[Li et~al.(2022)Li, Cai, and Li]{li2022transfer}
S.~Li, T.~T. Cai, and H.~Li.
\newblock Transfer learning for high-dimensional linear regression: Prediction,
  estimation and minimax optimality.
\newblock \emph{Journal of the Royal Statistical Society Series B: Statistical
  Methodology}, 84\penalty0 (1):\penalty0 149--173, 2022.

\bibitem[Li et~al.(2024{\natexlab{b}})Li, Zhang, Cai, and Li]{li2024estimation}
S.~Li, L.~Zhang, T.~T. Cai, and H.~Li.
\newblock Estimation and inference for high-dimensional generalized linear
  models with knowledge transfer.
\newblock \emph{Journal of the American Statistical Association}, 119\penalty0
  (546):\penalty0 1274--1285, 2024{\natexlab{b}}.

\bibitem[Lin et~al.(2023)Lin, Chung, Lao, and Zhao]{lin2023machine}
H.~Lin, J.~W. Chung, Y.~Lao, and W.~Zhao.
\newblock Machine unlearning in gradient boosting decision trees.
\newblock In \emph{Proceedings of the 29th ACM SIGKDD Conference on Knowledge
  Discovery and Data Mining}, pages 1374--1383, 2023.

\bibitem[Liu et~al.(2022)Liu, Liu, and Stone]{liu2022continual}
B.~Liu, Q.~Liu, and P.~Stone.
\newblock Continual learning and private unlearning.
\newblock In \emph{Conference on Lifelong Learning Agents}, pages 243--254.
  PMLR, 2022.

\bibitem[Liu et~al.(2025{\natexlab{a}})Liu, Yao, Jia, Casper, Baracaldo, Hase,
  Yao, Liu, Xu, Li, et~al.]{liu2025rethinking}
S.~Liu, Y.~Yao, J.~Jia, S.~Casper, N.~Baracaldo, P.~Hase, Y.~Yao, C.~Y. Liu,
  X.~Xu, H.~Li, et~al.
\newblock Rethinking machine unlearning for large language models.
\newblock \emph{Nature Machine Intelligence}, pages 1--14, 2025{\natexlab{a}}.

\bibitem[Liu et~al.(2025{\natexlab{b}})Liu, Chen, Huang, Ni, and
  Imani]{liu2025recovertoforget}
Y.~Liu, H.~Chen, W.~Huang, Y.~Ni, and M.~Imani.
\newblock Recover-to-forget: Gradient reconstruction from lo{RA} for efficient
  {LLM} unlearning.
\newblock In \emph{Socially Responsible and Trustworthy Foundation Models at
  NeurIPS 2025}, 2025{\natexlab{b}}.
\newblock URL \url{https://openreview.net/forum?id=n7peBaPUmk}.

\bibitem[Ma et~al.(2024)Ma, Verchand, and Samworth]{ma2024high}
T.~Ma, K.~A. Verchand, and R.~J. Samworth.
\newblock High-probability minimax lower bounds.
\newblock \emph{arXiv preprint arXiv:2406.13447}, 2024.

\bibitem[Maini et~al.(2024)Maini, Feng, Schwarzschild, Lipton, and
  Kolter]{maini2024tofu}
P.~Maini, Z.~Feng, A.~Schwarzschild, Z.~C. Lipton, and J.~Z. Kolter.
\newblock Tofu: A task of fictitious unlearning for llms.
\newblock \emph{arXiv preprint arXiv:2401.06121}, 2024.

\bibitem[Neel et~al.(2021)Neel, Roth, and Sharifi-Malvajerdi]{neel2021descent}
S.~Neel, A.~Roth, and S.~Sharifi-Malvajerdi.
\newblock Descent-to-delete: Gradient-based methods for machine unlearning.
\newblock In \emph{Algorithmic Learning Theory}, pages 931--962. PMLR, 2021.

\bibitem[Nesterov et~al.(2018)]{nesterov2018lectures}
Y.~Nesterov et~al.
\newblock \emph{Lectures on convex optimization}, volume 137.
\newblock Springer, 2018.

\bibitem[Nguyen et~al.(2024)Nguyen, Vu, Nguyen, Nguyen, Doan, and
  Wong]{nguyen2024empirical}
T.-H. Nguyen, H.-P. Vu, D.~T. Nguyen, T.~M. Nguyen, K.~D. Doan, and K.-S. Wong.
\newblock Empirical study of federated unlearning: Efficiency and
  effectiveness.
\newblock In \emph{Asian Conference on Machine Learning}, pages 959--974. PMLR,
  2024.

\bibitem[Reeve et~al.(2021)Reeve, Cannings, and Samworth]{reeve2021adaptive}
H.~W. Reeve, T.~I. Cannings, and R.~J. Samworth.
\newblock Adaptive transfer learning.
\newblock \emph{The Annals of Statistics}, 49\penalty0 (6):\penalty0
  3618--3649, 2021.

\bibitem[Sekhari et~al.(2021)Sekhari, Acharya, Kamath, and
  Suresh]{sekhari2021remember}
A.~Sekhari, J.~Acharya, G.~Kamath, and A.~T. Suresh.
\newblock Remember what you want to forget: Algorithms for machine unlearning.
\newblock \emph{Advances in Neural Information Processing Systems},
  34:\penalty0 18075--18086, 2021.

\bibitem[Shaik et~al.(2024)Shaik, Tao, Xie, Li, Zhu, and
  Li]{shaik2024exploring}
T.~Shaik, X.~Tao, H.~Xie, L.~Li, X.~Zhu, and Q.~Li.
\newblock Exploring the landscape of machine unlearning: A comprehensive survey
  and taxonomy.
\newblock \emph{IEEE Transactions on Neural Networks and Learning Systems},
  2024.

\bibitem[Sudlow et~al.(2015)Sudlow, Gallacher, Allen, Beral, Burton, Danesh,
  Downey, Elliott, Green, Landray, et~al.]{sudlow2015uk}
C.~Sudlow, J.~Gallacher, N.~Allen, V.~Beral, P.~Burton, J.~Danesh, P.~Downey,
  P.~Elliott, J.~Green, M.~Landray, et~al.
\newblock Uk biobank: an open access resource for identifying the causes of a
  wide range of complex diseases of middle and old age.
\newblock \emph{PLoS medicine}, 12\penalty0 (3):\penalty0 e1001779, 2015.

\bibitem[Tian and Feng(2023)]{tian2023transfer}
Y.~Tian and Y.~Feng.
\newblock Transfer learning under high-dimensional generalized linear models.
\newblock \emph{Journal of the American Statistical Association}, 118\penalty0
  (544):\penalty0 2684--2697, 2023.

\bibitem[Tibshirani(1996)]{tibshirani1996regression}
R.~Tibshirani.
\newblock Regression shrinkage and selection via the lasso.
\newblock \emph{Journal of the Royal Statistical Society Series B: Statistical
  Methodology}, 58\penalty0 (1):\penalty0 267--288, 1996.

\bibitem[Vershynin(2018)]{vershynin2018high}
R.~Vershynin.
\newblock \emph{High-dimensional probability: An introduction with applications
  in data science}, volume~47.
\newblock Cambridge university press, 2018.

\bibitem[Wang et~al.(2025)Wang, Zhou, Zhou, Shin, Han, and
  Weinberger]{wang2025rethinking}
Q.~Wang, J.~P. Zhou, Z.~Zhou, S.~Shin, B.~Han, and K.~Q. Weinberger.
\newblock Rethinking {LLM} unlearning objectives: A gradient perspective and go
  beyond.
\newblock In \emph{The Thirteenth International Conference on Learning
  Representations}, 2025.
\newblock URL \url{https://openreview.net/forum?id=huo8MqVH6t}.

\bibitem[Wu et~al.(2020)Wu, Dobriban, and Davidson]{wu2020deltagrad}
Y.~Wu, E.~Dobriban, and S.~Davidson.
\newblock Deltagrad: Rapid retraining of machine learning models.
\newblock In \emph{International Conference on Machine Learning}, pages
  10355--10366. PMLR, 2020.

\bibitem[Yang et~al.(2025)Yang, Wang, Huang, Liu, Zhang, and
  Han]{yang2025exploring}
P.~Yang, Q.~Wang, Z.~Huang, T.~Liu, C.~Zhang, and B.~Han.
\newblock Exploring criteria of loss reweighting to enhance {LLM} unlearning.
\newblock In \emph{Forty-second International Conference on Machine Learning},
  2025.
\newblock URL \url{https://openreview.net/forum?id=mGOugCZlAq}.

\bibitem[Yao et~al.(2024)Yao, Xu, and Liu]{yao2024large}
Y.~Yao, X.~Xu, and Y.~Liu.
\newblock Large language model unlearning.
\newblock \emph{Advances in Neural Information Processing Systems},
  37:\penalty0 105425--105475, 2024.

\bibitem[Zhang et~al.(2024)Zhang, Lin, Bai, and Mei]{zhang2024negative}
R.~Zhang, L.~Lin, Y.~Bai, and S.~Mei.
\newblock Negative preference optimization: From catastrophic collapse to
  effective unlearning.
\newblock \emph{CoRR}, 2024.

\bibitem[Zhang et~al.(2015)Zhang, Zhao, and LeCun]{zhang2015character}
X.~Zhang, J.~Zhao, and Y.~LeCun.
\newblock Character-level convolutional networks for text classification.
\newblock \emph{Advances in neural information processing systems}, 28, 2015.

\end{thebibliography}


\begin{thebibliography}{6}
\providecommand{\natexlab}[1]{#1}
\providecommand{\url}[1]{\texttt{#1}}
\expandafter\ifx\csname urlstyle\endcsname\relax
  \providecommand{\doi}[1]{doi: #1}\else
  \providecommand{\doi}{doi: \begingroup \urlstyle{rm}\Url}\fi

\bibitem[Balakrishnan et~al.(2017)Balakrishnan, Wainwright, and
  Yu]{balakrishnan2017statistical}
Sivaraman Balakrishnan, Martin~J Wainwright, and Bin Yu.
\newblock Statistical guarantees for the em algorithm: From population to
  sample-based analysis.
\newblock \emph{The Annals of Statistics}, 45\penalty0 (1):\penalty0 77--120,
  2017.

\bibitem[Nesterov et~al.(2018)]{nesterov2018lectures}
Yurii Nesterov et~al.
\newblock \emph{Lectures on convex optimization}, volume 137.
\newblock Springer, 2018.

\bibitem[Vershynin(2018)]{vershynin2018high}
Roman Vershynin.
\newblock \emph{High-dimensional probability: An introduction with applications
  in data science}, volume~47.
\newblock Cambridge university press, 2018.

\bibitem[Kuchibhotla and Chakrabortty(2022)]{kuchibhotla2022moving}
Arun~Kumar Kuchibhotla and Abhishek Chakrabortty.
\newblock Moving beyond sub-gaussianity in high-dimensional statistics:
  Applications in covariance estimation and linear regression.
\newblock \emph{Information and Inference: A Journal of the IMA}, 11\penalty0
  (4):\penalty0 1389--1456, 2022.

\bibitem[Ma et~al.(2024)Ma, Verchand, and Samworth]{ma2024high}
Tianyi Ma, Kabir~A Verchand, and Richard~J Samworth.
\newblock High-probability minimax lower bounds.
\newblock \emph{arXiv preprint arXiv:2406.13447}, 2024.

\bibitem[Tibshirani(1996)]{tibshirani1996regression}
Robert Tibshirani.
\newblock Regression shrinkage and selection via the lasso.
\newblock \emph{Journal of the Royal Statistical Society Series B: Statistical
  Methodology}, 58\penalty0 (1):\penalty0 267--288, 1996.

\end{thebibliography}

\end{document}